\documentclass[a4paper,USenglish,numberwithinsect]{lipics-v2021}

\nolinenumbers 
\pdfoutput=1 
\hideLIPIcs  

\usepackage{pgfplots}
\pgfplotsset{compat=1.18}

\usepackage{subcaption}

\newcommand{\coloneqq}{\mathrel{\mathord :=}}
\newcommand{\eqqcolon}{\mathrel{=\mathord :}}

\newcommand{\Rbf}{{\mathbf R}}
\newcommand{\Rbflin}{{\mathbf R_{\rm lin}}}
\newcommand{\R}{{\mathbb R}}
\newcommand{\Rbot}{{\R_\bot}}
\newcommand{\D}{{\mathcal D}}
\newcommand{\E}{{\mathcal E}}
\newcommand{\F}{{\mathcal F}}
\newcommand{\Net}{{\mathcal{N}}}
\newcommand{\A}{\mathcal{A}}
\newcommand{\B}{\mathcal{B}}

\newcommand{\Arr}{\A}
\newcommand{\Upscell}{\Upsilon^{\mathrm{cell}}}
\newcommand{\Upsnet}{\Upsilon^{\mathrm{net}}}
\newcommand{\Upsarr}{\Upsilon^{\mathrm{arr}}}
\newcommand{\Upspwl}{\Upsilon^{\mathrm{pwl}}}
\newcommand{\Upsconst}{\Upsilon^{\mathrm{const}}}
\newcommand{\PWLV}[1]{\Upspwl_{#1}}
\newcommand{\hpArrV}[1]{\Upsarr_{#1}}
\newcommand{\cellV}[1]{\Upscell_{#1}}
\newcommand{\foR}{\textup{FO}(\Rbf)}
\newcommand{\foRlin}{\textup{FO}(\Rbflin)}
\newcommand{\form}[1]{\textup{FO}(\Rbf,#1)}
\newcommand{\forlinm}[1]{\textup{FO}(\Rbflin,#1)}
\newcommand{\bbl}{\form F}
\newcommand{\bblin}{\forlinm F}

\newcommand{\wal}{\textup{FO(SUM)}}

\newcommand{\K}{\mathbf F}
\newcommand{\ReLU}{\mathrm{ReLU}}
\newcommand{\inn}{\mathrm{in}}
\newcommand{\out}{\mathrm{out}}
\newcommand{\bvec}[1]{{\boldsymbol #1}}
\newcommand{\ba}{\bvec a}
\newcommand{\xx}{\bvec x}
\newcommand{\vv}{\bvec v}
\newcommand{\Func}[1]{F^{#1}}
\newcommand{\Funcu}[2]{F^{#1}_{#2}}
\newcommand{\val}{\mathit{val}}
\newcommand{\eval}{\mathit{eval}}
\newcommand{\breakx}{\mathit{break}_x}
\newcommand{\breaky}{\mathit{break}_y}
\newcommand{\hiddenF}{\mathit{hidden}}
\newcommand{\breakNextF}{\textit{succ}}
\newcommand{\surface}{\mathit{area}(u_1,u_2)}
\newcommand{\mend}{\mathrm{end}}
\newcommand{\mstart}{\mathrm{start}}
\newcommand{\startR}{\mathrm{start}}
\newcommand{\EndR}{\mathrm{end}}
\newcommand{\structA}{\mathcal{A}}
\newcommand{\structB}{\mathcal{B}}

\newcommand{\domF}{\varphi_{\mathrm{dom}}}
\newcommand{\domR}[1]{\mathrm{dom}_{#1}}
\newcommand{\eqF}{\varphi_{=}}

\newcommand{\PL}{\mathcal{P\!\!L}}

\DeclareMathOperator{\Shap}{\text{\normalfont\textsc{Shap}}}

\usepackage{tikz}
\usetikzlibrary{arrows}
\usetikzlibrary{shapes.geometric}
\usetikzlibrary{calc}

\usepackage{graphicx}

\title{Query languages for neural networks}

\author{Martin Grohe}{RWTH Aachen University, Aachen, Germany}{grohe@informatik.rwth-aachen.de}{https://orcid.org/0000-0002-0292-9142}{}

\author{Christoph Standke}{RWTH Aachen University, Aachen, Germany}{standke@informatik.rwth-aachen.de}{https://orcid.org/0000-0002-3034-730X}{}

\author{Juno Steegmans}{UHasselt, Data Science Institute, Diepenbeek, Belgium}{juno.steegmans@uhasselt.be}{https://orcid.org/0000-0003-2087-9430}{Supported by the Special Research Fund (BOF) of UHasselt}

\author{Jan {Van den Bussche}}{UHasselt, Data Science Institute, Diepenbeek, Belgium}{jan.vandenbussche@uhasselt.be}{https://orcid.org/0000-0003-0072-3252}{}

\authorrunning{M. Grohe, C. Standke, J. Steegmans, and J, Van den Bussche}
\Copyright{Martin Grohe, Christoph Standke, Juno Steegmans and Jan Van den Bussche}



\begin{document}

\maketitle

\begin{abstract}

We lay the foundations for a database-inspired approach to
interpreting and understanding neural network models by
querying them using declarative languages. Towards this end we
study different query languages, based on first-order logic,
that mainly differ in their access to the neural network model.
First-order logic over the reals naturally yields a language
which views the network as a black box; only the input--output
function defined by the network can be queried.  This is
essentially the approach of constraint query languages.  On the
other hand, a white-box language can be obtained by viewing the
network as a weighted graph, and extending
first-order logic with summation over weight terms.  The latter
approach is essentially an abstraction of SQL\@.  In general,
the two approaches are incomparable in expressive power, as we
will show.  Under natural circumstances, however, the white-box
approach can subsume the black-box approach; this is our main
result.  We prove the result concretely for linear constraint
queries over real functions definable by
feedforward neural networks with a fixed number of hidden layers and
piecewise linear activation functions.

\end{abstract}

\section{Introduction}

Neural networks \cite{goodfellow-book} are a popular and
successful representation model for real functions
learned from data.  Once deployed, the neural network is
``queried'' by supplying it with inputs then obtaining the
outputs.  In the field of databases, however, we
have a much richer conception of querying than simply applying a
function to given arguments.  For example, in querying a database
relation $\rm Employee(name,salary)$, we can not only ask
for Anne's salary; we can also ask how many salaries are
below that of Anne's; we can ask whether no two employees have
the same salary; and so on.

In this paper, we consider the querying of neural networks from
this more general perspective.  We see many potential
applications: obvious ones are
in explanation, verification, and
interpretability of neural networks and other machine-learning
models \cite{molnar-book,aws-book,LiuALSBK21}.  These are huge
areas \cite{rudin-stop-explaining,foscadino-explain-survey}
where it is important \cite{silva-logic-explain,kwiat-cert-nn} to
have formal, logical definitions for the myriad notions of
explanation that are being considered.  Another potential
application is in managing machine-learning projects, where we
are testing many different architectures and training datasets,
leading to a large number of models, most of which become
short-term legacy code.  In such a context it would be useful if
the data scientist could search the repository for earlier
generated models having certain characteristics in their
architecture or in their behavior, which were perhaps not duly
documented.

The idea of querying machine learning models with an expressive,
declarative query language comes naturally to database
researchers, and indeed, Arenas et al.\ already proposed a
language for querying boolean functions over an unbounded set of
boolean features \cite{arenas-foil}.  In the modal logic
community, similar languages are being investigated
\cite[references therein]{lorini-logic-explain}.

In the present work, we focus on real, rather than boolean,
functions and models, as is indeed natural in the setting of
verifying neural networks \cite{aws-book}.

\subparagraph*{The constraint query language approach}

A natural language for querying real functions on a fixed number
of arguments (features) is obtained by simply using first-order
logic over the reals, with a function symbol $F$ representing the
function to be queried.  We denote this by $\foR$.  For example,
consider functions $F$ with three arguments.  The formula
$ \forall b'\, |F(a,b,c)-F(a,b',c)|<\epsilon $
expresses that the output on $(a,b,c)$ does not depend
strongly on the second feature, i.e., $F(a,b',c)$ is
$\epsilon$-close to $F(a,b,c)$ for any $b'$.  Here, $a$, $b$, $c$
and $\epsilon$ can be real constants or parameters (free
variables).

The language $\foR$ (also known as $\textsc{FO}+\textsc{Poly}$)
and its restriction $\foRlin$ to linear arithmetic (aka
$\textsc{FO}+\textsc{Lin}$) were intensively investigated in
database theory around the turn of the century, under the heading
of \emph{constraint query languages}, with applications to
spatial and temporal databases.  See the compendium volume
\cite{cdbbook} and book chapters \cite[chapter 13]{libkin_fmt},
\cite[chapter 5]{fmta_book}.  Linear formulas with only universal
quantifiers over the reals, in front of a quantifier-free
condition involving only linear arithmetic (as the above example
formula), can already model many properties considered in the
verification of neural networks \cite{aws-book}.  This universal
fragment of $\foRlin$ can be evaluated using linear programming
techniques \cite{aws-book}.

Full $\foR$ allows alternation of quantifiers over the reals, and
multiplication in arithmetic.  Because the
first-order theory of the reals is decidable
\cite{basu_algorithms}, $\foR$ queries can still be effectively
evaluated on any function that is \emph{semi-algebraic}, i.e.,
itself definable in first-order logic over the reals.  Although
the complexity of this theory is high, if the function is
presented as a quantifier-free formula, $\foR$ query evaluation
actually has polynomial-time \emph{data} complexity; here, the
``data'' consists of the given quantifier-free formula
\cite{kkr_cql}.

Functions that can be represented by feedforward neural networks
with ReLU hidden units and linear output units are clearly
semi-algebraic; in fact, they are piecewise linear.
For most of our results, we will indeed focus on this class of 
networks, which are widespread in practice
\cite{goodfellow-book}, and denote them by ReLU-FNN\@.

\subparagraph*{The SQL approach}  Another natural approach to
querying neural networks is to query them directly, as graphs of
neurons with weights on the nodes and edges.  For this purpose
one represents such graphs as relational structures with
numerical values and uses SQL to query them.  As an abstraction of
this approach, in this paper, we model neural networks as
weighted finite structures.  As a query language we use $\wal$:
first-order logic over weighted structures, allowing order
comparisons between weight terms, where weight terms can be built
up using rational arithmetic, if-then-else, and, importantly,
summation.

Originally introduced by Gr\"adel and Gurevich
\cite{gg_metafinite}, the language $\wal$ is comparable to the
relational calculus with aggregates \cite{klug_agg} and, thus, to
SQL \cite{libkin_sql}.  Logics close to $\wal$, but involving
arithmetic in different semirings, were recently also used for
unifying different algorithmic problems in query processing
\cite{toruncz-faq}, as well as for expressing hypotheses in the
context of learning over structures \cite{vbs-wal}.  The
well-known FAQ framework \cite{faq_sigmodrecord}, restricted to
the real semiring, can be seen as the conjunctive fragment of
$\wal$.

To give a simple example of an $\wal$ formula,
consider ReLU-FNNs with a single input unit,
one hidden layer of ReLU units, and a single linear output unit.
The following formula expresses the query that asks if the
function evaluation on a given input value is positive:
\[ 0 < b(\out) + \sum_{x:E(\inn,x)} w(x,\out) \cdot \ReLU(w(\inn,x) \cdot
\mathit{val} + b(x)). \]
Here, $E$ is the edge relation between neurons,
and constants in and out hold the input and output unit,
respectively.  Thus, variable $x$ ranges over
the neurons in the hidden layer.
Weight functions $w$ and $b$ indicate the
weights of edges and the biases of units, respectively; the
weight constant $\it val$ stands for a given input value.
We assume for clarity that ReLU is
given, but it is definable in $\wal$.

Just like the relational calculus with aggregates, or SQL select
statements, query evaluation for $\wal$ has polynomial time data
complexity, and techniques for query processing and optimization
from database systems directly apply.

\subparagraph*{Comparing expressive powers}

Expressive power of query languages has been a classical topic in
database theory and finite model theory
\cite{ahv_book,libkin_fmt}, so, with the advent of new models, it
is natural to revisit questions concerning expressivity.  The
goal of this paper is to understand and compare the expressive
power of the two query languages $\foR$ and $\wal$ on neural
networks over the reals.  The two languages are quite different.
$\foR$ sees the model as a black-box function $F$, but can
quantify over the reals.  $\wal$ can see the model as a white
box, a finite weighted structure, but can quantify only over the
elements of the structure, i.e., the neurons.

In general, indeed the two expressive powers are incomparable.
In $\wal$, we can express queries about the network topology; for
example, we may ask to return the hidden units that do not
contribute much to the function evaluation on a given input
value. (Formally, leaving them out of the network would yield an
output within some $\epsilon$ of the original output.)  Or, we
may ask whether there are more than a million neurons in the
first hidden layer.  For $\foR$, being a black box language, such
queries are obviously out of scope.

A more interesting question is how the two languages compare in
expressing \emph{model agnostic} queries: these are queries that
return the same result on any two neural networks that represent
the same input--output function.  For example, when restricting
attention to networks with one hidden layer, the example $\wal$
formula seen earlier, which evaluates the network,
is model agnostic.  $\foR$ is model agnostic by design,
and, indeed, serves as a very natural declarative benchmark of
expressiveness for model-agnostic queries.  It turns out that
$\wal$, still restricting to networks of some fixed depth,
can express model-agnostic queries that $\foR$
cannot.  For example, for any fixed depth $d$, we will show that
$\wal$ can express the integrals of a functions 
given by a ReLU-FNNs of depth $d$. In contrast, we will show
that this cannot be done in $\foR$ (Theorem~\ref{thmint}).

The depth of a neural network can be taken as a crude notion of
``schema''.  Standard relational query languages typically
cannot be used without knowledge of the schema of the data.
Similarly, we will show that without knowledge of the depth,
$\wal$ cannot express any nontrivial model-agnostic query
(Theorem~\ref{thm:fully_agnostic}).  Indeed, since
$\wal$ lacks recursion, function evaluation can only be expressed if
we known the depth.
(Extensions with
recursion is one of the many interesting directions for further
research.)

When the depth is known, however, for model-agnostic queries, the
expressiveness of $\wal$ exceeds the benchmark of expressiveness
provided by $\foRlin$.  Specifically, we show that every
$\foRlin$ query over functions representable by ReLU-FNNs is also
expressible in $\wal$ evaluated on the networks directly
(Theorem~\ref{theormain}).  This is our main technical result,
and can be paraphrased as \emph{``SQL can verify neural
networks.''}  The proof involves showing that the required
manipulations of higher-dimensional piecewise linear functions,
and the construction of cylindrical cell decompositions in
$\R^n$, can all be expressed in \wal.  To allow for a modular
proof, we also develop the notion of $\wal$ translation,
generalizing the classical notion of first-order interpretations
\cite{hodges}.

This paper is organized as follows. Section~\ref{secprel}
provides preliminaries on neural networks. Section~\ref{secbbl}
introduces $\foR$. Section~\ref{secwal} introduces weighted
structures and $\wal$, after which Section~\ref{secwbl}
introduces white-box querying. Section~\ref{secagnostic} considers
model-agnostic queries. Section~\ref{sectheorem} presents the
main technical result.
Section~\ref{seconc}
concludes with a discussion of topics for further research.

\section{Preliminaries on neural networks} \label{secprel}

A feedforward neural network \cite{goodfellow-book}, in general,
could be defined as a finite, directed, weighted, acyclic graph,
with some additional aspects which we discuss next. The nodes
are also referred to as \emph{neurons} or \emph{units}.  Some of the
source nodes are designated as \emph{inputs}, and some of the
sink nodes are designated as \emph{outputs}.  Both the inputs,
and the outputs, are linearly ordered.  Neurons that are neither
inputs nor outputs are said to be \emph{hidden}.  All nodes,
except for the inputs, carry a weight, a real value, called the
\emph{bias}.  All directed edges also carry a weight.

In this paper, we focus on \emph{ReLU-FNNs}: networks with
ReLU activations and linear outputs.
This means the following.  Let $\Net$ be a neural network with
$m$ inputs.  Then every node $u$ in $\Net$ represents a function
$\Funcu \Net u  : \R^m \to \R$ defined as follows.  We proceed
inductively based on some topological ordering of $\Net$.  For
input nodes $u$, simply $\Funcu \Net u(x_1,\dots,x_m) := x_i$, if $u$
is the $i$th input node.  Now let $u$ be a hidden neuron and assume
$\Funcu \Net v$ is already defined for all \emph{predecessors}
$v$ of $u$, i.e., nodes $v$ with an edge to $u$.
Let $v_1,\dots,v_l$ be these
predecessors, let $w_1,\dots,w_l$ be the weights on the
respective edges, and let $b$ be the bias of $v$.  Then \[ \Funcu
\Net u(\xx) := \ReLU(b + \sum_i w_i\Funcu \Net {v_i}(\xx)), \]
where $\ReLU : \R \to \R : z \mapsto \max(0,z)$.

Finally, for an output node $u$, we define $\Funcu \Net u$
similarly to hidden neurons, except that the application of ReLU
is omitted.
The upshot is that a neural network $\Net$ with $m$ inputs and $n$
outputs $u_1,\dots,u_n$ represents a function
$\Func\Net : \R^m \to \R^n$ mapping $\xx$ to $(\Funcu\Net{u_1}(\xx),\dots,\Funcu\Net{u_n}(\xx))$.
For any node $u$ in the network, $\Funcu\Net u$ is always a continuous piecewise
linear function. We denote the class of all continuous piecewise linear functions $F:\R^m\to\R$ by $\PL(m)$; that is, continuous functions $F$ that admit a partition of $\R^m$ into finitely many polytopes such that $F$ is affine linear on each of them.

\subparagraph*{Hidden layers} Commonly, the hidden neurons are
organized in disjoint blocks called \emph{layers}.  The layers
are ordered, such that the neurons in the first layer have
only edges from inputs, and the neurons in any later layer have
only edges from neurons in the previous layer.  Finally, outputs
have only edges from neurons in the last layer.

We will use $\K(m,\ell)$ to denote the class of layered networks with $m$ inputs of depth $\ell$, that is, with an input layer with $m$ nodes, $\ell-1$ hidden layers, and an output layer with a single node. Recall that the nodes on all hidden layer use ReLU activations and the output node uses the identity function.

It is easy to see that networks in $\K(1,1)$ just compute linear functions and that for every $\ell\ge 2$ we have
$\{F^{\Net}\mid \Net\in\K(1,\ell)\}=\PL(1)$,
that is, the class of functions $\R\to\R$ that can be computed by a network in $\K(1,\ell)$ is the class of all continuous piecewise linear functions.
The well-known \emph{Universal Approximation Theorem} \cite{Cybenko89,Hornik91} says that every continuous function $f:K\to\R$ defined on a compact domain $K\subseteq\R^m$ can be approximated to any additive error by a network in $\K(m,2)$.

\section{A black-box query language} \label{secbbl}

First-order logic over the reals, denoted here by $\foR$, is,
syntactically, just first-order logic over the vocabulary of
elementary arithmetic, i.e., with binary function symbols $+$ and
$\cdot$ for
addition and multiplication, binary predicate $<$,
and constant symbols $0$ and $1$ \cite{basu_algorithms}.
Constants for rational numbers, or even algebraic
numbers, can be added as an abbreviation (since they are
definable in the logic).

The fragment $\foRlin$ of \emph{linear} formulas uses
multiplication only for scalar multiplication, i.e.,
multiplication of variables with rational number constants.  For
example, the formula $y=3x_1-4x_2+7$ is linear, but the formula
$y=5x_1\cdot x_2 - 3$ is not.  In practice, linear queries are
often sufficiently expressive, both from earlier applications for
temporal or spatial data \cite{cdbbook}, as well as for querying
neural networks (see examples to follow). The only caveat is that
many applications assume a distance function on vectors.  When
using distances based on absolute value differences between real
numbers, e.g., the Manhattan distance or the max norm, we still
fall within $\foRlin$.

We will add to $\foR$ extra relation or function symbols;
in this paper, we will mainly consider $\bbl$, which is $\foR$
with an extra function symbol $F$.
The structure on the domain $\R$ of reals, with the arithmetic
symbols having their obvious interpretation, will be denoted here
by $\Rbf$.  Semantically, for any vocabulary $\tau$ of extra
relation and function symbols, $\form\tau$ formulas are
interpreted over structures that expand $\Rbf$ with additional
relations and functions on $\R$ of the right arities, that
interpret the symbols in $\tau$.  In this way, $\bbl$ expresses
queries about functions $F\colon \R^m\to \R$.

This language can express a wide variety of properties (queries)
considered in interpretable machine learning and neural-network
verification.  Let us see some examples.

\begin{example}\label{exrobust}
To check whether $F\colon \R^m\to \R$ is robust around
  an $m$-vector $\bvec a$ \cite{szegedy-intriguing}, using parameters
  $\epsilon$ and $\delta$,
  we can write the formula $ \forall \bvec
  x(d(\bvec x,\bvec a)<\epsilon \Rightarrow |F(\bvec x)-F(\bvec
  a)|<\delta)$.
  Here $\bvec x$ stands for a tuple of $m$ variables, and $d$ stands
  for some distance function which is assumed to be expressible.
\end{example}

\begin{example}\label{excounterf}
  Counterfactual explanation methods
  \cite{wachter-counterfactual} aim to find the closest
  $\bvec x$ to an input $\bvec a$ such that $F(\bvec x)$ is
  ``expected,'' assuming that $F(\bvec a)$ was unexpected.
  A typical example is credit denial; what should we change
  minimally to be granted credit?
  Typically we can define expectedness by some formula, e.g.,
  $F(\bvec x) > 0.9$. Then we can express the
  counterfactual explanation as $ F(\bvec x)>0.9 \land \forall
  \bvec y(F(\bvec y)>0.9 \Rightarrow d(\bvec x,\ba) \leq d(\bvec
  y,\ba))$.
\end{example}

\begin{example}\label{excontrib}
  We may define the contribution of an input feature $i$
  on an input $\ba=(a_1,\dots,a_m)$ as the inverse of the
  smallest change we have to make to that feature for the
  output to change significantly.  We can express that $r$ is
  such a change by writing (taking $i=1$ for clarity) $r>0 \land
  (d(F(a_1-r,a_2,\dots,a_m),F(\ba)) > \epsilon \lor
  d(F(a_1+r,a_2,\dots,a_m),F(\ba)) > \epsilon)$.  Denoting this
  formula by $\mathit{change}(r)$, the smallest change is then
  expressed as $\mathit{change}(r) \land \forall
  r'(\mathit{change}(r') \Rightarrow r\leq r')$.
\end{example}
  
\begin{example}
  We finally illustrate that $\bbl$
can express gradients and many other notions from
  calculus.  For simplicity
  assume $F$ to be unary.  Consider the definition
  $F'(a) = \lim_{x\to c} (F(x)-F(c))/(x-c)$ of the derivative
  in a point $c$.  So it suffices to show how to express
  that $l=\lim_{x\to c} G(x)$ for a function $G$ that is continuous in $c$.
  We can write down the textbook definition literally as
  $ \forall \epsilon>0\, \exists \delta>0\,
  \forall x (|x-c|<\delta \Rightarrow |G(x)-l|<\epsilon)$.
\end{example}

\subparagraph*{Evaluating $\foR$ queries}  Black box
queries can be effectively evaluated
using the decidability and quantifier elimination properties of
$\foR$.  This is the constraint query language approach
\cite{kkr_cql,cdbbook}, which we briefly recall next.

A function $f\colon \R^m\to\R$ is called semialgebraic
\cite{basu_algorithms} (or semilinear) if there exists an $\foR$
(or $\foRlin$) formula $\varphi(x_1,\dots,x_m,y)$ such that for
any $m$-vector $\bvec a$ and real value $b$, we have $\Rbf
\models \varphi(\bvec a,b)$ if and only if $F(\bvec a)=b$.

Now consider the task of
evaluating an $\bbl$ formula $\psi$ on a semialgebraic function $f$,
given by a defining formula $\varphi$.  By introducing auxiliary
variables, we may assume that the function symbol $F$ is used in
$\psi$ only in subformulas for the form
$z=F(u_1,\dots,u_m)$.  Then replace in $\psi$ each such
subformula by $\varphi(u_1,\dots,u_m,z)$, obtaining a pure $\foR$
formula $\chi$.

Now famously, the first-order theory of $\R$ is
decidable \cite{tarski_decision,basu_algorithms}.
In other words, there is an
algorithm that decides, for any $\foR$
formula $\chi(x_1,\dots,x_k)$ and $k$-vector $\bvec c$, whether
$\Rbf \models \chi(\bvec c)$.  Actually, a stronger property holds,
to the effect that every $\foR$-formula is equivalent to a
quantifier-free formula.  The upshot is that there is an
algorithm that, given a $\bbl$ query $\psi(x_1,\dots,x_k)$
and a semialgebraic function $f$ given by a defining formula,
outputs a quantifier-free formula defining the result set
$\{\bvec c \in \R^k \mid \Rbf,f \models \psi(\bvec c)\}$.
If $f$ is given by a quantifier-free formula, the evaluation can
be done in polynomial time in the length of the description of
$f$, so polynomial-time data complexity. This is because the
exponential complexity of the
first-order theory of the reals lies mainly in the number
of variables and the number of quantifiers
\cite{kkr_cql,basu_algorithms}.

\subparagraph*{Complexity} Of course, we want to evaluate queries
on the functions represented by neural networks.  From the
definition given in Section~\ref{secprel}, it is clear that the
functions representable by ReLU-FNNs are always semialgebraic
(actually, semilinear).  For every output feature $j$, it is
straightforward to compile, from the network, a quantifier-free
formula defining the $j$th output component function.  In this
way we see that $\bbl$ queries on ReLU-FNNs are, in principle,
computable in polynomial time.

However, the algorithms are notoriously complex, and we stress
again that $\bbl$ should be mostly seen as a \emph{declarative
benchmark} of expressiveness.  Moreover, we assume here for
convenience that ReLU is a primitive function that does not need
to be expressed using disjunction.  However, symbolic constraint
solving algorithms for the reals can indeed be extended to deal
with ReLU natively \cite{aws-book}.

\begin{remark}
In closing this Section
we remark that, to query the entire network function, we would
not strictly use just only a single
function symbol $F$, but rather the language
$\form{F_1,\dots,F_n}$, with function symbols for the $n$
outputs.  In this paper, for the sake of clarity, we will often
stick to a single output, but our treatment generalizes to
multiple outputs.
\end{remark}

\section{Weighted structures and $\wal$} \label{secwal}

Weighted structures are standard abstract structures equipped
with one or more weight functions from tuples of domain elements
to values from some separate, numerical domain.  Here, as
numerical domain, we will use $\Rbot = \R \cup \{\bot\}$, the set
of ``lifted reals'' where $\bot$ is an extra element representing
an undefined value.  Neural networks are weighted graph
structures.  Hence, since we are interested in declarative query
languages for neural networks, we are interested in logics over
weighted structures.  Such logics were introduced by Gr\"adel and
Gurevich \cite{gg_metafinite}.  We consider here a concrete
instantiation of their approach, which we denote by $\wal$.

Recall that a (finite, relational) vocabulary is a finite set of
function symbols and relation symbols, where each symbol
comes with an arity (a natural number).  We extend the notion of
vocabulary to also include a number of \emph{weight function
symbols}, again with associated arities. We allow $0$-ary weight function symbols, which we call \emph{weight constant symbols}.

A (finite) \emph{structure} $\A$ over such a vocabulary
$\Upsilon$ consists of a finite domain $A$, and functions and
relations on $A$ of the right arities, interpreting the standard
function symbols and relation symbols from $\Upsilon$.
So far this is standard.
Now additionally, $\A$ interprets
every weight function symbol $w$, of arity $k$, by
a function $w^\A : A^k \to \Rbot$.

The syntax of $\wal$ formulas (over some vocabulary)
is defined exactly as for standard first
order logic, with one important extension.
In addition to \emph{formulas} (taking Boolean values) and \emph{standard terms} (taking values in the structure), the logic contains \emph{weight terms} taking values in $\Rbot$. Weight terms $t$ are defined by the following grammar:
\[
t ::= \bot \mid w(s_1,\dots,s_n) \mid
r(t,\dots,t) \mid
\textsf{if $\varphi$ then $t$ else $t$}
\mid \sum_{\bvec x:\varphi} t 
\]
Here, $w$ is a weight function symbol of arity $n$ and
the $s_i$ are standard terms;
$r$ is a rational function applied to weight terms, with rational
coefficients; $\varphi$ is a formula; and $\bvec x$ is a tuple of
variables. The syntax of weight terms and formulas is mutually recursive. As just seen,
the syntax of formulas $\varphi$ is used in the syntax of weight terms;
conversely, weight terms $t_1$ and $t_2$ can be combined to form
formulas $t_1 = t_2$ and
$t_1 < t_2$.


Recall that a rational function is a fraction between two
polynomials.  Thus, the arithmetic operations that we consider
are addition, scalar multiplication by a rational number,
multiplication, and division.

The \emph{free variables} of a weight term are defined as
follows.  The weight term $\bot$ has no free variables.  The free
variables of $w(s_1,\dots,\allowbreak s_n)$ are simply the variables
occurring in the $s_i$.  A variable occurs free in $r(t_1,\dots,t_n)$
if it occurs free in some $t_i$.  A variable occurs free in
`\textsf{if $\varphi$ then $t_1$ else $t_2$}' if it occurs free in
$t_1$, $t_2$, or $\varphi$.  The free variables of
$\sum_{\bvec x:\varphi} t$ are those of $\varphi$ and
$t$, except for the variables in $\xx$. A formula or (weight) term is \emph{closed} if it has no free variables.

We can evaluate a weight term $t(x_1,\dots,x_k)$ on a
structure $\A$ and a tuple $\bvec a \in A^k$ providing values to the
free variables.  The result of the evaluation,
denoted by $t^{\A,\bvec a}$, is a value in $\Rbot$, defined
in the obvious manner.  In particular, when $t$ is of the form
$\sum_{\bvec y:\varphi} t'$, we have \[ t^{\A,\bvec a} =
\sum_{\bvec b : \A \models \varphi(\ba,\bvec b)} t'^{\A,\ba,\bvec
b}. \]  Division by zero, which can happen when evaluating
terms of the form $r(t,\dots,t)$, is given the value $\bot$. The
arithmetical operations are extended so that $x+\bot$, $q\bot$
(scalar multiply),
$x\cdot \bot$, and $x/\bot$ and $\bot/x$ always equal $\bot$.
Also, $\bot < a$ holds for all $a\in \R$.

\section{White-box querying} \label{secwbl}

For any natural numbers $m$ and $n$,
we introduce a vocabulary for neural networks with
$m$ inputs and $n$ outputs.  We denote this vocabulary by
$\Upsnet(m,n)$, or just $\Upsnet$ if $m$ and $n$ are understood.
It has a binary relation symbol $E$ for the edges;
constant symbols $\inn_1$, \dots, $\inn_m$ and $\out_1$, \dots,
$\out_n$ for the input and output nodes; a unary weight
function $b$ for the biases, and a binary weight function
symbol $w$ for the weights on the edges.

Any ReLU-FNN $\Net$, being a weighted graph, is
an $\Upsnet$-structure in the
obvious way.  When there is no edge from node $u_1$ to $u_2$, we
put $w^\Net(u_1,u_2)=0$. Since inputs have no bias, we put
$b^\Net(u) = \bot$ for any input $u$.

Depending on the application, we may want to enlarge $\Upsnet$
with some additional parameters.  For example, we can use
additional weight constant symbols to provide input values to be
evaluated, or output values to be compared with, or interval
bounds, etc.

The logic $\wal$ over the vocabulary
$\Upsnet$ (possibly enlarged as just mentioned) serves as a
``white-box'' query language for neural networks, since the
entire model is given and can be directly queried, just like an
SQL query can be evaluated on a given relational database.
Contrast this with the language
$\bbl$ from Section~\ref{secbbl}, which only has access to the function
$F$ represented by the network, as a black box.

\begin{example} \label{exeval}

  While the language $\bbl$ cannot see inside the model, at
  least it has direct access to the function represented by the
  model.  When we use the language $\wal$, we must compute this function 
  ourselves.  At least when we know the depth of the network, this is
  indeed easy.  In the Introduction, we
  already showed a weight term expressing the evaluation of a
  one-layer neural network on a single input and output.  We can
  easily generalize this to a weight term expressing the value of
  any of a fixed number of outputs, with any fixed number $m$ of
  inputs, and any fixed number of layers.  Let $\val_1$, \dots,
  $\val_m$ be additional weight constant symbols representing
  input values.  Then the weight term $
  \ReLU(b(u)+w(\inn_1,u)\cdot \val_1+\cdots+w(\inn_m,u)\cdot
  \val_m) $ expresses the value of any neuron $u$ in the first
  hidden layer ($u$ is a variable).  Denote this term by
  $t_1(u)$.  Next, for any subsequent layer numbered $l>1$, we
  inductively define the weight term $t_l(u)$ as \[ \ReLU(b(u) +
  \sum_{x : E(x,u)} w(x,u)\cdot t_{l-1}(x)). \] Here, $\ReLU(c)$
  can be taken to be the weight term \textsf{if $c>0$ then $c$
  else 0}.  Finally, the value of the $j$th output is 
  given by the weight term $\eval_j :=
  b(\out_j)+\sum_{x:E(x,\out_j)} w(x,\out_j)\cdot
  t_l(x)$, where $l$ is the number of the last hidden layer.
\end{example}

\begin{example} \label{exuseless}

We can also look for useless neurons: neurons that can
be removed from the network without altering the output too much
on given values.  Recall the weight term $\eval_j$ from the
previous example; for clarity we just write $\eval$.  Let $z$ be
a fresh variable, and let $\eval'$ be the term obtained from
  $\eval$ by altering the summing conditions $E(x,u)$ and
  $E(x,\out)$ by adding the conjunct $x \neq z$.
  Then the formula $|\eval
- \eval'| < \epsilon$ expresses that $z$ is useless.  (For $|c|$
  we can take the weight term \textsf{if $c>0$ then $c$ else
  $-c$}.)

\end{example}

Another interesting example is computing integrals.  Recall
that $\K(m,\ell)$ is the class of networks with $m$ inputs, one
output, and depth $\ell$.

\begin{lemma}
  \label{lemintm}
  Let $m$ and $\ell$ be natural numbers.
There exists an $\wal$ term $t$ over $\Upsnet(m,1)$ with 
  $m$ additional pairs of weight constant symbols $\mathit{min}_i$ and
  $\mathit{max}_i$ for $i \in \{ 1, \dots, m\}$, such that for
  any network $\Net$ in $\K(m,\ell)$, and values $a_i$ and $b_i$ for
  the $\mathit{min}_i$ and $\mathit{max}_i$, we have
  $
  t^{\Net,a_1,b_1,\dots,a_m,b_m} = \int_{a_1}^{b_1}\cdots
  \int_{a_m}^{b_m} \Func\Net\,dx_1\dots dx_m $.
\end{lemma}
\begin{proof}[Proof (sketch)]
  \let\qed\relax
  The proof for the full result is delayed to 
Appendix~\ref{secproofintmd}, since it needs results and notions
  from Section~\ref{sectheorem}.  However, we can sketch here a
  self-contained and elementary proof for $m=1$ and
  $\ell=2$ (one input, one hidden layer).  This case
  already covers all continuous piecewise linear functions
  $\R\to\R$.

  Every hidden neuron $u$ may represent a
  ``quasi breakpoint'' in the piecewise linear function (that is, a point where its slope may change). Concretely, we consider the hidden neurons with nonzero input weights to avoid dividing by zero. Its $x$-coordinate
  is given by the weight term $\breakx(u) :=
  -b(u)/w(\inn_1,u)$.  The $y$-value at the breakpoint
  is then given by $\breaky(u) \coloneqq \eval_1(\breakx(u))$, where $\eval_1$ is the
  weight term from Example~\ref{exeval} and we substitute 
  $\breakx(u)$ for $\val_1$.

  Pairs $(u_1,u_2)$ of neurons representing successive
  breakpoints are easy to define by a formula
  $\breakNextF(u_1,u_2)$.  Such pairs represent the pieces of
  the function, except for the very first and very last pieces.
  For this proof sketch, assume we simply want the integral
  between the first breakpoint and the last breakpoint.

  The area (positive or negative)
  contributed to the integral by the piece $(u_1,u_2)$
  is easy to write as a
  weight term: $\surface = \frac12 (\breaky(u_1) + \breaky(u_2))(\breakx(u_2) - \breakx(u_1))$.
  We sum these
  to obtain the desired integral.  However,
  since different neurons may represent the same quasi breakpoint, we
  must divide by the number of duplicates.
  Hence, our desired term $t$ equals \(
  \sum_{u_1,u_2 : \breakNextF(u_1,u_2)}
  \surface/
  (\sum_{u'_1,u'_2 : \gamma} 1), \)
  where $\gamma$ is the formula
  $ \breakNextF(u'_1,u'_2) \land
  \breakx(u'_1)=\breakx(u_1) \land
  \breakx(u'_2)=\breakx(u_2)$.
\end{proof}

\begin{example}
  A popular alternative to Example~\ref{excontrib} for measuring
the contribution of an input feature $i$ to an input $\bvec y = (y_1, \ldots, y_m)$ is the $\Shap$ score \cite{lund17}. It assumes a probability distribution $\mathbb{P}$ on the input space and quantifies the change to the expected value of $\Func\Net$ caused by fixing input feature $i$ to $y_i$ in a random fixation order of the input features:
	\[
	\Shap(i) = \mkern-6mu \sum_{I \subseteq \{1,\ldots,m\}\setminus\{i\}} \mkern-6mu 
	\frac{\lvert I \rvert ! (m - 1 - \lvert I \rvert !)}{m!} 
	\Big(
	\mathbb{E}\big(\Func\Net(\bvec x)\mid\bvec x_{I\cup\{i\}} = \bvec y_{I\cup\{i\}}\big)
	- 
	\mathbb{E}\big(\Func\Net(\bvec x)\mid\bvec x_{I} = \bvec y_{I}\big)
	\Big).
	\]
	When we assume that $\mathbb{P}$ is the product of uniform distributions over the intervals $(a_j, b_j)$, we can write the conditional expectation
        $\mathbb{E}\big(\Func\Net(\bvec x)\mid\bvec x_{J} = \bvec y_{J}\big)$ for some $J \subseteq \{1, \ldots, m\}$ by setting $\{1, \ldots m\}\setminus J \eqqcolon \{j_{1}, \ldots, j_r\}$ as follows.
	\[
	\mathbb{E}\big(\Func\Net(\bvec x)\mid\bvec x_{J} = \bvec y_{J}\big) 
	= \
	\prod_{k = 1}^r \frac{1}{b_{j_k} - a_{j_k}} \cdot \int_{a_{j_1}}^{b_{j_1}} \cdots \int_{a_{j_r}}^{b_{j_r}}
	\Func\Net(\bvec x \vert_{\bvec x_J = \bvec y_J})\, 
	d x_{j_r} \ldots d x_{j_1}
	\]
	where $\bvec x \vert_{\bvec x_J = \bvec y_J}$ is a short notation for the variable obtained from $\bvec x$ by replacing $x_j$ with $y_j$ for all $j \in J$. 
	With lemma \ref{lemintm}, this conditional expectation can be expressed in $\wal$ and by replacing $J$ with $I$ or $I\cup\{i\}$ respectively, we can express the $\Shap$ score.
\end{example}

\subparagraph{More examples}  Our main result will be that, over
networks of a given depth, all of $\bblin$
can be expressed
in $\wal$.  So the examples from Section~\ref{secbbl}
(which are linear if a Manhattan or max distance is used)
apply here as well.  Moreover, the techniques by which we show
our main result readily adapt to queries not about the final
function $F$ represented by the network, but about the function
$F_z$ represented by a neuron $z$ given as a parameter to the
query, much as in Example~\ref{exuseless}.  For example, in
feature visualization \cite{molnar-book} we want to find the
input that maximizes the activation of some neuron $z$.  Since
this is expressible in $\bblin$, it is also expressible in $\wal$.

\section{Model-agnostic queries} \label{secagnostic}

We have already indicated that $\bbl$ is ``black box'' while
$\wal$ is ``white box''.  Black-box queries are
commonly called \emph{model agnostic} \cite{molnar-book}.
Some $\wal$ queries may, and others may not, be model agnostic.

Formally, for some $\ell\ge 1$,
let us call a closed $\wal$ formula $\varphi$,
possibly using weight constants $c_1,\ldots,c_k$,
\emph{depth-$\ell$ model agnostic} if for all $m\ge 1$ all neural
networks $\Net,\Net'\in\bigcup_{i=1}^\ell\K(m,i)$ such that
$F^{\Net}=F^{\Net'}$, and all $a_1,\ldots,a_k\in\R$ we have
$\Net,a_1,\ldots,a_k\models\varphi$ $\Leftrightarrow$
$\Net',a_1,\ldots,a_k\models\varphi$.
A similar definition applies to closed $\wal$ weight terms.

For example, the term of Example~\ref{exeval} evaluating the
function of a neural network of depth at most $\ell$ is
depth-$\ell$ model agnostic. By comparison the formula stating
that a network has useless neurons (cf.~Example~\ref{exuseless}) is not model agnostic.
The term $t$ from Lemma~\ref{lemintm}, computing the integral, is
depth-$\ell$ model agnostic.

\begin{theorem}\label{thmint}
  The query $\int_0^1f=0$ for functions $f\in\PL(1)$  is
  expressible by a depth-$2$ agnostic $\wal$ formula, but not in
  $\bbl$.
\end{theorem}

\begin{proof}
  We have already seen the expressibility in $\wal$.
We prove nonexpressibility in $\bbl$.

  Consider the equal-cardinality
  query $Q$ about disjoint
  pairs $(S_1,S_2)$ of finite sets of reals, asking whether
  $|S_1|=|S_2|$.  Over \emph{abstract} ordered finite structures, equal
  cardinality is well-known
  not to be expressible in order-invariant first-order logic
  \cite{libkin_fmt}.  Hence, by the generic collapse theorem
  for constraint query languages over the reals
  \cite{cdbbook,libkin_fmt}, query $Q$ is not expressible
  in $\form{S_1,S_2}$.

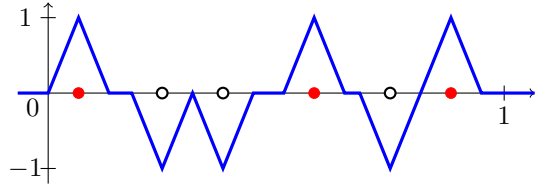
\begin{figure}
  \centering
    \begin{tikzpicture}[
            rdot/.style={circle,red,draw,fill,inner sep=0pt,minimum size=4pt},
            bdot/.style={circle,draw,thick,fill=white,inner sep=0pt,minimum size=4pt},
            scale=1
        ]
        \draw[->] (-0.4,0) -> (6.4,0);
        \draw[->] (0,-1.2) -> (0,1.2); 
        \draw (-0.1,1) -- (0.1,1) (-0.1,-1) -- (0.1,-1) (6,-0.1)--(6,0.1);
        \path (-0.3,1) node {$1$} (-0.3,-1) node {$-1$} (-0.2,-0.2) node {$0$} (6,-0.3) node {$1$};

        \path (0.4,0) node[rdot] {} (1.5,0) node[bdot] {} (2.3,0) node[bdot] {} (3.5,0) node[rdot] {} (4.5,0) node[bdot] {} (5.3,0) node[rdot] {};
        \draw[very thick,blue] (-0.4,0) -- (0,0) -- (0.4,1) -- (0.8,0) -- (1.1,0) -- (1.5,-1) -- (1.9,0) -- (2.3,-1) -- (2.7,0) -- (3.1,0) -- (3.5,1) -- (3.9,0) -- (4.1,0) -- (4.5,-1) -- (5.3,1) -- (5.7,0) -- (6.4,0);
    \end{tikzpicture}
    \caption{The function $f_{S_1,S_2}$ of the proof of Theorem~\ref{thmint} for the set $S_1$ consisting of the three red points and the set $S_2$ consisting of the three white points}
\end{figure}

  Now for any given $S_1$ and $S_2$, we construct a continuous piecewise
  linear function $f_{S_1,S_2}$ as follows.  We first apply a
  suitable
  affine transformation so that $S_1 \cup S_2$ falls within the
  open interval $(0,1)$.
  Now $f_{S_1,S_2}$ is a
  sawtooth-like function, with positive teeth at elements from
  $S_1$, negative teeth (of the same height, say 1)
  at elements from $S_2$, and zero everywhere else.
  To avoid teeth that overlap the zero
  boundary at the left or that overlap each other, we make
  them of width $\min\{m,M\}/2$, where $m$ is the minimum of
  $S_1\cup S_2$ and $M$ is the minimum distance between any two
  distinct elements in $S_1 \cup S_2$.

  Expressing the above construction uniformly in
  $\form{S_1,S_2}$ poses no difficulties; let $\psi(x,y)$ be a
  formula defining $f_{S_1,S_2}$.  Now assume, for the sake of
  contradiction, that $\int_0^1F=0$ would be expressible by a
  closed $\bbl$
  formula $\varphi$.  Then composing $\varphi$ with $\psi$ would
  express query $Q$ in $\form{S_1,S_2}$.
  Indeed, clearly, $\int_0^1f_{S_1,S_2}=0$ if and only if $|S_1|=|S_2|$.
  So, $\varphi$ cannot exist.
\end{proof}

It seems awkward that in the definition of model agnosticity we
need to bound the depth. Let us call an $\wal$ term or formula
\emph{fully model agnostic} if is depth-$\ell$ model agnostic for
every $\ell$. It turns out that there are no nontrivial fully
model agnostic $\wal$ formulas.

\begin{theorem}\label{thm:fully_agnostic}
  Let $\varphi$ be a fully model agnostic closed $\wal$
  formula over $\Upsnet(m,1)$. Then either $\Net\models\varphi$ for all $\Net\in\bigcup_{\ell\ge 1}\K(m,\ell)$ or $\Net\not\models\varphi$ for all $\Net\in\bigcup_{\ell\ge 1}\K(m,\ell)$.
\end{theorem}

The proof is in Appendix \ref{sec:proof_fully_agnostic}. The idea
is that $\wal$ is Hanf-local \cite{libkin_sql,libkin_fmt}. No
formula $\varphi$ can distinguish a long enough structures consisting of
two chains where the middle nodes are marked by two distinct
constants $c_1$ and $c_2$, from its sibling structure where the
markings are swapped.  We can turn the two structures into neural
networks by replacing the markings by two gadget networks
$N_1$ and $N_2$, representing different functions, that $\varphi$
is supposed to distinguish.  However, the construction is
done so that the function represented by the
structure is the same as that represented by the gadget in the
left chain.  Still, $\wal$ cannot distinguish these two
structures.  So, $\varphi$ is either not fully model-agnostic, or
$N_1$ and $N_2$ cannot exist and $\varphi$ is trivial.

\begin{corollary} \label{corzero}
  The $\bbl$ query $F(0)=0$ is not expressible in $\wal$.
\end{corollary}

\section{From $\foRlin$ to $\wal$} \label{sectheorem}

In practice, the number of layers in the
employed neural network architecture is often fixed and known.
Our main result then is that $\wal$ can express all $\foRlin$
queries.

\begin{theorem} \label{theormain}
  Let $m$ and $\ell$ be natural numbers.  For every closed
  $\bblin$ formula
  $\psi$ there exists a closed $\wal$ formula $\varphi$ such that for
  every network $\Net$ in $\K(m,\ell)$, we have $ \Rbf,\Func\Net
  \models \psi $ iff $ \Net \models \varphi $.
\end{theorem}
The challenge in proving this result is to simulate, using
quantification and summation over neurons, the
unrestricted access to real numbers that is available in
$\foRlin$.  Thereto, we will divide the relevant real space in a
finite number of cells which we can represent by finite tuples of
neurons.

The proof involves several steps that transform weighted
structures.  Before presenting the proof,
we formalize such transformations in the notion of
$\wal$ translation, which
generalize the classical notion of first-order interpretation
\cite{hodges} to weighted structures.

\subsection{$\wal$ translations}

Let $\Upsilon$ and $\Gamma$ be vocabularies for weighted
structures,
and let $n$ be a natural number.
An \emph{$n$-ary $\wal$ translation $\varphi$ from $\Upsilon$ to
$\Gamma$} consists of a number of formulas and weight terms over
$\Upsilon$, described next.  There are formulas $\domF(\bvec{x})$
and $\eqF(\bvec{x}_1,\bvec{x}_2)$; formulas
$\varphi_R(\bvec{x}_1,\dots,\bvec{x}_k)$ for every $k$-ary relation
symbol $R$ of $\Gamma$; and formulas
$\varphi_f(\bvec x_0,\bvec{x}_1,\dots,\bvec{x}_k)$ for every
$k$-ary standard function symbol $f$ of $\Gamma$.
Furthermore, there are weight terms
$\varphi_w(\bvec{x}_1,\dots,\bvec{x}_k)$ for every $k$-ary weight
function $k$ of $\Gamma$.

In the above description, bold $\bvec x$ denote
$n$-tuples of distinct variables.  Thus, the formulas
and weight terms of $\varphi$ define relations or weight
functions of arities that are a multiple of $n$.

We say that $\varphi$ maps a weighted structure $\structA$ over
$\Upsilon$ to a weighted structure $\structB$ over $\Gamma$ if
there exists a surjective function $h$ from $\domF({\structA})
\subseteq A^n$ to $B$ such that:
\begin{itemize}

\item \( h(\bvec{a}_1) = h(\bvec{a}_2) \Leftrightarrow \structA
  \models \eqF(\bvec{a}_1, \bvec{a}_2) \);

\item \( (h(\bvec{a}_1), \dots, h(\bvec{a}_k)) \in R^{\structB}
  \Leftrightarrow \structA \models \varphi_R(\bvec{a}_1, \dots,
    \bvec{a}_k) \);

  \item \( (h(\bvec a_0) =
    f^{\B}(h(\bvec{a}_1), \dots, h(\bvec{a}_k))
  \Leftrightarrow \structA \models \varphi_f(\bvec a_0,\bvec{a}_1, \dots,
    \bvec{a}_k) \);

\item \( w^{\structB}(h(\bvec{a}_1), \dots, h(\bvec{a}_m)) =
  \varphi_{w}^{\structA}(\bvec{a}_1, \dots, \bvec{a}_n) \).

\end{itemize}
In the above, the bold $\bvec a$ denote $n$-tuples in $\domF(\A)$.

For any given $\A$, if $\varphi$ maps $\A$ to $\B$, then $\B$
is unique up to isomorphism. Indeed,
the elements of $B$ can be understood as
representing the equivalence classes of the equivalence relation
$\eqF(\A)$ on $\domF(\A)$. In particular, for
$\structB$ to exist, $\varphi$ must be \emph{admissible} on $\A$,
which means that $\eqF(\A)$ is indeed an equivalence relation on
$\domF({\structA})$, and all relations and all functions
$\varphi_R(\A)$, $\varphi_f(\A)$ and $\varphi_w(\A)$ are
invariant under this equivalence relation.

If $\mathbf K$ is a class of structures over $\Upsilon$, and $T$
is a transformation of structures in $\mathbf K$ to structures
over $\Gamma$, we say that $\varphi$ \emph{expresses} $T$ if
$\varphi$ is admissible on every $\A$ in $\mathbf K$, and maps
$\A$ to $T(\A)$.

The relevant reduction theorem for translations is the following:

\begin{theorem}

  Let $\varphi$ be an $n$-ary $\wal$ translation from $\Upsilon$
  to $\Gamma$, and let $\psi(y_1,\dots,y_k)$ be a formula
  over $\Gamma$. Then there exists a formula
  $\varphi_\psi(\bvec{x}_1,\dots, \bvec{x}_k)$ over $\Upsilon$ such
  that whenever $\varphi$ maps $\A$ to $\B$ through $h$, we have
  $\structB \models \psi(h(\bvec{a}_1),\dots, h(\bvec{a}_k))$ iff
  $\structA \models \varphi_\psi(\bvec{a}_1,\dots, \bvec{a}_k)$.
  Furthermore, for any weight term $t$ over $\Gamma$, there exists
  a weight term $\varphi_t$ over $\Upsilon$ such that
  $t^\structB(h(\bvec{a}_1),\dots, h(\bvec{a}_k)) =
  \varphi_t^{\structA}(\bvec{a}_1,\dots, \bvec{a}_k)$.
\end{theorem}

\begin{proof}[Proof (sketch)]
  \let\qed\relax
  As this result
  is well known and straightforward to prove for classical
  first-order interpretations,
  we only deal here with summation terms, which are the main new aspect.
  Let $t$ be of the form $\sum_{y : \gamma} t'$.
  Then for $\varphi_t$ we take
  \( \sum_{\bvec x : \varphi_\gamma}
  {\varphi_{t'}(\bvec x_1,\dots,\bvec x_k,\bvec x)}
  /
  (\sum_{\bvec x' : \eqF(\bvec x,\bvec x')} 1)\).
\end{proof}

\subsection{Proof of Theorem~\ref{theormain}}
\label{subproofmainthm}

We sketch the proof of Theorem~\ref{theormain}.  For clarity of
exposition, we present it first for single inputs, i.e., the case $m=1$.
We present three Lemmas which can be chained together to obtain
the theorem.

\subparagraph*{Piecewise linear functions}

We can naturally model piecewise linear (PWL) functions from $\R$
to $\R$ as weighted structures, where the elements are simply the
pieces.  Each piece $p$ is defined by a line $y=ax+b$ and left
and right endpoints.  Accordingly, we use a vocabulary
$\Upspwl_1$ with four unary weight functions indicating $a$, $b$,
and the $x$-coordinates of the endpoints.  (The left- and
rightmost pieces have no endpoint; we set their
$x$-coordinate to $\bot$.)

For $m=1$ and $\ell=2$, the proof of the following Lemma is based on the same
ideas as in the proof sketch we gave for Lemma~\ref{lemintm}.
For $m>1$, PWL functions from $\R^m$ to $\R$ are more complex;
the vocabulary $\Upspwl_m$ and a general proof of the lemma
will be described in Section~\ref{submultiple}.

\begin{lemma} \label{lempwl}
  Let $m$ and $\ell$  be natural numbers.
  There is an $\wal$ translation from $\Upsnet(m,1)$ to
  $\Upspwl_m$ that transforms every network $\Net$ in $\K(m,\ell)$ into a proper
  weighted structure representing $\Func\Net$.
\end{lemma}

\subparagraph*{Hyperplane arrangements}

An \emph{affine function} on $\R^d$ is a function of the form
$a_0+a_1x_1+\cdots+a_dx_d$.  An affine \emph{hyperplane} is the
set of zeros of some  non-constant affine function (i.e. where
at least one of the $a_i$ with $i>0$ is non-zero). A \emph{hyperplane
arrangement} is a collection of affine hyperplanes.

We naturally model a hyperplane arrangement as a weighted
structure, where the elements are the hyperplanes.  The
vocabulary $\hpArrV{d}$ simply consists of unary weight functions
$a_0$, $a_1$, \dots, $a_d$ indicating the coefficients of the
affine function defining each hyperplane.

\begin{remark}
  \label{remhpdup}
  An $\hpArrV{d}$-structure may have duplicates, i.e., different elements representing the same hyperplane. This happens when they have the same coefficients up to a constant factor. In our development, we will allow structures with duplicates as representations of hyperplane arrangements.
\end{remark}

\subparagraph*{Cylindrical decomposition}

We will make use of a linear version of the notion of cylindrical
decomposition (CD) \cite{basu_algorithms}, which we call
\emph{affine CD}\@.  An affine CD of $\R^d$ is a sequence
$\D=\D_0,\dots,\D_d$, where each $\D_i$ is a partition of $\R^i$.
The blocks of partition $\D_i$ are referred to as $i$-cells or
simply cells.  The precise definition is by induction on $d$.
For the base case, there is only one possibility $\D_0 =
\{\R^0\}$.  Now let $d>0$.  Then $\D_0,\dots,\D_{d-1}$ should
already be an affine CD of $\R^{d-1}$.  Furthermore, for every
cell $C$ of $\D_{d-1}$, there must exist finitely many affine
functions $\xi_1$, \dots, $\xi_r$ from $\R^{d-1}$ to $\R$,
where $r$ may depend on $C$.  These are called the \emph{section
mappings} above $C$, and must satisfy $\xi_1<\cdots<\xi_r$ on 
$C$.  In this way, the section mappings induce a partition of the
cylinder $C \times \R$ in \emph{sections} and \emph{sectors}.
Each section is the graph of a section mapping, restricted to
$C$.  Each sector is the volume above $C$ between two consecutive
sections.  Now $\D_d$ must equal $\{C\times S \mid C \in
\D_{d-1}$ and $S$ is a section or sector above $C\}$.

The ordered sequence of cells $C \times S$ formed by the
sectors and sections of $C$ is called the \emph{stack above $C$},
and $C$ is called the \emph{base cell} for these cells.

An affine CD of $\R^d$ is \emph{compatible} with a hyperplane
arrangement $\A$ in $\R^d$ if every every $d$-cell $C$ lies
entirely on, or above, or below every hyperplane $h=0$.
(Formally, the affine function $h$ is everywhere zero, or everywhere
positive, or everywhere negative, on $C$.)

We can represent a CD compatible with a hyperplane arrangement as
a weighted structure with elements of two sorts: cells and
hyperplanes.  There is a constant $o$ for the ``origin cell''
$\R^0$.  Binary relations link every $i+1$-cell to its base
$i$-cell, and to its delineating section mappings.  (Sections are viewed
as degenerate sectors where the two delineating section mappings are
identical.) Ternary relations give the order of two hyperplanes in $\R_{i+1}$
above an $i$-cell, and whether they are equal. The vocabulary for CDs of $\R^d$ is
denoted by $\Upscell_d$.

\begin{lemma} \label{lemcd}
  Let $d$ be a natural number.
  There is an $\wal$ translation from $\Upsarr_d$ to $\Upscell_d$
  that maps any hyperplane arrangement $\A$ to a CD that is
  compatible with $\A$.
\end{lemma}
\begin{proof}[Proof (sketch)]
  We follow the method of vertical decomposition
  \cite{halperin_arr}.
  There is a projection phase, followed
  by a buildup phase.  For the projection phase, let $\Arr_d:=\A$.
  For $i=d,\dots,1$, take all intersections between hyperplanes in
  $\Arr_i$, and project one dimension down, i.e., project on the
  first $i-1$ components.  The result is a hyperplane arrangement
  $\A_{i-1}$ in $\R^{i-1}$.  For the buildup phase, let $\D_0 :=
  \{\R^0\}$.  For $i=0,\dots,d-1$, build a stack above every cell
  $C$ in $\D_i$ formed by intersecting $C \times \R$ with all
  hyperplanes in $\Arr_{i+1}$.  The result is a partition
  $\D_{i+1}$ such that $\D_0,\dots,\D_{i+1}$ is a CD of
  $\R^{i+1}$ compatible with $\Arr_{i+1}$.  This algorithm is
  implementable in $\wal$.
\end{proof}

\subparagraph*{Ordered formulas and cell selection}

Let $\psi$ be the $\bblin$ formula under consideration.
Let $x_1,\dots,x_d$ be an enumeration
of the set of variables in $\psi$, free or bound.  We may
assume that $\psi$ is in prenex normal form $Q_{1}x_{1} \dots
Q_dx_d\, \chi$, where each $Q_i$ is
$\exists$ or $\forall$, and $\chi$ is quantifier-free.

We will furthermore assume that $\psi$ is \emph{ordered}, meaning
that every atomic subformula is of the form
$F(x_{i_1},\dots,x_{i_m})=x_j$ with
$i_1<\cdots<{i_m}<j$, or is a \emph{linear constraint} of the form
$a_0+a_1x_1+\cdots+a_dx_d > 0$. By using
extra variables, every $\bblin$ formula can be brought in ordered
normal form.

Consider a PWL function $f : \R \to \R$.  Every
piece is a segment of a line $ax+b=y$ in $\R^2$.
We define the \emph{hyperplane arrangement
corresponding to $f$ in $d$ dimensions} to consist of all
hyperplanes $ax_i+b=x_j$, for all lines $ax+b=y$ of $f$, where $i
< j$ (in line with the ordered assumption on the formula
$\psi$).  We denote this arrangement by $\Arr_f$.  

Also the query $\psi$ gives rise to a hyperplane
arrangement, denoted by $\Arr_\psi$, which
simply consisting of all hyperplanes corresponding to the linear
constraints in $\psi$.

For the following statement, we use the disjoint union $\uplus$
of two weighted structures.  Such a disjoint union can itself be
represented as a weighted structure over the disjoint union of
the two vocabularies, with two extra unary relations to
distinguish the two domains.

\begin{lemma} \label{lemsel}

  Let $\psi \equiv Q_{1}x_{1} \dots Q_dx_d\, \chi$
  be an ordered closed $\bblin$ formula
  with function symbol $F$ of arity $m$.
  Let $k \in \{0,\dots,d\}$, and let $\psi_k$
  be $Q_{k+1}x_{k+1} \dots Q_dx_d\, \chi$.  There exists a unary
  $\wal$ query over $\Upspwl_m \uplus \Upscell_d$ that
  returns, on any piecewise linear function $f:\R^m\to\R$
  and any
  CD $\D$ of $\R^d$ compatible with $\Arr_f \cup
  \Arr_\psi$, a set of cells in $\R^k$ whose union equals \(
    \{(v_1,\dots,v_k) \mid \Rbf,f \models \psi(v_1,\dots,v_k)\}
    \).
\end{lemma}
\begin{proof}[Proof (sketch)]
  As already mentioned we focus first on $m=1$.
  The proof is
  by downward induction on $k$.  The base case $k=d$ deals with the
  quantifier-free part of $\psi$.  We focus on the atomic
  subformulas. Subformulas of the form
  $F(x_i)=x_j$ are dealt with as follows.
  For every piece $p$ of $f$, with line
  $y=ax+b$, select all $i$-cells
  where $x_i$ lies between $p$'s endpoints.  For each such
  cell, repeatedly take all cells in the stacks above it
  until we reach $j-1$-cells.  Now for each of these
  cells, take the section in its stack given by the section mapping
  $x_j=ax_i+b$.  For each of these sections, again repeatedly
  take all cells in the stacks above it until we reach $d$-cells.
  Denote the obtained set of $d$-cells by $S_p$; the desired
  set of cells is $\bigcup_p S_p$.

  Subformulas that are linear constraints,
  where $i$ is the largest index such that $a_1$ is nonzero,
  can be dealt with by taking, above every
  $i-1$-cell all sections that lie above the
  hyperplane corresponding to the constraint, if $a_i>0$, or,
  if $a_i < 0$, all sections that lie below it.
  The described algorithm for the quantifier-free part can
  be implemented in $\wal$.

  For the inductive case, if $Q_{k+1}$ is $\exists$, we must show
  that we can project a set of cells down one dimension, which is
  easy given the cylindrical nature of the decomposition; we just
  move to the underlying base cells.  If $Q_{k+1}$ is $\forall$,
  we treat it as $\neg \exists \neg$, so we complement the
  current set of cells, project down, and complement again.
\end{proof}

To conclude, let us summarise the structure of the whole proof.
We are given a neural network $\Net$ in $\K(m,\ell)$, and we want
to evaluate a closed $\bblin$ formula $\psi$. We assume the
query to be in prenex normal form and ordered. We start with an
interpretation that transforms $\Net$ to a structure representing
the piecewise linear function $\Func{\Net}$ (Lemma~\ref{lempwl}).
Then, using another interpretation, we expand the structure by
the hyperplane arrangement obtained from the linear pieces of
$\Func{\Net}$ as well as the query. Using Lemma~\ref{lemcd}, we
expand the current structure by a cell decomposition compatible
with the hyperplane arrangement. Finally, using
Lemma~\ref{lemsel} we inductively process the query on this cell
decomposition, at each step selecting the cells representing all
tuples satisfying the current formula. Since the formula $\psi$
is closed, we eventually either get the single $0$-dimensional
cell, in which case $\psi$ holds, or the empty set, in which case
$\psi$ does not hold.

\subsection{Extension to multiple inputs} \label{submultiple}

For $m>1$, the notion of PWL function $f:\R^m\to\R$ is more
complex. We can conceptualize our representation of $f$ as a decomposition of 
$\R^m$ into polytopes where, additionally, every polytope $p$ is accompanied by an
affine function $f_p$ such that $f = \bigcup_p f_p|_p$. We call $f_p$ the \emph{component function} of $f$ on $p$.
Where for $m=1$ each pies of piece $f$ was delineated by just two breakpoints, now our polytope in $\R^m$
may be delineated by many hyperplanes, called breakplanes. Thus, the vocabulary $\PWLV{m}$ includes
the position of a polytope relative to the breakplanes, 
indicating whether the polytope is on the breakplane, or on the positive or negative side of it.
We next sketch how to prove Lemma~\ref{lempwl} in its generality.  The
proof of Lemma~\ref{lemsel} poses no additional problems.



We will define a
PWL function $f_u$ for every neuron $u$ in the network; the final result is then
$f_{\rm out}$.  To represent these functions for
every neuron, we simply add one extra relation symbol, indicating to
which function each element of a $\PWLV{m}$-structure belongs.
The construction is by induction on the layer number.
At the base of the induction are the input neurons.  The $i$-th
input neuron defines the PWL functions where there is only one polytope
($\R^m$ itself), whose section mapping is the function $\bvec x
\mapsto x_i$.  

\textit{Scaling:}
For any hidden neuron $u$ and incoming edge $v\to u$ with weight
$w$, we define an auxiliary function $f_{v,u}$ which simply scales $f_v$ by
$w$.

To represent the function defined by $u$, we need to
sum the $f_{v,u}$'s and add $u$'s bias; and apply ReLU\@.  We
describe these two steps below, which can be implemented in
$\wal$.  For $u=\rm out$, the ReLU step is omitted.

 \textit{Summing:}
 For each $v\to u$, let $\D_{v,u}$ be the CD
    for $f_{v,u}$, and let $\Arr_{v,u}$ be the set of hyperplanes
    in $\R^m$ that led to $\D_{v,u}$.  We define the arrangements
    $\Arr_u = \bigcup_v \Arr_{v,u}$ and $\Arr = \bigcup_u \Arr_{u}$.
    We apply Lemma~\ref{lemcd} to $\Arr$ to obtain a CD
    $\D$ of $\Arr$, which is also compatible with each $\Arr_u$.
    Every $m$-cell $C$ in $\D$ is contained in 
    a unique polytope $p_{v,u}^{C} \in f_{v,u}$ for every $v\to u$.  We can define $p_{v,u}^{C}$
    as the polytope that is positioned the same with respect to
    the hyperplanes in $\Arr_{v,u}$ as $C$ is. 
    Two $m$-cells $C$ and $C'$ are called $u$-equivalent if $p_{v,u}^{C} = p_{v,u}^{C'}$
    for every $v \to u$. We can partition $\R^m$ in polytopes formed by merging
    each $u$-equivalence class $[C]$. Over this partition we define a PWL function $g_u$.
    On each equivalence class $[C]$, we define $g_u$ as $\sum_{v \to u} f_{p_{v,u}^{C}}$,
    plus $u$'s bias. The constructed function $g_u$ equals $b(u)+\sum_v f_{u,v}$.

\textit{ReLU:} 
To represent $\ReLU(g_u)$, we construct the new
    arrangements $\B_u$ formed by the union of $\Arr_u$ with all
    hyperplanes given by component functions of $g_u$, and $\B = \bigcup_u \B_u$.  
    Again apply Lemma~\ref{lemcd} to $\B$ to obtain a CD
    $\E$ of $\B$, which is compatible with each $B_u$. 
    Again every $m$-cell $C$ in $\E$ is
    contained in a unique polytope $p_{u}^{C}$ of $g_u$ with respect to $\Arr_u$.
    Now two $m$-cells $C$ and $C'$ are called strongly $u$-equivalent if 
    they are positioned the same with respect to the hyperplanes in $\B_u$.
    This implies $p_{u}^{C} = p_{u}^{C'}$ but is stronger.
    We can partition $\R^m$ in polytopes formed by merging
    each $u$-equivalence class $[C]$. Over this partition 
    we define a PWL function ${f_u}'$.
    Let $\xi_{u}^{C}$ be the component function of $g_u$ on $p_{u}^{C}$.
    On each equivalence class $[C]$, we define ${f_u}'$ as $\xi_{u}^{C}$ if it
    is positive on $C$; otherwise it is set to be zero.
    The constructed function ${f_u}'$ equals $f_{u}$ as desired.





\section{Conclusion} \label{seconc}

The immediate motivation for this research is explainability and
the verification of machine learning models. In this sense, our
paper can be read as an application to machine learning of
classical query languages known from database theory. The novelty
compared to earlier proposals
\cite{arenas-foil,lorini-logic-explain} is our focus on
real-valued weights and input and output features.  More
speculatively, we may envision machine learning models as genuine
data sources, maybe in combination with more standard databases,
and we want to provide a uniform interface.  For example,
practical applications of large language models will conceivably
also need to store a lot of hard facts. However, just being able
to query them through natural-language prompts may be suboptimal
for integrating them into larger systems. Thus query languages
for machine learning models may become a highly relevant research
direction.

$\wal$ queries will be likely very complex, so our result opens
up challenges for query processing of complex,
analytical SQL queries.  Such queries are at the focus
of much current database systems research, and supported by
recent systems such as DuckDB~\cite{duckdb} and Umbra~\cite{umbra}.
It remains to be investigated to what extent
white-box querying can be made useful in practice.  The
construction of a cell decomposition of variable space turned out crucial
in the proof of our main result.  Such cell decompositions might
be preproduced by a query processor as a novel kind of index data
structure.

While the language $\foR$ should mainly be seen as
an expressiveness benchmark, techniques from SMT
solving and linear programming are being adapted in the context
of verifying neural networks \cite{aws-book}.  Given the
challenge, it is conceivable that for specific classes of
applications, $\foR$ querying can be made practical.

Many follow-up questions remain open.  Does
the simulation of $\foRlin$ by $\wal$ extend to $\foR$?
Importantly, how
about other activations than ReLU \cite{tjeng}?  If we extend $\wal$ with
quantifiers over weights, i.e., real numbers, what is the
expressiveness gain?  Expressing $\foR$ on bounded-depth neural
networks now becomes immediate, but do we get strictly more
expressivity?  Also, to overcome the problem of being unable to
even evaluate neural networks of unbounded depth, it seems
natural to add recursion to $\wal$.  Fixed-point languages with
real arithmetic can be difficult to handle
\cite{BenediktGLS03,GeertsK05}.

The language $\wal$ can work with weighted relational structures
of arbitrary shapes, so it is certainly not restricted to the FNN
architecture.  Thus, looking at other NN architectures is another
direction for further research.  Finally, we mention the question
of designing flexible model query languages where the number of
inputs, or outputs, need not be known in advance
\cite{arenas-foil,arenas-foil2}.

\bibliography{database,extra}

\appendix
\section{Proof of Lemma~\ref{lemintm} for $m=1$ and $\ell = 2$}\label{secproofint1d}

\newcommand{\ifText}{\textsf{if }}
\newcommand{\thenTextNS}{\textsf{then }}
\newcommand{\thenText}{\ \thenTextNS}
\newcommand{\elseTextNS}{\textsf{else }}
\newcommand{\elseText}{\ \elseTextNS}

\newcommand{\breakF}{\textit{is-break}}
\newcommand{\areaCalc}{\textit{area-calc}}
\newcommand{\sameSSurf}{\textit{same-side-area}}
\newcommand{\sameSSurfFirst}{\textit{same-side-area-first}}
\newcommand{\sameSSurfLast}{\textit{same-side-area-last}}
\newcommand{\sameSSurfEnds}{\textit{same-side-area-ends}}
\newcommand{\zeroOffset}{\textit{zeropoint-offset}}
\newcommand{\zeroOffsetFirst}{\textit{zeropnt-offset-first}}
\newcommand{\zeroOffsetEnds}{\textit{zeropnt-offset-ends}}
\newcommand{\diffSSurf}{\textit{diff-side-area}}
\newcommand{\diffSSurfFirst}{\textit{diff-side-area-first}}
\newcommand{\diffSSurfLast}{\textit{diff-side-area-last}}
\newcommand{\diffSSurfEnds}{\textit{diff-side-area-ends}}
\newcommand{\sameSide}{\textit{same-side}}
\newcommand{\sameSideFirst}{\textit{same-side-first}}
\newcommand{\sameSideLast}{\textit{same-side-last}}
\newcommand{\sameSideEnds}{\textit{same-side-ends}}
\newcommand{\breakIB}{\textit{break-in-bounds}}
\newcommand{\firstB}{\textit{first-break}}
\newcommand{\lastB}{\textit{last-break}}
\newcommand{\pieceIB}{\textit{piece-in-bounds}}
\newcommand{\samePiece}{\textit{same-piece}}
\newcommand{\integrateT}{\mathsf{integrate}}

We thought it would be instructive to
prove the following special case separately in detail,
since the proof is elementary and easy to follow.

\begin{lemma}[Repetition of Lemma~\ref{lemintm} for $m=1$ and $\ell = 2$]
  There exists an $\wal$ term $t$ over $\Upsnet(1,2)$ with two
  additional weight constant symbols $\mathit{left}$ and
  $\mathit{right}$ such that for
  any network $\Net$ in $\K(1,2)$ and values $a$ and $b$ for
  $\mathit{left}$ and $\mathit{right}$,
  we have $t^{\Net,a,b} = \int_a^b \Func\Net$.
\end{lemma}

Every hidden neuron $u$ may represent a ``quasi breakpoint'' in the piecewise linear function (that is, a point where its slope may change). Concretely, we consider the hidden neurons with nonzero input weights, to avoid dividing by zero. Its $x$-coordinate is given by the weight term $\breakx(u) := -b(u)/w(\inn_1,u)$.  The $y$-value at the breakpoint is then given by $\breaky(u) \coloneqq \eval_1(\breakx(u))$, where $\eval_1$ is the weight term from Example~\ref{exeval} and we substitute $\breakx(u)$ for $\val_1$.

The area (positive or negative) contributed to the integral by a subinterval $(x_1,x_2)$ lying underneath a single piece is simple to calculate (see the $\areaCalc$ formula below). For a piece $(u_1,u_2)$ lying with the bounds of the given interval $(a,b)$, the contributed area is not simply $\areaCalc(\breakx(u_1),\breakx(u_2))$. Indeed, since different neurons may represent the same quasi breakpoint, we must divide by the number of duplicates. All in all, our desired term for $\int_{a}^{b}F$ equals

\begin{tabbing}
  \quad\=\,\,\,\=\,\,\,\=\,\,\,\=\,\,\,\=\,\,\,\=\,\,\,\=\,\,\,\=\kill  
  \(\integrateT() := \) \\[\jot]
  \> \(\ifText \not\exists u\, \breakIB(u) \thenText \areaCalc(a,b) \) \(\elseText \)\\[\jot]
  \> \> \(\displaystyle \sum_{u_1,u_2:\pieceIB(u_1,u_2)} \frac{\areaCalc(\breakx(u_1),\breakx(u_2))}{\sum_{u_1',u_2':\samePiece(u_1,u_2,u_1',u_2')}1}\) \\[\jot]
  \> \> + \(\displaystyle \sum_{u: \firstB(u)} \frac{\areaCalc(a,\breakx(u))}{\sum_{u':\firstB(u')}1}\) \\[\jot]
  \> \> + \(\displaystyle \sum_{u: \lastB(u)} \frac{\areaCalc(\breakx(u),b)}{\sum_{u':\lastB(u')}1}\)
\end{tabbing}
with the following auxiliary terms and formulas

\begin{tabbing}
  \quad\=\,\,\,\=\,\,\,\=\,\,\,\=\,\,\,\=\,\,\,\=\,\,\,\=\,\,\,\=\kill
  \(\hiddenF(u) := E(\inn_1,u) \) \\[\jot]
  \(\breakF(u) := \hiddenF(u) \land w(\inn_1,u) \neq 0 \) \\[\jot]
  \(\breakx(u) := \frac{-b(u)}{w(\inn_1,u)} \) \\[\jot]
  \(\breaky(u) := \eval_1(\breakx(u))\) \\[\jot]
  \( \breakIB(u) := \)\\[\jot]
  \> \( \breakF(u) \land \breakx(u) > a \land \breakx(u) < b \) 
\end{tabbing}

\begin{figure}
  \centering
  \begin{tikzpicture}
    \begin{axis}[
        xlabel = \(x\),
        xmin=-0.75, xmax=0.3,
        ylabel = \(y\),
        ymin=-0.1, ymax=0.9,
        axis lines = middle,
        width = 0.75\linewidth,
        axis equal image
      ]
      \addplot [
        color = black
      ] coordinates {
          (-0.8,0.55) (-0.6,0.5) (0.1, 0.8) (0.5,0)
        };
      \addplot [
        color = black,
        only marks
      ] coordinates {
          (-0.6,0.5) (0.1, 0.8)
        };
      \node[anchor=north west] at (-0.6, 0.5) {$(u_x,u_y)$};
      \node[anchor=south west] at (0.1, 0.8) {$(u_x',u_y')$};
      \addplot [
        color = blue,
        dashed,
        line width=0.125em
      ] coordinates {
          (-0.6,0.65) (0.1,0.65) (0.1,0) (-0.6,0)
        } -- cycle;
        \addplot [
        color = red,
        dashed,
        dash phase=0.5em,
        line width=0.125em
      ] coordinates {
          (-0.6,0.5) (0.1,0.8) (0.1,0) (-0.6,0)
        } -- cycle; 
       \draw[<->, color = blue] (-0.6,0.1) -- node[below] {$u_x' - u_x$} (0.1,0.1);
       \draw[<->, color = blue] (0.15,0) -- node[rotate=90, below] {$\frac12(u_y + u'_y)$} (0.15,0.65);
    \end{axis}
  \end{tikzpicture}
  \caption{Area under a piece where both endpoints are positive. Using the additivity of the integration, all other cases can be reduced to this one.}
  \label{fig:PWL_int_same_side}
\end{figure}
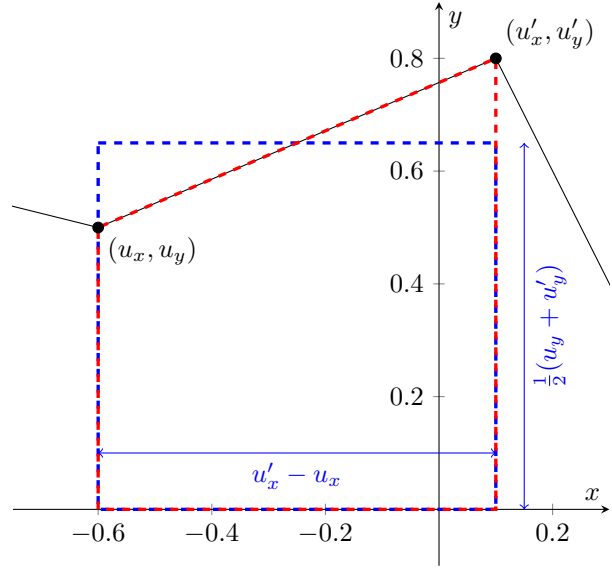

Our task is to sum up the following areas: (1) the area between $a$ and the first quasi breakpoint; (2) the areas between successive quasi breakpoints; and (3) the area between the last quasi breakpoint and $b$.
As noted earlier, we need to divide by the number of duplicates.
Additionally, we also need to catch the case where no quasi
breakpoint lies between $a$ and $b$. All these terms and formulas
are definable in $\wal$:

\begin{tabbing}
  \quad\=\,\,\,\=\,\,\,\=\,\,\,\=\,\,\,\=\,\,\,\=\,\,\,\=\,\,\,\=\kill
  \(\breakNextF(u_1,u_2) := \)\\[\jot]
  \> \( \breakF(u_1) \land \breakF(u_2) \land \breakx(u_1) < \breakx(u_2) \) \\[\jot]
  \> \( {} \land \neg \exists u_3(\breakF(u_3)\)\\[\jot]
  \> \> \( {} \land \breakx(u_1) < \breakx(u_3) < \breakx(u_2)) \)\\[\jot]

  \( \firstB(u) := \)\\[\jot]
  \> \( \breakIB(u) \land \neg \exists u' (\breakNextF(u',u) \land \breakIB(u'))\) \\[\jot]
  \( \lastB(u) := \)\\[\jot]
  \> \( \breakIB(u) \land \neg \exists u' (\breakNextF(u,u') \land \breakIB(u'))\) \\[\jot]
  \( \pieceIB(u_1,u_2) := \)\\[\jot]
  \> \( \breakIB(u_1) \land \breakIB(u_2) \land \breakNextF(u_1,u_2)\) \\[\jot]
  \( \samePiece(u_1,u_2,u_1',u_2') := \) \\[\jot]
  \> \(\breakNextF(u_1,u_2) \land \breakNextF(u_1',u_2') \) \\[\jot]
  \> \({} \land \breakx(u_1') = \breakx(u_1) \land \breakx(u_2') = \breakx(u_2) \)\\[\jot]
  \(\areaCalc(x,y) \coloneqq \frac12(y - x)(\eval_1(x) + \eval_1(y))\) \\[\jot]  
\end{tabbing}

\section{Proof of Theorem~\ref{thm:fully_agnostic}}
\label{sec:proof_fully_agnostic}
\begin{theorem}[Repetition of Theorem~\ref{thm:fully_agnostic}]
  Let $\varphi$ be a fully model agnostic closed $\wal$ formula over $\Upsnet(m,1)$. Then either $\Net\models\varphi$ for all $\Net\in\bigcup_{\ell\ge 1}\K(m,\ell)$ or $\Net\not\models\varphi$ for all $\Net\in\bigcup_{\ell\ge 1}\K(m,\ell)$.
\end{theorem}

\begin{proof}
  For simplicity we only give a proof for the $1$-dimensional
  case $m=1$; the generalisation to arbitrary $m$ is
  straightforward. Suppose for contradiction that there are  $\Net_1,\Net_2\in\bigcup_{\ell\ge 1}\K(1,\ell)$ such that $\Net_1\models\varphi$ and $\Net_2\not\models\varphi$. Since $\varphi$ is full model agnostic, we have $F^{\Net_1}\neq F^{\Net_2}$. We shall construct two networks $\Net_1',\Net_2'$ such that $F^{\Net_i'}=F^{\Net_i}$ and $\Net_1'\models\varphi\Leftrightarrow\Net_2'\models\varphi$. This will be a contradiction.

  Let vocabulary $\Upsilon$ consist of a binary relation symbol $E$ and six constant symbols $\mstart_1,\mstart_2,\mend_1,\mend_2,c_1,c_2$. Consider $\Upsilon$-structures that
  consist of two disjoint chains, using the
  binary relation $E$, with the constants $\mstart_i$ and $\mend_i$ marking the first and last node of the chain $i$ and the constants $c_1,c_2$ marking arbitrary distinct internal nodes.
  We call such structures ``two-chain structures''.

  For any
  natural number $n$, let $A_n$ be the two-chain structure where both
  chains have length $2n$, the middle node of the first chain
  is marked by $c_1$, and the middle node of the second chain is marked by $c_2$. Let $B_n$ be similar but with $c_1$
  marking the middle node of the second chain and $c_2$
  marking the middle node of the first chain.

  Standard locality arguments show that there does not
  exist an $\wal$ formula over $\Upsilon$ that distinguishes
  between $A_n$ and $B_n$ for all $n$.  Indeed, clearly, no
  Hanf-local query can do that, and $\wal$, which is subsumed by
  the relational calculus with aggregates and arithmetic, is
  Hanf-local for queries on abstract structures
  \cite{libkin_sql,libkin_fmt}.

  Now for any two-chain structure $A$, we construct a network
  $\Net(A)$ with one input and one output as follows.  The nodes
  and edges of $A$ serve as hidden neurons and connections
  between them.  We add an input with edges to the starts of both
  chains.  We also add an output with edges from the ends of both
  chains.  The biases of all nodes are $0$, and all edges have
  weight $1$, except for the edge from the end of the first chain
  to the output, which has weight $0$. Then we replace the node $c_i$ by the net $\Net_i$ (introduced at the beginning of the proof). That is, we remove the node $c_i$, add a copy of $\Net_i$, connect the predecessor of $c_i$ on its chain with the input node of $\Net_i$ (by an edge of weight $1$), and connect the output node of $\Net_i$ with the successor of $c_i$ (also by an edge of weight $1$). Then, for a sufficiently large $n$, the formula $\varphi$ cannot distinguish $\Net_1'\coloneqq\Net(A_n)$ and $\Net_2'\coloneqq\Net(B_n)$.

  It remains to observe that indeed we have $F^{\Net_i'}=F^{\Net_i}$.
\end{proof}

\section{Addition of constants}
\newcommand{\fConst}{f^{\mathrm{const}}}

In the coming proofs we will often assume structures interpreting vocabularies consisting of auxiliary constant symbols. These vocabularies will always be denoted by $\Upsconst$ (possibly subscripted). Call any vocabulary $\Upsilon$ without standard function symbols of non-zero arity, a \emph{relational} vocabulary. So $\Upsilon$ can have relational symbols, constant symbols, and weight function symbols. Formally, let $\A$ be any $\Upsilon$ structure with $\Upsilon$ relational, and let $\Upsconst = \{\delta_1,\dots,\delta_n\}$. We define the notation $\A \cup \{\delta_1,\dots,\delta_n\}$ to denote the structure $\A'$ over the vocabulary $\Upsilon \cup \Upsconst$ defined up to isomorphism as follows.
\begin{itemize}
  \item The domain $A'$ of $\A'$ equals $A \cup \{e_1,\dots,e_n\}$, where $e_1,\dots,e_n$ are distinct elements not in $A$.
  \item $\delta_{i}^{\Arr'} = e_i$ for $i \in \{1,\dots, n\}$.
  \item For each relation symbol $R$ and weight function symbol $w$ from $\Upsilon$ and any $\bar{a} \in A^k$, with $k$ the arity, we have $R^{\A'}(\bar{a}) \Leftrightarrow R^{\A}(\bar{a})$, and $w^{\A'}(\bar{a}) = w^{\A}(\bar{a})$.
  \item Moreover, for any $\bar{a} \in {A'}^k$ involving at least one of the new elements $e_i$, we set $R^{\A'}(\bar{a})$ to false, and $w^{\A'}(\bar{a})$ to $\bot$.
\end{itemize}

\begin{lemma}
  \label{lemconsttrans}
  Let $\Upsilon$ be a relational vocabulary including at least two constant symbols, say $c_1$ and $c_2$. Let $\Upsconst$ be as above. There exists an $\wal$ translation $\varphi$ from $\Upsilon$ to $\Upsilon \cup \Upsconst$ with the following behavior. For any $\Upsilon$-structure $\A$ such that $c_{1}^{\A}$ and $c_{2}^{\Arr}$ are distinct, $\varphi$ maps $\A$ to $\A \cup \{\delta_1,\dots,\delta_n\}$.
\end{lemma}
\begin{proof}
  Let $m$ be a number such that $2^m \geq n$. Let $\fConst:\Upsconst \to \{\delta_1, \delta_2\}^m$ be an injective function and let $C$ be its image. If $\bar{y} \in C$, then $\bar{y}_i$ refers to the $i$-th component in $\bar{y}$. We define:
  \begin{tabbing}
    \quad\=\,\,\,\=\,\,\,\=\,\,\,\=\,\,\,\=\,\,\,\=\,\,\,\=\,\,\,\=\kill
    \( \displaystyle \domF(t,x_1, \dots, x_m) := (t = c_1 \land x_1 = \dots = x_m) \lor (t = c_2 \land (\bigvee_{\bar{y} \in C}(\bigwedge_{i \in \{1,\dots,m\}} x_i = \bar{y}_i))) \) \\[\jot]
    \( \varphi_{c} := (c_1,c,\dots,c) \)\\[\jot]
    \> (for $c$ a constant symbol in $\Upsilon$)\\[\jot]
    \( \varphi_{R}(t^{1},x_{1}^{1}, \dots, x_{m}^{1},\dots,t^{n},x_{1}^{n}, \dots, x_{m}^{n}) := t^{1} = \cdots = t^{n} = c_1 \land R(x_{1}^{1},\dots,x_{1}^{m}) \)\\[\jot]
    \> (for $R$ an $n$-ary relation symbol in $\Upsilon$)\\[\jot]
    \( \varphi_{w}(t^{1},x_{1}^{1}, \dots, x_{m}^{1},\dots,t^{n},x_{1}^{n}, \dots, x_{m}^{n}) := \ifText t^{1} = \cdots = t^{n} = c_1 \thenText w(x_{1}^{1},\dots,x_{1}^{m}) \elseText \bot \) \\[\jot]
    \> (for $w$ an $n$-ary weight symbol in $\Upsilon$)\\[\jot]
    \( \varphi_{\delta_{i}} := (c_2,\bar{y}_1,\dots,\bar{y}_m) \)\\[\jot]
    \> (for $i \in \{1,\dots,n\}$ and $ \fConst(\delta_i)= \bar{y}$)\\[\jot]
    \( \eqF(t,x_1, \dots, x_m,t', x_1', \dots, x_m') := \displaystyle t = t' \land (\bigwedge_{i \in \{1,\dots,m\}} x_i = x_i') \)
  \end{tabbing}
\end{proof}

\section{Proof of Lemma~\ref{lemcd}}

\newcommand{\Upsconstcell}{\Upsconst_{\mathrm{cell}}}
\newcommand{\hpArrVSet}[2]{\hpArrV{#1} \uplus \dots \uplus \hpArrV{#2}}

\newcommand{\cellR}{\mathrm{cell}}
\newcommand{\baseR}{\mathrm{base}}
\newcommand{\hpR}{\mathrm{hp}^{\mathrm{cell}}}
\newcommand{\stackEqR}{\textrm{stack-eq}}
\newcommand{\stackEqiR}[1]{\textit{stack-eq-{#1}}}
\newcommand{\stackLtR}{\textrm{stack-lt}}
\newcommand{\stackLtiR}[1]{\textit{stack-lt-{#1}}}

\newcommand{\coefPWLW}[1]{\mathrm{a}^{\mathrm{pwl}}_{#1}}
\newcommand{\coefCellW}[1]{\mathrm{a}^{\mathrm{cell}}_{#1}}
\newcommand{\coefArrW}[2]{\mathrm{a}^{\mathrm{arr},{#1}}_{#2}}

\newcommand{\minInfC}{\delta_{-\infty}}
\newcommand{\plusInfC}{\delta_{+\infty}}
\newcommand{\oldC}{\delta_{\mathrm{old}}}
\newcommand{\newC}{\delta_{\mathrm{new}}}
\newcommand{\dimC}[1]{\delta_{#1\textrm{-dim}}}
\newcommand{\noneDimC}{\delta_{\textrm{none-dim}}}
\newcommand{\extraHpC}[1]{\delta_{\textrm{extra-hp-}{#1}}}
\newcommand{\originCellCC}{\delta_{o}}

\newcommand{\originCellC}{o}
\newcommand{\trueC}{\mathrm{true}}
\newcommand{\falseC}{\mathrm{false}}

\label{seccdproof}

We follow the method of vertical decomposition \cite{halperin_arr}. There is a projection phase, followed by a buildup phase.
For the projection phase, let $\Arr_d:=\A$. For $i=d,\dots,1$, take all intersections between hyperplanes in $\Arr_i$, and project one dimension down, i.e., project on the first $i-1$ components. The result is a hyperplane arrangement $\A_{i-1}$ in $\R^{i-1}$. 
For the buildup phase, let $\D_0 := \{\R^0\}$. For $i=0,\dots,d-1$, build a stack above every cell $C$ in $\D_i$ formed by intersecting $C \times \R$ with all hyperplanes in $\Arr_{i+1}$. The result is a partition $\D_{i+1}$ such that $\D_0,\dots,\D_{i+1}$ is a CD of $\R^{i+1}$ compatible with $\Arr_{i+1}$. In the following sublemmas, we work out $\wal$ translations implementing the different steps.

Before giving the translations, we will define the vocabulary $\hpArrV{i}$ of a $i$-dimensional hyperplane arrangement in $\R^{i+1}$ slightly differently to how it was previously defined to have unique relation names with the final vocabulary:
\begin{itemize}
  \item the unary weight functions $\coefArrW{i}{j}(h)$ indicating the $j$-th coefficient, $a_j$, of the $i$-dimensional hyperplane $h$ with the function $a_0 + a_1 x_1 + \dots + a_j x_j + \dots + a_i x_i$
\end{itemize}

\subsection{Projection of hyperplane arrangement}
\newcommand{\isVertical}{\mathit{vertical}}
\newcommand{\parallelF}{\mathit{parallel}}
\newcommand{\verticalHpF}{\textit{vertical-hp}}

We call a hyperplane in $i$ dimensions $a_0 + a_1 x_1 + \cdots + a_i x_i$ vertial if $a_i = 0$.
\begin{lemma}
  \label{lemarrlow}
  Let $i \geq 2$ be a natural number. There exists an $\wal$ translation from $\Upsarr_i$ to $\Upsarr_{i-1}$ that maps any hyperplane arrangement $\Arr$ in $i+1$ dimensions to a structure $\Arr'$ representing the hyperplane arrangement consisting of all intersections of all pairs of hyperplanes in $\Arr$, projected onto the first $i-1$ dimensions plus the projection of all vertical hyperplanes in $\Arr$.
\end{lemma}
\begin{proof}
  We encode the domain of $\Arr'$ by pairs $(h_1, h_2)$ of elements from $\Arr$, where $h_1$ and $h_2$ are not vertical and not parallel to each other. The projections of vertical hyperplanes $h$ in $\Arr$ is encoded by identical pairs $(h,h)$.
  \begin{tabbing}
    \quad\=\,\,\,\=\,\,\,\=\,\,\,\=\,\,\,\=\,\,\,\=\,\,\,\=\,\,\,\=\kill
    \( \parallelF(h_1, h_2) := \)\\[\jot]
    \> \( \displaystyle \coefArrW{i}{0}(h_1) \neq \coefArrW{i}{0}(h_2) \land (\bigvee_{j \in \{1,\dots,i\}} \coefArrW{i}{j}(h_1) \neq 0) \)\\[\jot]
    \> \( \displaystyle {} \land (\bigwedge_{j \in \{1,\dots,i\}} \coefArrW{i}{j}(h_1) \neq 0 \implies (\bigwedge_{k \in \{1,\dots,i\}} \frac{\coefArrW{i}{j}(h_2)\cdot\coefArrW{i}{k}(h_1)}{\coefArrW{i}{j}(h_1)} = \coefArrW{i}{k}(h_2))) \) \\[\jot]
    \( \isVertical(h):= \coefArrW{i}{j}(h) = 0 \) \\[\jot]
    \( \domF(h_1,h_2) := \) \\[\jot]
    \> \( (h_1 = h_2 \land \isVertical(h_1)) \) \\[\jot]
    \> \( {} \lor (\neg \parallelF(h_1,h_2) \land \neg \isVertical(h_1) \land \neg \isVertical(h_2)) \) \\[\jot]
  \end{tabbing}
  
  The new coefficients are now readily defined as follows:
  \begin{tabbing}
    \quad\=\,\,\,\=\,\,\,\=\,\,\,\=\,\,\,\=\,\,\,\=\,\,\,\=\,\,\,\=\kill
    \( \verticalHpF(h_1, h_2) := h_1 = h_2 \) \\[\jot]
    \( \varphi_{\coefArrW{i-1}{j}}(h_1, h_2) := \) \\[\jot]
    \> \( \ifText \verticalHpF(h_1, h_2) \thenText \coefArrW{i}{j}(h_1) \) \\[\jot]
    \> \( \elseTextNS \coefArrW{i}{j}(h_1)-\frac{\coefArrW{i}{j}(h_2)\coefArrW{i}{i}(h_1)}{\coefArrW{i}{i}(h_2)}\) \\[\jot]
    \> (for \(j\in\{0,\dots,i-1\}\)) \\[\jot]
    \( \eqF(h_1, h_2, h_1', h_2') := \) \\[\jot]
    \> \(\displaystyle (\bigwedge_{j \in \{0,\dots,d\}} \varphi_{\coefArrW{i}{k}}(h_1, h_2) = \varphi_{\coefArrW{i}{k}}(h_1', h_2')) \)
  \end{tabbing}
\end{proof}

\subsection{Vocabularies for cylindrical decompositions}

We will be representing cylindrical decompositions of $\R^i$ by structures over the following vocabulary denoted by $\cellV{i}$:
\begin{itemize}
  \item A unary relation $\cellR$ indicates the cells.
  \item The constant $\originCellC$ which represents the $0$-dimensional cell on top of which all other cells are built.
  \item A binary predicate $\baseR(c, c_b)$ that is true if the cell $c_b$ is the base cell of $c$.
  \item A unary relation $\hpR$ indicating the hyperplanes representing section mappings. (A section mapping $\xi: \R^j \to \R$ is viewed as the hyperplane $x_{j+1} = \xi(x_1,\dots,x_j)$ in $\R^{j+1}$).
  \item The unary weight functions $\coefCellW{j}$ indicate the coefficients of the hyperplanes. When a hyperplane lives in $\R^k$, all its values for $\coefCellW{l}$ with $l > k$ are set to $\bot$.
  \item A binary predicate $\startR(c, h)$ that is true if the the hyperplane $h$ represents the section mapping that forms the lower bound of the cell $c$.
  \item A binary predicate $\EndR(c, h)$ that is true if the the hyperplane $h$ represents the section mapping that forms the upper bound of the cell $c$.
  \item A ternary predicate $\stackEqR(c, h_1,h_2)$ that is true if $c$ is a $j$-cell with $j < i$ and $h_1$ and $h_2$ are hyperplanes representing the same section mapping above $c$. 
  \item A ternary predicate $\stackLtR(c, h_1,h_2)$, similar to $\stackEqR(c, h_1,h_2)$, but now representing that $\xi_1 < \xi_2$ above $c$ where $\xi_1$ ($\xi_2$) is the section mapping above $c$ represented by $h_1$ ($h_2$).
\end{itemize}
As in Remark~\ref{remhpdup}, duplicated hyperplanes are allowed. Moreover, even non-duplicate hyperplanes may represent the same section mapping when they intersect above an $(i-1)$-cell of dimension $< i-1$. See Figure~\ref{figsamesecmap} for a simple illustration.

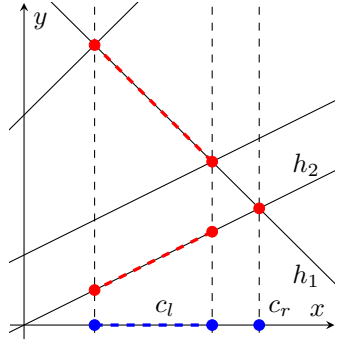
\begin{figure}
  \centering
  \begin{tikzpicture}
      \begin{axis}[
        xlabel = \(x\),
        xmin=-0.1, xmax=2.2,
        ylabel = \(y\),
        ymin=-0.1, ymax=2.3,
        ticks=none,
        axis lines = center,
        width = 0.5\linewidth,
        axis equal image
      ]
      \addplot [
        color = black
      ] coordinates {
          (-0.2,-0.1) (3.5,1.75)
        };
      \addplot [
        color = black
      ] coordinates {
          (-0.2,0.4) (3.5,2.25)
        };
      \addplot [
        color = black
      ] coordinates {
          (0,2.5) (3.5,-1)
        };
      \addplot [
        color = black
      ] coordinates {
          (-0.5,1) (1.5,3)
        };
      \addplot [
        color = black,
        dashed
      ] coordinates {
          (0.5,2.5) (0.5, -0.5)
        };
      \addplot [
        color = black,
        dashed
      ] coordinates {
          (5/3,2.5) (5/3, -0.5)
        };
      \addplot [
        color = black,
        dashed
      ] coordinates {
          (4/3,2.5) (4/3,-0.5)
        };
      \addplot [
        color = red,
        only marks
      ] coordinates {
          (0.5,0.25) (0.5, 2) (4/3, 7/6) (4/3, 4/6) (5/3, 5/6)
        };
      \addplot [
        color = blue,
        only marks
      ] coordinates {
          (0.5,0) (4/3, 0) (5/3, 0)
        };
      \addplot[
        color = red,
        dashed,
        line width=0.125em
      ] coordinates {
          (4/3, 7/6) (0.5, 2) 
        };
      \addplot[
        color = red,
        dashed,
        line width=0.125em
      ] coordinates {
          (4/3, 4/6) (0.5,0.25)
        };
      \addplot[
        color = blue,
        dashed,
        line width=0.125em
      ] coordinates {
          (0.5, 0) (4/3, 0) 
        };
      \node[anchor=north] at (2,0.5) {$h_1$};
      \node[anchor=south] at (2,1) {$h_2$};
      \node[anchor=south] at (1,0) {$c_l$};
      \node[anchor=south west] at (5/3,0) {$c_r$};   
    \end{axis}
    \end{tikzpicture}
    \caption{Hyperplanes $h_1$ and $h_2$ define different section mappings (in red) above the cell $c_l$ in blue but define the same section mapping above the cells $c_r$}
    \label{figsamesecmap}
\end{figure}

In the coming proofs we will assume structures interpreting the vocabulary $\Upsconstcell$ consisting of the following constant symbols:
\begin{itemize}
  \item $\minInfC$ used for encoding the concept of $-\infty$
  \item $\plusInfC$ used for encoding the concept of $+\infty$
  \item $\oldC$ used for encoding something previously defined
  \item $\newC$ used for encoding something not yet previously defined
  \item $\originCellCC$ used for encoding the encoding of the $0$-D cell in $0$ dimensional space
\end{itemize}
We will assume that structures interpret these constants by pairwise distinct elements that are also disjoint from the other elements in the structure. This assumption is without loss of generality (Lemma~\ref{lemconsttrans}).

\subsection{Buildup 0-D cylindrical decomposition}

\newcommand{\followsF}{\varphi_{\text{follows}}}
\newcommand{\intersectT}{\text{intersect}}

\begin{lemma}
  \label{lemcd0d}
  There exists an $\wal$ translation $\varphi$ from $\Upsconstcell$ to
  $\Upscell_0$ with the following behavior. For each $\Upsconstcell$-structure $\A$, $\varphi$ maps $\A$ to an affine CD of $\R^0$ that only contains the $0$-dimensional origin cell $o$.
\end{lemma}
\begin{proof}
  Since we only need the vocabulary to contain a single $0$-cell, we encode it with a single constant in the form $(c)$ where $c$ is the constant $\originCellCC$ if the encoding represents the origin cell $\originCellC$. Since the only thing encoded is the origin cell, the formulas $\varphi_{\baseR}(h, h')$ cannot be true on any value as the origin cell does not have a base cell. Additionally, the formulas $\varphi_{\startR}(h, h')$, $\varphi_{\EndR}(h, h')$, and $\varphi_{\hpR}(h)$ cannot be true for any value because there are no hyperplanes encoded and thus there are also no start or ending hyperplanes of a cell. For the same reason $\varphi_{\coefCellW{0}}(h)$ simply returns $\bot$. Finally, elements of this domain are equivalent if they encode the same constant, namely the origin cell. All this can be written in $\wal$ as follows:
  \begin{tabbing}
    \quad\=\,\,\,\=\,\,\,\=\,\,\,\=\,\,\,\=\,\,\,\=\,\,\,\=\,\,\,\=\kill
    \( \domF(h) := h=\originCellCC \) \\[\jot]
    
    \( \varphi_{\originCellC} := (\originCellCC) \) \\[\jot]
    \( \varphi_{\cellR}(h) := h = \originCellCC \) \\[\jot]
    \( \varphi_{\baseR}(h,h') := \falseC \) \\[\jot]
    \( \varphi_{\startR}(h,h') := \falseC \) \\[\jot]
    \( \varphi_{\EndR}(h,h') := \falseC \) \\[\jot]
    \( \varphi_{\hpR}(h) := \falseC \)\\[\jot]
    \( \varphi_{\stackLtR}(h,h',h'') := \falseC \)\\[\jot]
    \( \varphi_{\stackEqR}(h,h',h'') := \falseC \)\\[\jot]
    \( \varphi_{\coefCellW{0}}(h) := \bot \)\\[\jot]
    
    \( \eqF(h,h') := h=h\)
  \end{tabbing}
\end{proof}

\subsection{Build-up of higher dimensional cylindrical decomposition}
\label{secidimcd}

The following lemma verifies that the projection phase is a correct basis for building up a cylindrical decomposition.
\begin{lemma}
  \label{lemcdhporder}
  Let $\Arr$ be a hyperplane arrangement in $\R^i$, and let $\Arr'$ be a hyperplane arrangement in $\R^{i-1}$ containing at least the projections of all the intersections of hyperplanes in $\Arr$. Let $\D$ be a CD of $\R^{i-1}$ compatible with $\Arr'$, and let $c$ be an $(i-1)$-cell in $\D$. Then the set of all non-vertical hyperplanes in $\Arr$ represent a valid total order $\xi_1 < \cdots < \xi_r$ of section mappings above $c$.
\end{lemma}
\begin{proof}
  Assume to the contrary that there exist two different hyperplanes $x_i = \xi_1(x_1,\dots,x_{i-1})$ and $x_i = \xi_2(x_1,\dots, x_{i-1})$, such that either 
  \begin{enumerate}
    \item $\xi_1(\mathbf{b_1}) = \xi_2(\mathbf{b_1})$ and $\xi_1(\mathbf{b_2}) \neq \xi_2(\mathbf{b_2})$  for some $\mathbf{b_1}, \mathbf{b_2} \in c$, or
    \item $\xi_1(\mathbf{b_1}) < \xi_2(\mathbf{b_1})$ and $\xi_1(\mathbf{b_2}) > \xi_2(\mathbf{b_2})$ for some $\mathbf{b_1}, \mathbf{b_2} \in c$
  \end{enumerate}
  By the intermediate value theorem, in the second case there also exists $\mathbf{b}$ as $\mathbf{b_1}$ in the first case.
  Hence $c$ intersects with the projection of the intersection of the two hyperplanes. This contradicts the assumption that $\D$ is compatible with $\Arr'$.
\end{proof}

\newcommand{\newHp}{\textit{new-hp}}
\newcommand{\oldHp}{\textit{old-hp}}
\newcommand{\inCF}[1]{\mathit{in}_{#1}}
\newcommand{\oldCell}{\textit{old-cell}}
\newcommand{\StartC}{\mathord\vdash}
\newcommand{\EndC}{\mathord\dashv}
\newcommand{\coordT}{\mathit{coord}}
\newcommand{\projT}{\mathit{proj}}
\newcommand{\coplanarF}{\textit{eq-hp-above-cell}}
\newcommand{\smallerF}[1]{\textit{lower-in-}{#1}}
We next show that the total order from Lemma~\ref{lemcdhporder} is in fact definable in $\wal$:
\begin{lemma}
  \label{lemcdstackord}
  There exists an $\wal$-formula $\stackEqiR{i}(c,h_1,h_2)$ over $\hpArrV{i} \uplus \cellV{i-1}$ with the following properties. Let $\Arr$ and $\D$ be as in Lemma~\ref{lemcdhporder}. Then on $\Arr \uplus \D$, $\stackEqiR{i}$ defines all triples $(c,h_1,h_2)$ such that $c$ is an $(i-1)$-cell from $\D$ and $h_1$ and $h_2$ are hyperplanes from $\Arr$ that represent the same section mapping above $c$. Similarly, $\stackLtiR{i}$ defines all triples $(c,h_1,h_2)$ such that $c$ is an $(i-1)$-cell from $\D$ and that $\xi_1 < \xi_2$ above $c$ where $\xi_1$ ($\xi_2$) is the section mapping above $c$ represented by $h_1$ ($h_2$) from $\Arr$.
\end{lemma}
\begin{proof}
  Since $\stackEqiR{i}$ and $\stackLtiR{i}$ are only true for stacks above $(i-1)$-cells, we will also define the formula $\inCF{i-1}(c_{i-1})$ expressing that $c_{i-1}$ is an $(i-1)$-cell. We implement this by expressing that the length of the chain of base cells from $c$ to the origin cell $\originCellC$ is exactly $i-1$ steps long.
  \begin{tabbing}
    \quad\=\,\,\,\=\,\,\,\=\,\,\,\=\,\,\,\=\,\,\,\=\,\,\,\=\,\,\,\=\kill
    \( \displaystyle \inCF{i-1}(c_{i-1}) := \exists c_1, \dots, c_{i-2} (\baseR(c_1,\originCellC) \land (\bigwedge_{j \in \{1,\dots,i-2\}}\baseR(c_{j+1}, c_{j}))) \)
  \end{tabbing}

  For any $(i-1)$-cell $c$, any non-vertical hyperplane $h$ of $\Arr$ and any corner $r$ of $c$, denote by $\projT_{i,r}(c,h)$ the $i$-th coordinate of $r$ projected onto $h$. By Lemma~\ref{lemcdhporder} we know that for any $c$, $h_1$, and $h_2$ we have either:
  \begin{itemize}
    \item $\projT_{i,r}(c,h_1) \leq \projT_{i,r}(c,h_2)$ for all corners $r$ of $c$; or
    \item $\projT_{i,r}(c,h_1) \geq \projT_{i,r}(c,h_2)$ for all corners of $r$ of $c$.
  \end{itemize}
  Hence, we can express $\stackLtiR{i}(c, h_1, h_2)$ by expressing that $\projT_{i,r}(c,h_1) < \projT_{i,r}(c,h_2)$ for at least one corner $r$ of $c$. We can express $\stackEqiR{i}(c, h_1, h_2)$ similarly.

  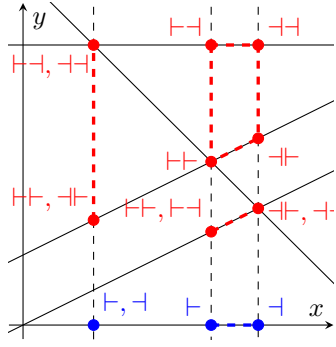
\begin{figure}
    \centering
    \begin{tikzpicture}
        \begin{axis}[
          xlabel = \(x\),
          xmin=-0.1, xmax=2.2,
          ylabel = \(y\),
          ymin=-0.1, ymax=2.3,
          ticks = none,
          axis lines = center,
          width = 0.5\linewidth,
          axis equal image
        ]
        \addplot [
          color = black
        ] coordinates {
            (-0.2,2) (3.5,2)
          };
        \addplot [
          color = black
        ] coordinates {
            (-0.2,-0.1) (3.5,1.75)
          };
        \addplot [
          color = black
        ] coordinates {
            (-0.2,0.4) (3.5,2.25)
          };
        \addplot [
          color = black
        ] coordinates {
            (0,2.5) (3.5,-1)
          };
        \addplot [
          color = black,
          dashed
        ] coordinates {
            (0.5,2.5) (0.5,2)
          };
        \addplot [
          color = black,
          dashed
        ] coordinates {
            (0.5,0.75) (0.5,-1)
          };
        \addplot [
          color = black,
          dashed
        ] coordinates {
            (5/3,2.5) (5/3, 2)
          };
        \addplot [
          color = black,
          dashed
        ] coordinates {
            (5/3, 8/6) (5/3,-1)
          };
        \addplot [
          color = black,
          dashed
        ] coordinates {
            (4/3,2.5) (4/3,2)
          };
        \addplot [
          color = black,
          dashed
        ] coordinates {
            (4/3, 7/6) (4/3,-1)
          };
        \addplot [
          color = red,
          only marks
        ] coordinates {
            (0.5,0.75) (0.5, 2) (4/3,2) (4/3, 7/6) (4/3, 4/6) (5/3, 5/6) (5/3, 8/6) (5/3, 2)
          };
        \addplot[
          color = red,
          dashed,
          line width=0.125em
        ] coordinates {
            (0.5,0.75) (0.5, 2)
          };
        \addplot[
          color = red,
          dashed,
          line width=0.125em
        ] coordinates {
            (4/3,2) (4/3, 7/6) (5/3, 8/6) (5/3, 2) 
          } -- cycle;
        \addplot[
          color = red,
          dashed,
          line width=0.125em
        ] coordinates {
            (4/3, 4/6) (5/3, 5/6)
          };
        \addplot [
          color = blue,
          only marks
        ] coordinates {
            (0.5,0) (4/3,0) (5/3,0)
          };
        \addplot[
          color = blue,
          dashed,
          line width=0.125em
        ] coordinates {
            (4/3, 0) (5/3, 0)
          };
        \node[anchor=south east, color=red] at (0.535, 0.75) {$\StartC\StartC, \EndC\StartC$};
        \node[anchor=north east, color=red] at (0.535, 2) {$\StartC\EndC, \EndC\EndC$};
        \node[anchor=south east, color=red] at (4/3,2) {$\StartC\EndC$};
        \node[anchor=east, color=red] at (1.31, 1.175) {$\StartC\StartC$};
        \node[anchor=south east, color=red] at (4/3, 4/6) {$\StartC\StartC, \StartC\EndC$};
        \node[anchor=west, color=red] at (5/3, 0.8) {$\EndC\StartC, \EndC\EndC$};
        \node[anchor=north west, color=red] at (5/3, 8/6)  {$\EndC\StartC$};
        \node[anchor=south west, color=red] at (5/3, 2) {$\EndC\EndC$};    
        \node[anchor=south west, color=blue] at (0.5, 0) {$\StartC, \EndC$};    
        \node[anchor=south east, color=blue] at (4/3, 0) {$\StartC$};    
        \node[anchor=south west, color=blue] at (5/3, 0) {$\EndC$};    
      \end{axis}
      \end{tikzpicture}
      \caption{Cylindrical decomposition of a $2$-D hyperplane arrangement with 2-cells highlighted in red and their base cells in blue on the $x$-axis, all with the identifiers of their corners given.}
      \label{figcornerenc}
  \end{figure}

  It remains to show how to identify the corners of a $j$-cell. There can be at most $2^j$ such corners; we will represent them by length $j$-strings over the alphabet $\{\StartC,\EndC\}$. We do this by induction on $j$. The only corner of the origin cell $\originCellC$ is $\originCellC$ itself; we identify this corner by the empty string. Now assume we know the corners of the base cell $c_b$ of a $j$-cell $c$. Recall that $c$ is delineated by two section mappings $\xi_1$, $\xi_2$ which are either equal (in case $c$ is a section) or consecutive in the order above $c_b$ (in case $c$ is a sector). By Lemma~\ref{lemcdhporder}, $\xi_1$ and $\xi_2$ are represented by non-vertical hyperplanes $h_1$ and $h_2$. Assuming we know the coordinates $(y_1,\dots, y_{j-1})$ of a corner $r$ of $c_b$, we can calculate the $j$-th coordinate of corners $r\StartC$ and $r\EndC$ of $c$, denoted by $\coordT_{j,r\StartC}(c)$ and $\coordT_{j,r\EndC}(c)$ respectively. These are obtained by $\projT_{j,r}(c,h_1)$ and $\projT_{j,r}(c,h_2)$ by solving the equations for $h_1$ and $h_2$ respectively, for $x_j$, after filling in $(y_1,\dots,y_{j-1})$ for $(x_1,\dots,x_{j-1})$.
  Note that the first $j-1$ coordinates of $r\StartC$ and $r\EndC$ are the same as those for $r$. We can thus construct terms $\coordT_{k,r\StartC}(c)$ and $\coordT_{r\EndC}(c)$ with $1 \leq k \leq j$ by induction on $j$.

  \begin{tabbing}
    \quad\=\,\,\,\=\,\,\,\=\,\,\,\=\,\,\,\=\,\,\,\=\,\,\,\=\,\,\,\=\kill
    \( \displaystyle \projT_{j,r}(c,h) := (\sum_{l = 1}^{j-1} \frac{- \coefCellW{l}(h)\cdot \coordT_{l,r}(c)}{\coefCellW{j}(h)}) - \frac{\coefCellW{0}(h)}{\coefCellW{j}(h)} \) \\[\jot]
    \> (for \(1 \leq j < i\) and \(r \in \{\StartC,\EndC\}^{j-1}\)) \\[\jot]
    \( \displaystyle \projT_{i,r}(c,h) := (\sum_{l = 1}^{i-1} \frac{- \coefArrW{i}{l}(h)\cdot \coordT_{l,r}(c)}{\coefArrW{i}{i}(h)}) - \frac{\coefArrW{i}{0}(h)}{\coefArrW{i}{i}(h)} \) \\[\jot]
    \> (for \(r \in \{\StartC,\EndC\}^{i-1}\)) \\[\jot]
    
    \( \displaystyle \coordT_{j,r\StartC}(c) := \frac{\displaystyle \sum_{h: \startR(c,h)} \sum_{c_b: \baseR(c,c_b)} \projT_{j,r}(c_b, h)}{\sum_{h': \startR(c,h')} 1} \) \\[\jot]
    \> (for \(1 \leq j < i\) and \(r \in \{\StartC,\EndC\}^{j-1}\)) \\[\jot]
    \( \displaystyle \coordT_{j,r\EndC}(c) := \frac{\displaystyle \sum_{h: \EndR(c,h)} \sum_{c_b: \baseR(c,c_b)} \projT_{j,r}(c_b, h)}{\sum_{h': \EndR(c,h')} 1} \) \\[\jot]
    \> (for \(1 \leq j < i\) and \(r \in \{\StartC,\EndC\}^{j-1}\)) \\[\jot]
    \( \displaystyle \coordT_{k,r\StartC}(c) := \sum_{c_b: \baseR(c,c_b)} \coordT_{k,r}(c_b) \) \\[\jot]
    \> (for \(1 \leq k < j < i\) and \(r \in \{\StartC,\EndC\}^{j-1}\)) \\[\jot]
    \( \displaystyle \coordT_{k,r\EndC}(c) := \sum_{c_b: \baseR(c,c_b)} \coordT_{k,r}(c_b) \) \\[\jot]
    \> (for \(1 \leq k < j < i\) and \(r \in \{\StartC,\EndC\}^{j-1}\)) \\[\jot]

    \( \stackLtiR{i}(c,h_1,h_2) := \) \\[\jot] 
    \> \( \displaystyle \domR{\hpArrV{i}}(h_1) \land \domR{\hpArrV{i}}(h_2) \land \domR{\cellV{i-1}}(c) \land \inCF{i-1}(c) \land (\bigvee_{r \in R}\projT_{i,r}(c,h_1) < \projT_{i,r}(c,h_2)) \) \\[\jot]
    \> (with \( R = \{\StartC,\EndC\}^{i-1}\))\\[\jot]
    \( \stackEqiR{i}(c,h_1,h_2) := \) \\[\jot] 
    \> \( \displaystyle \domR{\hpArrV{i}}(h_1) \land \domR{\hpArrV{i}}(h_2) \land \domR{\cellV{i-1}}(c) \land \inCF{i-1}(c) \land (\bigwedge_{r \in R}\projT_{i,r}(c,h_1) = \projT_{i,r}(c,h_2)) \) \\[\jot]
    \> (with \( R = \{\StartC,\EndC\}^{i-1}\))
  \end{tabbing}
\end{proof}

\begin{lemma}
  \label{lemcdbuilup}
  Let $i \geq 1$ be a natural number. There exists an $\wal$ translation $\varphi$ from $\Upsconstcell \cup (\Upsarr_i \uplus \Upscell_{i-1})$ to $\Upscell_i$ with the following behavior. For any $\Arr$, $\Arr'$, and $\D$ as in Lemma~\ref{lemcdhporder}, $\varphi$ maps $\Arr' \uplus \D$ to an affine CD of $\R^i$ that is compatible with $\Arr$.
\end{lemma}
\begin{proof}
  In general, if we are encoding a new cell,
  we have a starting hyperplane $h_s$,
  an ending hyperplane $h_e$, and a base cell $c_b$.
  If we are encoding anything else that already exists in one of
  the input vocabularies, we simply set $h_s$, $h_e$, and $c_b$
  equal to that. In more detail, if we wish to encode one
  of the hyperplanes $h$ of $\Arr'$ or $\Arr$, we can do
  so simply by setting $h_s=h_e=c_b=h$. Next, if we wish to
  encode a cell $c$ from $\D$, we can do this by setting
  $h_s=h_e=c_b=c$. If we want to encode a section determined by
  some non-vertical hyperplane $h \in \Arr$ in the stack above some $(i-1)$-cell $c\in \D$, we do so by setting $h_s=h_e=h$ and $c_b=c$. Should we wish to encode a sector above $c$, we first make sure that $h_1$ and $h_2$ are a pair of successive hyperplanes above $c$ and then set $h_s=h_1$, $h_e=h_2$, and $c_b=c$. Now for the first cell in the stack above $c$, we set $h_s=\minInfC$ to indicate that there is no start, we set $h_e=h$ where $h$ is the first hyperplane in the stack above $c$, and we set $c_b=c$. Similarly, for the last cell in the stack, we set $h_s=h$ where $h$ is the last hyperplane in the stack above $c$, we set $h=\plusInfC$ to indicate the cell has no end, and we set $c_b=c$. Finally, if there is only a single cell in the stack above $c$, which means that $\Arr$ only contains vertical hyperplanes, we set $h_s=\minInfC$, $h_e=\plusInfC$ and $c_b=c$.

  Since we often need to identify $(i-1)$-cells, we use
  a formula $\inCF{i-1}(c_{i-1})$ expressing that $c_{i-1}$ is an $(i-1)$-cell.
  We can implement this formula by expressing that the length of the chain of base cells from $c$ to the origin cell $\originCellC$ is exactly $i-1$ steps long.

  We use the formula $\stackLtiR{i}$ from Lemma~\ref{lemcdstackord}. The formulas and the encoding can be written in $\wal$ as follows: 
  \begin{tabbing}
    \quad\=\,\,\,\=\,\,\,\=\,\,\,\=\,\,\,\=\,\,\,\=\,\,\,\=\,\,\,\=\kill
    \( \newHp(h) := \domR{\hpArrV{i}}(h) \) \\[\jot]

    \( \followsF(h_l,h_u,c_b) := \) \\[\jot]
    \> \( \cellR(c_b) \land \newHp(h_l) \land \newHp(h_u) \land \stackLtiR{i}(c_b,h_l,h_u) \) \\[\jot]
    \> \( {} \land \neg \exists h' (\newHp(h') \land \stackLtiR{i}(c_b,h_l,h') \land \stackLtiR{i}(c_b,h',h_u))\) \\[\jot]
    \( \displaystyle \inCF{i-1}(c_{i-1}) := \exists c_1, \dots, c_{i-2} (\baseR(c_1,\originCellC) \land (\bigwedge_{j \in \{1,\dots,i-2\}}\baseR(c_{j+1}, c_{j}))) \) \\[\jot]

    \( \domF(h_s,h_e,c_b) := \) \\[\jot]
    \> \( (h_s = h_e = c_b \land (\hpR(c_b) \lor \newHp(c_b))) \) \\[\jot]
    \> \({} \lor (\cellR(c_b) \) \\[\jot]
    \> \> \( {} \land (h_s = h_e = c_b \) \\[\jot]
    \> \> \> \( {} \lor (\inCF{i-1}(c_b) \) \\[\jot]
    \> \> \> \> \( {} \land ((h_s=h_e \land \newHp(h_s) \land \neg \isVertical(h_s)) \)\\[\jot]
    \> \> \> \> \> \( {} \lor \followsF(h_s,h_e,c_b) \) \\[\jot]
    \> \> \> \> \> \( {} \lor (h_s=\minInfC \land \newHp(h_e) \land \neg \isVertical(h_e) \land \neg \exists h'\, \followsF(h',h_e,c_b)) \) \\[\jot]
    \> \> \> \> \> \( {} \lor (h_e=\plusInfC \land \newHp(h_s) \land \neg \isVertical(h_s) \land \neg \exists h'\, \followsF(h_s,h',c_b)) \)\\[\jot]
    \> \> \> \> \> \( {} \lor (h_s=\minInfC \land h_e=\plusInfC \land \neg \exists h' (\newHp(h') \land \neg \isVertical(h'))))))) \)
  \end{tabbing}

  Next we can start translating the constant $\originCellC$ and the relations $\hpR$ and $\cellR$ in the obvious manner based on the previously described encoding. Additionally, we add the formulas $\oldHp(h_s, h_e, c_b)$ and $\oldCell(h_s, h_e, c_b)$ which is only true if $(h_s, h_e, c_b)$ encodes a hyperplane or cell from $\Upscell_{i-1}$ respectively. 

  \begin{tabbing}
    \quad\=\,\,\,\=\,\,\,\=\,\,\,\=\,\,\,\=\,\,\,\=\,\,\,\=\,\,\,\=\kill
    \( \varphi_{\originCellC} := (\originCellC, \originCellC, \originCellC) \) \\[\jot]
    \( \varphi_{\hpR}(h_s,h_e,c_b) := (\hpR(c_b) \lor \newHp(c_b)) \) \\[\jot]
    \( \oldHp(h_s,h_e,c_b) := \varphi_{\hpR}(h_s,h_e,c_b) \land \hpR(c_b) \) \\[\jot]
    \( \varphi_{\cellR}(h_s,h_e,c_b) := \cellR(c_b) \) \\[\jot]
    \( \oldCell(h_s,h_e,c_b) := \varphi_{\cellR}(h_s,h_e,c_b) \land h_s = h_e = c_b \)
  \end{tabbing}
  
  The formula $\varphi_{\baseR}(h_s, h_e, c_b, {h_s}', {h_e}', {c_b}')$ is true if $(h_s, h_e, c_b)$ encodes a cell which has as its base cell the cell encoded by $({h_s}', {h_e}', {c_b}')$. When checking this, there are two cases: either $(h_s, h_e, c_b)$ encodes a new cell, in which case $({h_s}', {h_e}', {c_b}')$ must encode the cell $c_b$, or $(h_s, h_e, c_b)$ encodes an old cell, in which case we refer to $\baseR$ of $\Upscell_i$. 
  
  \begin{tabbing}
    \quad\=\,\,\,\=\,\,\,\=\,\,\,\=\,\,\,\=\,\,\,\=\,\,\,\=\,\,\,\=\kill
    \( \varphi_{\baseR}(h_s,h_e,c_b,{h_s}',{h_e}',{c_b}') := \) \\[\jot]
    \> \( \varphi_{\cellR}(h_s,h_e,c_b) \land \oldCell({h_s}',{h_e}',{c_b}') \) \\[\jot]
    \> \( {} \land ((\neg \oldCell(h_s,h_e,c_b) \land c_b = {c_b}' ) \lor (\oldCell(h_s,h_e,c_b) \land \baseR(c_b, {c_b}'))) \)
 \end{tabbing}

  The formula $\varphi_{\startR}(h_s, h_e, c_b, {h_s}', {h_e}', {c_b}')$
  is true only if $(h_s, h_e, c_b)$ encodes a cell which has the
  hyperplane encoded by $({h_s}', {h_e}', {c_b}')$ as one of its starting
  hyperplanes. Again, there are two cases: either $(h_s, h_e, c_b)$
  encodes a new cell, in which case we check that the hyperplane
  encoded by $({h_s}', {h_e}', {c_b}')$ is equal to $h_s$ in the stack
  above $c_b$, or $(h_s, h_e, c_b)$ encodes an old cell, in which
  case we refer to $\startR$ of $\cellV{i-1}$. Here equality in the stack above $c_b$ is 
  expressed using the formula $\stackEqiR{i}$ from Lemma~\ref{lemcdstackord}.
  The formula $\varphi_{\EndR}(h_s, h_e, c_b, {h_s}', {h_e}', {c_b}')$ is similar to
  the previous one but for the ending hyperplane instead, and is
  similarly implemented substituting $h_e$ for $h_s$ and
  $\EndR(c, h)$ for $\startR(c, h)$.

  \begin{tabbing}
    \quad\=\,\,\,\=\,\,\,\=\,\,\,\=\,\,\,\=\,\,\,\=\,\,\,\=\,\,\,\=\kill
    \( \varphi_{\startR}(h_s,h_e,c_b,{h_s}',{h_e}',{c_b}') := \) \\[\jot]
    \> \( \varphi_{\cellR}(h_s,h_e,c_b) \land \varphi_{\hpR}({h_s}',{h_e}',{c_b}') \) \\[\jot]
    \> \( {} \land (\stackEqiR{i}(c_b,h_s,{h_s}') \lor (\oldCell(h_s,h_e,c_b) \land \oldHp({h_s}',{h_e}',{c_b}') \land \startR(c_b,{h_s}')))\) \\[\jot]

    \( \varphi_{\EndR}(h_s,h_e,c_b,{h_s}',{h_e}',{c_b}') := \) \\[\jot]
    \> \( \varphi_{\cellR}(h_s,h_e,c_b) \land \varphi_{\hpR}({h_s}',{h_e}',{c_b}') \) \\[\jot]
    \> \( {} \land (\stackEqiR{i}(c_b,h_e,{h_e}') \lor (\oldCell(h_s,h_e,c_b) \land \oldHp({h_s}',{h_e}',{c_b}') \land \EndR'(c_b,{h_e}'))) \)
  \end{tabbing}

  The formula $\varphi_{\stackLtR}(h_s, h_e, c_b, {h_s}', {h_e}', {c_b}', {h_s}'', {h_e}'', {c_b}'')$
  if true only if $(h_s, h_e, c_b)$ encodes a cell in which the stack above it
  the hyperplane encoded by $({h_s}', {h_e}', {c_b}')$ represents a section mapping
  that is lower than the section mapping of the hyperplane encoded by $({h_s}'', {h_e}'', {c_b}'')$.
  Here we have two cases, either $(h_s, h_e, c_b)$ encodes an $(i-1)$-cell,
  in which case we use $\stackLtiR{i}$ from Lemma~\ref{lemcdstackord} to check
  the order of the two hyperplanes above it, or $(h_s, h_e, c_b)$ encodes a $j$-cell
  with $j<i-1$ in which case we refer to $\stackLtR$ of $\cellV{i-1}$.
  The formula $\varphi_{\stackEqR}(h_s, h_e, c_b, {h_s}', {h_e}', {c_b}', {h_s}'', {h_e}'', {c_b}'')$
  is similar to the previous one but checking that the two section mappings are equal instead,
  and is similarly implemented by substituting $\stackEqiR{i}$ and $\stackEqR$ for $\stackLtiR{i}$
  and $\stackLtR$ respectively.

  \begin{tabbing}
    \quad\=\,\,\,\=\,\,\,\=\,\,\,\=\,\,\,\=\,\,\,\=\,\,\,\=\,\,\,\=\kill
    \( \varphi_{\stackLtR}(h_s, h_e, c_b, {h_s}', {h_e}', {c_b}', {h_s}'', {h_e}'', {c_b}'') := \) \\[\jot]
    \> \( \varphi_{\cellR}(h_s,h_e,c_b) \land \oldCell(h_s,h_e,c_b) \land \varphi_{\hpR}({h_s}',{h_e}',{c_b}') \land \varphi_{\hpR}({h_s}'',{h_e}'',{c_b}'') \) \\[\jot]
    \> \( {} \land (\inCF{i-1}(c_b) \land \neg \oldHp({h_s}',{h_e}',{c_b}') \land \neg \oldHp({h_s}'',{h_e}'',{c_b}'') \land \stackLtiR{i}(c_b,h_s,{h_s}') \)\\[\jot]
    \> \>\({} \lor (\neg \inCF{i-1}(c_b) \land \oldHp({h_s}',{h_e}',{c_b}') \land \oldHp({h_s}'',{h_e}'',{c_b}'') \land \stackLtR(c_b, {h_s}', {h_s}'')))\) \\[\jot]

    \( \varphi_{\stackEqR}(h_s, h_e, c_b, {h_s}', {h_e}', {c_b}', {h_s}'', {h_e}'', {c_b}'') := \) \\[\jot]
    \> \( \varphi_{\cellR}(h_s,h_e,c_b) \land \oldCell(h_s,h_e,c_b) \land \varphi_{\hpR}({h_s}',{h_e}',{c_b}') \land \varphi_{\hpR}({h_s}'',{h_e}'',{c_b}'') \) \\[\jot]
    \> \( {} \land (\inCF{i-1}(c_b) \land \neg \oldHp({h_s}',{h_e}',{c_b}') \land \neg \oldHp({h_s}'',{h_e}'',{c_b}'') \land \stackEqiR{i}(c_b,h_s,{h_s}') \)\\[\jot]
    \> \>\({} \lor (\neg \inCF{i-1}(c_b) \land \oldHp({h_s}',{h_e}',{c_b}') \land \oldHp({h_s}'',{h_e}'',{c_b}'') \land \stackEqR(c_b, {h_s}', {h_s}'')))\)
  \end{tabbing}

  Finally, the weight functions $\varphi_{\coefCellW{j}}(h_s, h_e,
  c_b)$ return the coefficients of the hyperplane encoded by $(h_s, h_e,
  c_b)$ for $j \in \{1, \dots, i\}$. Here we again have two cases:
  either $(h_s, h_e, c_b)$ encodes a hyperplane from $\Arr$, in
  which case we simply return $\coefArrW{i}{j}$ of that hyperplane,
  or it encodes a hyperplane is from $\Arr'$, in which case we
  return $\coefCellW{j}$ of the hyperplane it encodes, unless
  $j=i$. In this case we return $\bot$ for
  $\varphi_{\coefCellW{j}}(h_s, h_e, c_b)$.
  \begin{tabbing}
    \quad\=\,\,\,\=\,\,\,\=\,\,\,\=\,\,\,\=\,\,\,\=\,\,\,\=\,\,\,\=\kill
    \( \varphi_{\coefCellW{j}}(h_s,h_e,c_b) := \ifText \newHp(c_b) \thenText \coefArrW{i}{j}(h_s) \elseText \coefCellW{j}(h_s) \) \\[\jot]
    \> (for  \(j\in\{0,\dots,i-1\}\)) \\[\jot]
    \( \varphi_{\coefCellW{i}}(h_s,h_e,c_b) := \ifText \newHp(c_b) \thenText \coefArrW{i}{i}(h_s) \elseText \bot\)
  \end{tabbing}

  Next, we need to implement $\eqF(h_s,h_e,c_b,{h_s}',{h_e}',{c_b}')$ which is true if $(h_s, h_e, c_b)$ and $({h_s}', {h_e}', {c_b}')$ encode the same hyperplane or cell. If both encode a hyperplane, we can simply check that they agree on all coefficients. If $(h_s, h_e, c_b)$ and $({h_s}', {h_e}', {c_b}')$ encode cells, we again have two cases: If they both encode old cells, we simply check $h_s = {h_s}'$, $h_e = {h_e}'$, and $c_b = {c_b}'$ which implies $h_s=h_e=c_b={h_s}'={h_e}'={c_b}'$ due to how we encode old cells. If they are both new cells, we check that they have the same base cell $c_b$ and that their starting and ending hyperplanes are both equal in the stack above $c_b$.
  This can all be written in $\wal$ as follows:
  \begin{tabbing}
    \quad\=\,\,\,\=\,\,\,\=\,\,\,\=\,\,\,\=\,\,\,\=\,\,\,\=\,\,\,\=\kill
    \( \eqF(h_s,h_e,c_b,{h_s}',{h_e}',{c_b}') :=  \) \\[\jot]
    \> \( \displaystyle (\varphi_{\hpR}(h_s,h_e,c_b) \land \varphi_{\hpR({h_s}',{h_e}',{c_b}')} \land (\bigwedge_{j \in \{1,\dots,i\}} \varphi_{\coefCellW{j}}(h_s,h_e,c_b) = \varphi_{\coefCellW{j}}({h_s}',{h_e}',{c_b}'))) \) \\[\jot]
    \> \( {} \lor (\varphi_{\cellR}(h_s,h_e,c_b) \land \varphi_{\cellR}({h_s}',{h_e}',{c_b}')\) \\[\jot]
    \> \> \( {} \land ((h_s = {h_s}' \land h_e = {h_e}' \land c_b = {c_b}') \) \\[\jot]
    \> \> \> \( {} \lor (\neg \oldCell(h_s,h_e,c_b) \land \neg \oldCell({h_s}',{h_e}',{c_b}') \land c_b={c_b}' \) \\[\jot]
    \> \> \> \> \( {} \land \stackEqiR{i}(c_b,h_s, {h_s}') \land \stackEqiR{i}(c_b,h_e, {h_e}')))) \)
  \end{tabbing}
\end{proof}

\begin{remark}
  \label{remhpcoef}
  In the proof for Lemma~\ref{lemcdbuilup} it is guaranteed that every hyperplane in $\Arr_i$ is included as a hyperplane of the CD with the same coefficients as it had in $\Upsarr_i$.
  This is because every hyperplane in $\Arr_i$ is encoded in the translation and because for the encoded hyperplanes of $\Arr_i$ the term $\varphi_{\coefCellW{j}}(h_s,h_e,c_b)$ simply returns the coefficients they had in $\Upsarr_i$.
\end{remark}
\begin{remark}
  \label{remverthpproj}
  Inspecting the proof we can verify that following property of the construction of $\D_d$: if the $1 \leq j \leq d$ largest coefficients of a hyperplane $h \in \Arr_d$ are zero, there is a hyperplane $h'$ in $\D_d$ with $\coefCellW{k}(h') = \bot$ for $k \in \{d-j, \dots, d\}$ and $\coefCellW{\ell}(h') = \coefArrW{d}{\ell}(h)$ for $\ell \in \{0, \dots, d-j-1\}$. Indeed, this follows from how we handle vertical hyperplanes in the projection phase, simply removing the largest coefficient, followed by the buildup phase, where Remark~\ref{remhpcoef} shows we keep the coefficients. Additionally, in the buildup phase, when a hyperplane that lives in some $\R^j$ with $j<i$ is encoded in $\D^i$, its $i$-th coefficient is set to $\bot$.
\end{remark}

\section{Proof of Lemma~\ref{lempwl} for general $m$}

\newcommand{\hpPosCR}{\textit{pos-cells-bp}}
\newcommand{\hpNegCR}{\textit{neg-cells-bp}}
\newcommand{\hpOnCR}{\textit{on-hp-cells-bp}}

\newcommand{\Upspwlnode}{\Upsilon^{\textrm{pwl-node}}}
\newcommand{\Upspwledge}{\Upsilon^{\textrm{pwl-edge}}}
\newcommand{\PWLNodeV}[1]{(\Upsnet \uplus \PWLV{#1}) \cup \Upspwlnode}
\newcommand{\PWLEdgeV}[1]{(\Upsnet \uplus \PWLV{#1}) \cup \Upspwledge}
\newcommand{\Upsconstpwl}{\Upsconst_{\mathrm{pwl}}}

\newcommand{\pwlC}{\delta_{\mathrm{pwl}}}
\newcommand{\netC}{\delta_{\mathrm{net}}}
\newcommand{\cellC}{\delta_{\mathrm{cell}}}

\newcommand{\breakplaneR}{\mathrm{hp}^{\mathrm{pwl}}}
\newcommand{\ptopeR}{\mathrm{ptope}}
\newcommand{\posSR}{\textrm{pos-side}}
\newcommand{\negSR}{\textrm{neg-side}}
\newcommand{\onSR}{\textrm{on-bp}}
\newcommand{\genNodeR}{\textrm{gen-node}}
\newcommand{\genEdgeR}{\textrm{gen-edge}}

\newcommand{\layerIF}[1]{\textit{layer}_{#1}}
\newcommand{\bpToHpF}{\textit{bp-to-hp}}
\newcommand{\cellToPtopeF}{\textit{cell-in-old-ptope}}
\newcommand{\posPtopeF}{\textit{pos-ptope}}
\newcommand{\falseT}{\mathit{false}}

\newcommand{\eqCellF}{\textit{eq-cell}}
\newcommand{\eqHpF}{\textit{eq-hyperplane}}
\newcommand{\hpPosCCHpR}{\textit{pos-cells-in-CD}}
\newcommand{\hpNegCCHpR}{\textit{neg-cells-in-CD}}
\newcommand{\hpOnCCHpR}{\textit{on-hp-cells-in-CD}}

In this Appendix we provide full details for the construction sketched out in Section~\ref{submultiple}.

Recall that a $f:\R^m \to \R$ is called PWL if there is a partition of $\R^m$ in polytopes, such that the result of $f$ on each polytope is an affine function. We call all these functions given on the polytopes the \emph{component functions} of $f$. We introduce a vocabulary $\PWLV{m}$ so that piecewise linear functions $\R^m \to \R$ can be represented by $\PWLV{m}$-structures. Such structures have two kinds of elements: breakplanes, which are affine hyperplanes in $\R^m$; and polytopes bounded by these hyperplanes. The vocabulary $\PWLV{m}$ contains the following:
\begin{itemize}
  \item The unary relation $\breakplaneR$ indicates the breakplanes.
  \item The unary relation $\ptopeR$ indicates the polytopes.
  \item The unary weight functions $\coefPWLW{i}$ for $i \in \{0,\dots,m\}$, applied to a hyperplane $h$, indicate the coefficients of $h$. When applied to a polytope $p$ they indicate the coefficients of its component function.
  \item $\posSR$, $\negSR$, $\onSR$ are binary relations relating polytopes to breakplanes.
\end{itemize}
\begin{definition}
  \label{defproperpwl}
  A $\PWLV{m}$-structure $\mathcal{P}$ is called proper if
  \begin{itemize}
    \item Every element satisfies exactly one of the relations $\breakplaneR$ and $\ptopeR$.
    \item Relations $\posSR$, $\negSR$, $\onSR$ are pairwise disjoint, and their union equals the set of all pairs $(c,h)$ with $c \in \ptopeR$ and $h \in \breakplaneR$.
    \item Let us call a ``position vector'' any function $v: \breakplaneR \to \{<,>,=\}$. Every polytope $c$ has an associated position vector $v_c$ where
    \begin{equation*}
      v_c(h)=
      \begin{cases}
        + \text{ if } \posSR(c,h)\\
        - \text{ if } \negSR(c,h)\\
        = \text{ if } \onSR(c,h)
      \end{cases}
    \end{equation*}
    Now in $\mathcal{P}$ no two different polytopes can have the same position vector.
    \item We call a position vector \emph{realizable} if the
      following system of inequalities has a solution in $\R^m$. For each hyperplane $h$, $\coefPWLW{i}(h)$ as $a_i$ and writing $v(h)$ as $\theta$, we include the inequality
    \begin{equation*}
      a_0+a_1 x_1 + \cdots + a_n x_n \mathrel{\theta} 0.
    \end{equation*}
    Now in $\mathcal{P}$, for every realizable position vector, $v$ there must exist a polytope $c$ such that $v_c = v$
  \end{itemize}
\end{definition}

In every proper $\PWLV{m}$-structure, the polytopes form a partition of $\R^m$, so the structure represents a unique PWL function $\R^m \to \R$.

Uniformly over any $\ell$-layer network $\Net$, our goal is to generate for every node $u$ a structure for $\Funcu{\Net}{u}$. We collect all these structures, which are kept disjoint, in a single $\PWLV{m} \cup \Upspwlnode$-structure. Here, the vocabulary $\Upspwlnode$ provides a single extra relation $\genNodeR$ that associates elements from the structure representing $\Funcu{\Net}{u}$ with $u$.
We start with the input neurons and then proceed through the layers. 

\subsection{Creating initial functions}

Recall, for an input neuron $u$, that $\Funcu{\Net}{u} = x_i$, if $u$ is the $i$-th input. In the following, we involve an auxiliary vocabulary $\Upsconstpwl$, consisting of two standard constant symbols $\pwlC$ and $\netC$.

\begin{lemma}
  \label{lempwlinit}
  Let $m$ and $\ell$ be natural numbers. There exists an $\wal$-translation $\varphi$ from $\Upsnet \uplus \Upsconstpwl$ to $\PWLNodeV{m}$ with the following behavior. For any $\Net \in \K(m,\ell)$, let $\Net'$ be the disjoint union of $\Net$ with two distinct constants. Then $\varphi$ maps $\Net'$ to the disjoint union of $\Net$ with a proper structure that associates every input neuron $u$ with $\Funcu{\Net}{u}$.
\end{lemma}

\begin{proof}
  We will encode our domain with tuples of the form $(t,n)$, where $t$ is a constant from $\Upsconstpwl$ and $n$ is a node from $\Net$. If $t=\pwlC$, the tuple $(t,n)$ encodes the single polytope of the function $f_n$. If $t=\netC$, the tuple $(t,n)$ encodes the node $n$. The translations of relations that take a breakplane can all be translated to a formula that is false, since there are no breakplanes. The relations, constants and weight functions of $\Upsnet$ are equally easily translated since they simply check the domain of their arguments and then pass these to their original relations in the input. The coefficients of the functions above the polytopes are given by the terms $\varphi_{\coefPWLW{i}}(t,n)$ which return $1$ if $n$ is the $i$-th input node and otherwise returns $0$.
  \begin{tabbing}
    \quad\=\,\,\,\=\,\,\,\=\,\,\,\=\,\,\,\=\,\,\,\=\,\,\,\=\,\,\,\=\kill
    \( \layerIF{0}(n) := \neg \exists n' \, E(n',n) \) \\[\jot]
    \( \domF(t,n) :=\)\\[\jot]
    \> \( (\domR{\Upsnet} (n) \land t = \netC) \lor (\land \layerIF{0}(n) \land t=\pwlC) \) \\[\jot]
    \( \varphi_{\domR{\PWLV{m}}}(t,n) := t=\pwlC \)\\[\jot]
    \( \varphi_{\domR{\Upsnet}}(t,n) := t=\netC \)\\[\jot]
    \( \varphi_{\breakplaneR}(t,n) := \falseT \)\\[\jot]
    \( \varphi_{\ptopeR}(t,n) := \varphi_{\domR{\PWLV{m}}}(t,n) \)\\[\jot]
    
    \( \varphi_{\coefPWLW{i}}(t,n) := \)\\[\jot]
    \> \(\ifText \varphi_{\domR{\PWLV{m}}}(t,n) \land n=\inn_i \thenText 1 \) \\[\jot]
    \> \( \elseTextNS (\ifText \varphi_{\domR{\PWLV{m}}}(t,n) \thenText 0 \elseText \bot )\) \\[\jot]
    \> (for $i \in \{1,\dots,m\}$)\\[\jot]
    \( \varphi_{\coefPWLW{0}}(t,n) := \ifText \varphi_{\domR{\PWLV{m}}}(t,n) \thenText 0 \elseText \bot\) \\[\jot]
    
    \( \varphi_{\posSR}(t,n,t',n') := \falseT \)\\[\jot]
    \( \varphi_{\negSR}(t,n,t',n') := \falseT \)\\[\jot]
    \( \varphi_{\onSR}(t,n,t',n') := \falseT \)\\[\jot]
    
    \( \varphi_{\inn_i} := (\netC, \inn_i) \) \\[\jot]
    \> (for $i \in \{1,\dots,m\}$) \\[\jot]
    \( \varphi_{\out_1} := (\netC, \out_1) \) \\[\jot]
    \( \varphi_{E}(t,n,t',n') := \varphi_{\domR{\Upsnet}}(t,n) \land \varphi_{\domR{\Upsnet}}(t',n') \land E(n,n') \) \\[\jot]
    \( \varphi_{b}(t,n) := \ifText \varphi_{\domR{\Upsnet}}(t,n) \thenText b(n) \elseText \bot \)\\[\jot]
    \( \varphi_{w}(t,n,t',n') := \ifText
    \varphi_{\domR{\Upsnet}}(t,n) \land
    \varphi_{\domR{\Upsnet}}(t',n') \) \\
\>\>\>\>\>\>\>\>
    \( \thenText w(n,n') \elseText \bot \)\\[\jot]
    
    \( \varphi_{\genNodeR}(t,n,t',n') := \)\\[\jot]
    \> \( \varphi_{\domR{\PWLV{m}}}(t,n) \land \varphi_{\domR{\Upsnet}}(t',n') \land n = n' \)\\[\jot]
    \( \eqF(t,n,t',n,) := t=t' \land n=n' \)
  \end{tabbing}
\end{proof}

\subsection{Scaling for $i \in \{0,\dots,\ell-1\}$}

In the next step, we assume that $\Funcu{\Net}{u}$ is already represented for all neurons in layer $i$ (layer 0 is the input layer). In order to proceed to layer $i+1$ we must first scale each function $\Funcu{\Net}{u}$ to $w(u,v)\Funcu{\Net}{u}$ for every edge $u \to v$ in $\Net$. The vocabulary $\Upspwledge$ provides an extra relation $\genEdgeR$ that associates $w(u,v)\Funcu{\Net}{u}$ to the edge $(u,v)$.

\begin{lemma}
  \label{lempwlscale}
  Let $m$ and $\ell$ be natural numbers and let $i \in \{0,dots,\ell-1\}$. There exists an $\wal$-translation $\varphi$ from $\PWLNodeV{m}$ to $\PWLEdgeV{m}$ with the following behavior. For any $\Net \in \K(m,\ell)$, let $\Net'$ be the extension of $\Net$ associating PWL functions $f_u$ to all neurons $u$ in the $i$-th layer. Then $\varphi$ maps $\Net'$ to the disjoint union of $\Net$ with a proper structure that associates every edge $u \to v$ with $u$ in the $i$-the layer to $w(u,v)\cdot f_u$.
\end{lemma}
\begin{proof}
  We will encode any element of our translation as a pair of the
  form $(c,n)$. If it encodes a node $n'$ of $\Net$, we simply
  set $c=n=n'$. If $(c,n)$ encodes a breakplane or a polytope $y$ of
  $w(u,v)f_u$ to be generated, we set $c$ to $y$ and $n$ to $u$.
  \begin{tabbing}
    \qquad\=\quad\=\quad\=\quad\=\quad\=\quad\=\quad\=\quad\=\kill
    \( \layerIF{0}(n) := \neg \exists n' \, E(n',n) \) \\[\jot]
    \( \layerIF{j}(n) := \exists n' (E(n',n) \land \layerIF{j-1}(n')) \)\\[\jot]
    \> (for $1 \leq j \leq i+1$) \\[\jot]
    \( \domF(c,n) :=\)\\[\jot]
    \> \( (\domR{\Upsnet} (c) \land c = n) \lor (\domR{\PWLV{m}}(c) \land \layerIF{i+1}(n)) \) \\[\jot]
    \( \varphi_{\domR{\PWLV{m}}}(c,n) := \domR{\PWLV{m}}(c) \)\\[\jot]
    \( \varphi_{\domR{\Upsnet}}(c,n) := \domR{\Upsnet}(c) \)\\[\jot]
    \( \varphi_{\breakplaneR}(c,n) := \varphi_{\domR{\PWLV{m}}}(c,n) \land \breakplaneR(c) \)\\[\jot]
    \( \varphi_{\ptopeR}(c,n) := \varphi_{\domR{\PWLV{m}}}(c,n) \land \ptopeR(c) \)\\[\jot]
    \( \varphi_{\coefPWLW{j}}(c,n) := \)\\[\jot]
    \> \(\ifText \varphi_{\breakplaneR}(c,n) \thenText \coefPWLW{j}(c)\) \\[\jot]
    \> \(\displaystyle \elseTextNS (\ifText \varphi_{\ptopeR}(c,n) \thenText \sum_{n_i: \genNodeR(c,n_i)} \coefPWLW{j}(c) \cdot w(n_i,n)\) \\[\jot]
    \> \> \(\elseTextNS \bot) \)\\[\jot]
    \> (for $j \in \{0,\dots,m\}$)\\[\jot]
    \( \varphi_{\posSR}(c,n,c',n') := \posSR(c,c') \)\\[\jot]
    \( \varphi_{\negSR}(c,n,c',n') := \negSR(c,c') \)\\[\jot]
    \( \varphi_{\onSR}(c,n,c',n') := \onSR(c,c') \)\\[\jot]
    \( \varphi_{\inn_i} := (\inn_i, \inn_i) \) \\[\jot]
    \> (for $i \in \{1,\dots,m\}$) \\[\jot]
    \( \varphi_{\out_1} := (\out_1, \out_1) \) \\[\jot]
    \( \varphi_{E}(c,n,c',n') := E(c,c') \) \\[\jot]
    \( \varphi_{b}(c,n) := b(c)\)\\[\jot]
    \( \varphi_{w}(c,n,c',n') := w(c,c')\)\\[\jot]
    \( \varphi_{\genEdgeR}(c,n,c',n',c'',n'') := \)\\[\jot]
    \> \( \varphi_{\domR{\PWLV{m}}}(c,n) \land \varphi_{\domR{\Upsnet}}(c',n') \land \varphi_{\domR{\Upsnet}}(c',n') \)\\[\jot]
    \> \( {} \land \genNodeR(c,c') \land n = n'' \)
  \end{tabbing}
\end{proof}

\subsection{Summing}

Next, we assume we have represented the functions $f_{uv} = w(u,v) \cdot \Funcu{\Net}{u}$ from layer $i$ to layer $i+1$, in the previous step. For each neuron in layer $i+1$, we want to represent the function $f_v = b(v) + \sum_{u \to v}f_{uv}$.

\begin{lemma}
  \label{lempwlsum}
  Let $m$ and $\ell$ be natural numbers and let $i \in \{0,dots,\ell-1\}$. There exists an $\wal$-translation $\varphi$ from $\PWLEdgeV{m}$ to $\PWLNodeV{m}$ with the following behavior. For any $\Net \in \K(m,\ell)$, let $\Net'$ be the extension of $\Net$ associating PWL functions $f_uv$ to all edges $u \to v$ from layer $i$ to layer $i+1$. Then $\varphi$ maps $\Net'$ to the disjoint union of $\Net$ with a proper structure that associates every neuron $v$ in the $i$-the layer with the function $f_v \coloneqq b(v) + \sum_{u \to v}f_{uv}$.
\end{lemma}

The proof of this lemma proceeds in three steps. The first step collects all breakplanes from all functions $f_{uv}$, as described in the following sublemma:
\begin{lemma}
  \label{lempwlsumarr}
  Let $m$ and $\ell$ be natural numbers. There exists an $\wal$-translation $\varphi$ from $\PWLEdgeV{m}$ to $\hpArrV{m}$ with the following behavior. Let $\mathcal{B}$ be the union of disjoint structures representing functions in $\PL(m)$.
  Then $\varphi$ maps $\mathcal{B}$ to the arrangement consisting of all breakplanes of all functions of $\mathcal{B}$.
\end{lemma}
\begin{proof}
  \begin{tabbing}
    \quad\=\,\,\,\=\,\,\,\=\,\,\,\=\,\,\,\=\,\,\,\=\,\,\,\=\,\,\,\=\kill
    \\
    \( \domF(h) := \breakplaneR(h) \) \\[\jot]
    \( \varphi_{\coefArrW{m}{i}}(h) := \coefPWLW{i}(h) \) \\[\jot]
    \> (for $i \in \{0,\dots,m\}$) \\[\jot]
    \( \displaystyle \eqF(h,h') := \bigwedge_{i \in \{0,\dots,m\}} \varphi_{\coefArrW{m}{i}}(h) = \varphi_{\coefArrW{m}{i}}(h') \)
  \end{tabbing}
\end{proof}

The second step is to invoke Lemma~\ref{lemcd} to create a CD for the arrangement from the first step.
In the third step, using this CD, we can do the actual summing. Thus, the following sublemma is a specialization of the Lemma~\ref{lempwlsum} we need to prove.

\begin{lemma}
  \label{lempwlsumcd}
  Let $m$ and $\ell$ be natural numbers and let $i \in \{0,\dots,\ell-1\}$. There exists an $\wal$-translation $\varphi$ from $(\Upsnet \uplus \PWLV{m} \uplus \cellV{m}) \cup \Upspwledge$ to $\PWLNodeV{m}$ with the following behavior. For any $\Net \in \K(m,\ell)$, let $\Net'$ be the extension of $\Net$ associating PWL functions $f_uv$ to all edges $u \to v$ from layer $i$ to layer $i+1$. Let $\Net''$ be an extension of $\Net'$ by a CD $\D$ compatible with the union of all breakplanes of all functions $f_{uv}$. Then $\varphi$ maps $\Net''$ to the disjoint union of $\Net$ with a proper structure that associates every neuron $v$ in the $i$-the layer with the function $f_v \coloneqq b(v) + \sum_{u \to v}f_{uv}$.
\end{lemma}
\begin{proof}
  We will encode any element of the domain as pairs of the form $(c,n)$. If it encodes a node $n' \in \Net$, we simply set $c=n=n'$. If $(c,n)$ encodes a breakplane of a function $f_{v}$ to be generated, $c$ will be a breakplane of a function $f_{uv}$ in the input and $n=v$. If $(c,n)$ encodes a polytope of a function $f_v$, $c$ will be a cell in $\D$ and $n=v$. Thus, the formulas defining the domains and sorts are as follows:
  \begin{tabbing}
    \quad\=\,\,\,\=\,\,\,\=\,\,\,\=\,\,\,\=\,\,\,\=\,\,\,\=\,\,\,\=\kill
    \( \domF(c,n) :=\)\\[\jot]
    \> \( (\domR{\Upsnet} (c) \land c = n) \)\\[\jot]
    \> \( {} \lor (\domR{\PWLV{m}}(c) \land \breakplaneR(c) \land \exists u\,\genEdgeR(c,u,n)) \) \\[\jot]
    \> \( {} \lor (\domR{\cellV{m}}(c) \land \cellR(c) \land \layerIF{i+1}(n)) \) \\[\jot]
    \( \varphi_{\domR{\PWLV{m}}}(c,n) := \domR{\PWLV{m}}(c) \lor \domR{\cellV{m}}(c) \)\\[\jot]
    \( \varphi_{\domR{\Upsnet}}(c,n) := \domR{\Upsnet}(c) \)\\[\jot]
    \( \varphi_{\breakplaneR}(c,n) := \varphi_{\domR{\PWLV{m}}}(c,n) \land \breakplaneR(c) \)\\[\jot]
    \( \varphi_{\ptopeR}(c,n) := \varphi_{\domR{\PWLV{m}}}(c,n) \land \domR{\cellV{m}}(c) \)
  \end{tabbing}

  Next, we need to find a way to map the breakplanes in the PWL
  function to a hyperplane in $\D$. We do this with the formula 
  $\bpToHpF(b,h)$ simply by finding the hyperplane $h$ in $\D$ 
  that has the same coefficients as the breakplane $b$. We know
  that such a hyperplane is guaranteed to exist by the definition
  of $\D$. This then allows us to define the formulas
  $\hpPosCR(c,b)$, $\hpNegCR(c,b)$, and $\hpOnCR(c,b)$ that
  select all $m$-cells $c$ that are on the positive side,
  negative side, or on top of the breakplane $b$ respectively. We
  do this by mapping the breakplanes to their
  corresponding hyperplane in $\D$ and then select the cells on
  the correct side of this hyperplane using $\hpPosCCHpR(c,h)$,
  $\hpNegCCHpR(c,h)$, and $\hpOnCCHpR(c,h)$ from Lemma~\ref{lemselrelhp}. 
  These formulas select the cells on the positive side, negative side and on the hyperplane in the cell decomposition respectively.
  We can also use these formulas in the obvious manner to define the formulas $\varphi_{\posSR}(c,n,c',n')$, $\varphi_{\negSR}(c,n,c',n')$, and $\varphi_{\onSR}(c,n,c',n')$.
  The described $\wal$ formulas are written as follows:
  \begin{tabbing}
    \quad\=\,\,\,\=\,\,\,\=\,\,\,\=\,\,\,\=\,\,\,\=\,\,\,\=\,\,\,\=\kill
    \( \bpToHpF(b,h) := \) \\[\jot]
    \> \( \displaystyle (\bigwedge_{i \in \{0,\dots,m\}} \coefPWLW{i}(b) = \coefCellW{i}(h)) \land \domR{\cellV{d}}(h) \land \hpR(h) \)\\[\jot]
  
    \( \hpPosCR(c,b) :=\) \\[\jot]
    \> \( \exists h (\bpToHpF(b,h) \land \hpPosCCHpR(c,h))\) \\[\jot]
    \( \hpNegCR(c,b) :=\) \\[\jot]
    \> \( \exists h (\bpToHpF(b,h) \land \hpNegCCHpR(c,h))\) \\[\jot]
    \( \hpOnCR(c,b) :=\) \\[\jot]
    \> \( \exists h (\bpToHpF(b,h) \land \hpOnCCHpR(c,h))\) \\[\jot]

    \( \varphi_{\posSR}(c,n,c',n') := \)\\[\jot]
    \> \( \varphi_{\ptopeR}(c,n) \land \varphi_{\breakplaneR}(c',n') \land n = n \land \hpPosCR(c,c') \)\\[\jot]
    \( \varphi_{\negSR}(c,n,c',n') := \)\\[\jot]
    \> \( \varphi_{\ptopeR}(c,n) \land \varphi_{\breakplaneR}(c',n') \land n = n \land \hpNegCR(c,c') \)\\[\jot]
    \( \varphi_{\onSR}(c,n,c',n') := \)\\[\jot]
    \> \( \varphi_{\ptopeR}(c,n) \land \varphi_{\breakplaneR}(c',n') \land n = n \land \hpOnCR(c,c') \)
  \end{tabbing}

  Recall that the cells from $\D$ will serve as the polytopes for $f_v$. Let us refer to the polytopes of all functions $f_{uv}$ as the ``old polytopes''. For each cell $c$ we must know the old polytopes $p$, containing $c$. Thereto we introduce the formula $\cellToPtopeF(c,n,p)$. The coefficients $\coefPWLW{j}(c,n)$ of the component functions of $f_v$ can then be defined by summation. The coefficients of the breakplanes remain unchanged.

  \begin{tabbing}
    \quad\=\,\,\,\=\,\,\,\=\,\,\,\=\,\,\,\=\,\,\,\=\,\,\,\=\,\,\,\=\kill
    \( \cellToPtopeF(c,n,p) :=  \) \\[\jot]
    \> \( \varphi_{\ptopeR}(c,n) \land \domR{\PWLV{m}}(p) \land \ptopeR(p)\)\\[\jot]
    \> \( {} \land \exists u (\genEdgeR(p,u,n) \) \\[\jot]
    \> \> \( {} \land \forall b ((\breakplaneR(b) \land \genEdgeR(b,u,n))\)\\[\jot]
    \> \> \> \( {} \implies ((\posSR(p,b) \land \hpPosCR(c,b))\)\\[\jot]
    \> \> \> \> \( {} \lor (\negSR(p,b) \land \hpNegCR(c,b)) \) \\[\jot]
    \> \> \> \> \( {} \lor (\onSR(p,b) \land \hpOnCR(c,b))))) \) \\[\jot]
    
    \( \varphi_{\coefPWLW{j}}(c,n) := \)\\[\jot]
    \> \(\ifText \varphi_{\breakplaneR}(c,n) \thenText \coefPWLW{j}(c)\) \\[\jot]
    \> \( \displaystyle \elseTextNS (\ifText \varphi_{\ptopeR}(c,n) \thenText \sum_{p: \cellToPtopeF(c,n,p)} \coefPWLW{j}(p)\) \\[\jot]
    \> \> \(\elseTextNS \bot ) \)\\[\jot]
    \> (for $j \in \{1,\dots,m\}$)\\[\jot]
    \( \varphi_{\coefPWLW{0}}(c,n) := \)\\[\jot]
    \> \(\ifText \varphi_{\breakplaneR}(c,n) \thenText \coefPWLW{0}(c)\) \\[\jot]
    \> \( \elseTextNS (\ifText \varphi_{\ptopeR}(c,n) \)\\[\jot] 
    \> \> \( \displaystyle \thenTextNS b(n) + (\sum_{p: \cellToPtopeF(c,n,p)} \coefPWLW{0}(p))\) \\[\jot]
    \> \> \( \elseTextNS \bot ) \)\\[\jot]

    \( \varphi_{E}(c,n,c',n') := E(c,c') \) \\[\jot]
    \( \varphi_{b}(c,n) := b(c)\)\\[\jot]
    \( \varphi_{w}(c,n,c',n') := w(c,c')\)\\[\jot]
    \( \varphi_{\inn_i} := (\inn_i, \inn_i) \) \\[\jot]
    \> (for $i \in \{1,\dots,m\}$) \\[\jot]
    \( \varphi_{\out_1} := (\out_1, \out_1) \) \\[\jot]

    \( \varphi_{\genNodeR}(c,n,c',n') := \varphi_{\domR{\PWLV{m}}}(c,n) \land \varphi_{\domR{\Upsnet}}(c',n') \land n=n' \)\\[\jot]
  \end{tabbing}

  Different cells representing the same polytope for some $f_v$ must be equated by the equality formula. Thereto we use the auxiliary formula $\eqCellF$. The duplicate breakplanes must also be equated.
  \begin{tabbing}
    \quad\=\,\,\,\=\,\,\,\=\,\,\,\=\,\,\,\=\,\,\,\=\,\,\,\=\,\,\,\=\kill
    \( \eqHpF(c,n,c',n') := \)\\[\jot]
    \> \( \varphi_{\breakplaneR}(c,n) \land \varphi_{\breakplaneR}(c',n') \land n=n' \)\\[\jot]
    \> \( \displaystyle {} \land ( \bigwedge_{i \in \{0,\dots,m\}} \varphi_{\coefPWLW{i}}(c,n) = \varphi_{\coefPWLW{i}}(c',n') )\) \\[\jot]
    \( \eqCellF(c,n,c',n') := \)\\[\jot]
    \> \( \varphi_{\ptopeR}(c,n) \land \varphi_{\ptopeR}(c',n') \land n=n' \)\\[\jot]
    \> \( {} \land \forall p ((\breakplaneR(p) \land \genNodeR(p,n)) \)\\[\jot]
    \> \> \( {} \implies ((\hpPosCR(c,p) \land \hpPosCR(c',p)) \)\\[\jot]
    \> \> \> \( {} \lor (\hpNegCR(c,p) \land \hpNegCR(c',p))\)\\[\jot]
    \> \> \> \( {} \lor (\hpOnCR(c,p) \land \hpOnCR(c',p)))) \)\\[\jot]
    \( \eqF(c,n,c',n') := \) \\[\jot]
    \> \( n=n' \)\\[\jot]
    \> \( {} \land ((\varphi_{\domR{\Upsnet}}(c,n) \land c=c') \)\\[\jot]
    \> \> \( {} \lor (\varphi_{\breakplaneR}(c,n) \land \varphi_{\breakplaneR}(c',n') \)\\[\jot]
    \> \> \> \( {} \land \eqHpF(c,n,c',n')) \)\\[\jot]
    \> \> \( {} \lor (\varphi_{\ptopeR}(c,n) \land \varphi_{\ptopeR}(c',n') \land \eqCellF(c,n,c',n'))) \)
  \end{tabbing}
\end{proof}

\subsection{ReLU}

At this point in the proof we have PWL functions $f_v$ for every neuron in the $i+1$-th layer, given together in a $\PWLV{m} \cup \Upspwlnode$-structure. In order to obtain the desired function $\Funcu{\Net}{v}$, it remains to apply $\ReLU$ to each function.

\begin{lemma}
  \label{lempwlrelu}
  Let $m$ be a natural number. There exists an $\wal$-translation $\varphi$ from $\PWLV{m} \cup \Upspwlnode$ to $\PWLV{m} \cup \Upspwlnode$ with the following behavior.  Let $\B$ be a structure representing a family $(f_v)_v$ of functions $f_v \in \PL(m)$ linked to the identifier $v$ by the relation $\genNodeR$ as before. Then $\varphi$ maps $\B$ to a structure representing the family $({f_v}')_v$ with ${f_v}' = \ReLU(f_v)$.
\end{lemma}

As like Lemma~\ref{lempwlsum}, the proof of this Lemma proceeds in three steps. The first step is similar to Lemma~\ref{lempwlsumarr}. This time it collects not only all breakplanes from all functions $f_v$ but also from each component function defined for the polytopes, the hyperplane where that function is zero. It is important to include these zero-hyperplanes, so that an arrangement is obtained that induces a partition of $\R^m$ in polytopes that is fine enough to represent $\ReLU(f_v)$.

\begin{lemma}
  \label{lempwlreluarr}
  Let $m$ be a natural number. There exists an $\wal$-translation $\varphi$ from $\PWLV{m} \cup \Upspwlnode$ to $\hpArrV{m}$ with the following behavior. Let $\B$ be the union of disjoint structures representing functions in $\PL(m)$.
  Then $\varphi$ maps $\B$ to the arrangement consisting of all breakplanes of all functions of $\B$ plus all zero hyperplanes of all the component functions.
\end{lemma}
\begin{proof}
  We use each breakplane and each cell to encode their respective functions as the hyperplanes to be produced.
  \begin{tabbing}
		\qquad\=\quad\=\quad\=\quad\=\quad\=\quad\=\quad\=\quad\=\kill
		\( \domF(c) := \breakplaneR(c) \lor \cellR(c) \) \\[\jot]
		\( \varphi_{\coefArrW{m}{i}}(t,c) := \coefPWLW{i}(c) \) \\[\jot]
		\> (for $i \in \{0,\dots,m\}$) \\[\jot]
    \( \displaystyle \eqF(h,h'') := \bigwedge_{i \in \{0,\dots,m\}} \varphi_{\coefArrW{m}{i}}(h) = \varphi_{\coefArrW{m}{i}}(h') \)
	\end{tabbing}
\end{proof}

Again, the second step is to invoke Lemma~\ref{lemcd} to create a CD for the arrangement from the first step.

In the third step, using this CD, we do the actual application of $\ReLU$. Thus, the following sublemma is a specialization of the Lemma~\ref{lempwlrelu} we need to prove.
\begin{lemma}
  \label{lempwlrelucd}
  Let $m$ be a natural number. There exists an $\wal$-translation $\varphi$ from $(\PWLV{m} \uplus \cellV{m}) \cup \Upspwlnode$ to $\PWLV{m} \cup \Upspwlnode$ with the following behavior.  Let $\B$ be a structure representing a family $(f_v)_v$ of functions $f_v \in \PL(m)$ linked to the identifier $v$ by the relation $\genNodeR$ as before, and let $\B'$ be an extension of $\B$ by a CD $\D$ compatible with the union of all breakplanes and all zero hyperplanes of all component functions of all function $f_v$. Then $\varphi$ maps $\B'$ to a structure representing the family $({f_v}')_v$ with ${f_v}' = \ReLU(f_v)$.
\end{lemma}
\begin{proof}
  The proof largely follows the proof of Lemma~\ref{lempwlsumcd}.
  The breakplanes of the functions $\ReLU(f_v)$ to be generated are now encoded by pairs $(c,n)$ where $n=v$ and $c$ is a breakplane or polytope of $f_v$. (If $c$ is a polytope, then $(c,n)$ represents the corresponding zero hyperplane.) The polytopes of the new functions are again encoded by pairs $(c,n)$ where $c$ is a cell of $D$.
  In the auxiliary formula $\cellToPtopeF$, we now use $\genNodeR$ instead of $\genEdgeR$.

  \begin{tabbing}
    \quad\=\,\,\,\=\,\,\,\=\,\,\,\=\,\,\,\=\,\,\,\=\,\,\,\=\,\,\,\=\kill
    \( \domF(c,n) :=\)\\[\jot]
    \> \( (\breakplaneR(c) \land \genNodeR(c,n))\)\\[\jot]
    \> \( {} \lor (\ptopeR(c) \land \genNodeR(c,n)) \)\\[\jot]
    \> \( {} \lor (\domR{\cellV{m}}(c) \land \cellR(c) \land \exists c'(\genNodeR(c',n))) \)\\[\jot]
    \> \( {} \lor (c = n \land \exists c'(\genNodeR(c',n))) \)\\[\jot]
    \( \varphi_{\breakplaneR}(c,n) := \ptopeR(c) \lor \breakplaneR(c) \)\\[\jot]
    \( \varphi_{\ptopeR}(c,n) := \domR{\cellV{m}}(c) \)\\[\jot]

    \( \bpToHpF(b,h) := \) \\[\jot]
    \> \( \displaystyle (\bigwedge_{i \in \{0,\dots,m\}} \coefPWLW{i}(b) = \coefCellW{i}(h)) \land \domR{\cellV{d}}(h) \land \hpR(h) \)\\[\jot]
  
    \( \hpPosCR(c,b) :=\) \\[\jot]
    \> \( \exists h (\bpToHpF(b,h) \land \hpPosCCHpR(c,h))\) \\[\jot]
    \( \hpNegCR(c,b) :=\) \\[\jot]
    \> \( \exists h (\bpToHpF(b,h) \land \hpNegCCHpR(c,h))\) \\[\jot]
    \( \hpOnCR(c,b) :=\) \\[\jot]
    \> \( \exists h (\bpToHpF(b,h) \land \hpOnCCHpR(c,h))\) \\[\jot]

    \( \varphi_{\posSR}(c,n,c',n') := \)\\[\jot]
    \> \( \varphi_{\ptopeR}(c,n) \land \varphi_{\breakplaneR}(c',n') \land n = n \land \hpPosCR(c,c') \)\\[\jot]
    \( \varphi_{\negSR}(c,n,c',n') := \)\\[\jot]
    \> \( \varphi_{\ptopeR}(c,n) \land \varphi_{\breakplaneR}(c',n') \land n = n \land \hpNegCR(c,c') \)\\[\jot]
    \( \varphi_{\onSR}(c,n,c',n') := \)\\[\jot]
    \> \( \varphi_{\ptopeR}(c,n) \land \varphi_{\breakplaneR}(c',n') \land n = n \land \hpOnCR(c,c') \)\\[\jot]

    \( \cellToPtopeF(c,n,p) :=  \) \\[\jot]
    \> \( \varphi_{\ptopeR}(c,n) \land \domR{\PWLV{m}}(p) \land \ptopeR(p) \land \genNodeR(p,n) \) \\[\jot]
    \> \( {} \land \forall b ((\breakplaneR(b) \land \genNodeR(b,n))\)\\[\jot]
    \> \> \( {} \implies ((\posSR(p,b) \land \hpPosCR(c,b))\)\\[\jot]
    \> \> \> \( {} \lor (\negSR(p,b) \land \hpNegCR(c,b)) \) \\[\jot]
    \> \> \> \( {} \lor (\onSR(p,b) \land \hpOnCR(c,b)))) \)
  \end{tabbing}

  The main difference in the proof is the definition of the coefficients of the component functions of the functions $\ReLU(f_v)$ to be generated. The auxiliary formula $\posPtopeF$ checks whether the component function above a given polytope is entirely positive. If this is not the case, the coefficients of the component function above this polytope will all become zero.
  \begin{tabbing}
    \quad\=\,\,\,\=\,\,\,\=\,\,\,\=\,\,\,\=\,\,\,\=\,\,\,\=\,\,\,\=\kill
    \( \posPtopeF(c,n) := \)\\[\jot]
    \> \(\exists p( \cellToPtopeF(c,n,p) \land \genNodeR(p,n) \land \hpPosCR(c,p)) \) \\[\jot]
    \( \varphi_{\coefPWLW{i}}(c,n) := \)\\[\jot]
    \> \( \ifText \varphi_{\breakplaneR}(c,n) \thenText \coefPWLW{i}(c)\) \\[\jot]
    \> \( \elseTextNS (\ifText \varphi_{\ptopeR}(c,n) \land \posPtopeF(c,n) \)\\[\jot]
      \> \> \> \> \( \displaystyle \thenTextNS \sum_{p: \cellToPtopeF(c,n,p)}\coefPWLW{i}(p)\) \\[\jot]
      \> \> \> \( \elseTextNS (\ifText \domR{\cellV{m}}(c) \land \neg \posPtopeF(c,n) \thenText 0\) \\[\jot]
        \> \> \> \> \( \elseTextNS \bot )) \)\\[\jot]
    \> (for $i \in \{0,\dots,m\}$)\\[\jot]

    \( \varphi_{\genNodeR}(c,n,c',n') := \)\\[\jot]
    \> \( \varphi_{\domR{\PWLV{m}}}(c,n) \land \genNodeR(c,c') \land \genNodeR(c,n') \land c' = n' \)\\[\jot]
    \( \eqHpF(c,n,c',n') := \)\\[\jot]
    \> \( \varphi_{\breakplaneR}(c,n) \land \varphi_{\breakplaneR}(c',n') \land n=n' \)\\[\jot]
    \> \( \displaystyle {} \land ( \bigwedge_{i \in \{0,\dots,m\}} \varphi_{\coefPWLW{i}}(c,n) = \varphi_{\coefPWLW{i}}(c',n') )\) \\[\jot]
    \( \eqCellF(c,n,c',n') := \)\\[\jot]
    \> \( \varphi_{\ptopeR}(c,n) \land \varphi_{\ptopeR}(c',n') \land n=n' \)\\[\jot]
    \> \( {} \land \forall b (\varphi_{\breakplaneR}(b,n) \)\\[\jot]
    \> \> \( {} \implies ((\varphi_{\posSR}(c,n, b,n) \land \varphi_{\posSR}(c',n', b,n)) \)\\[\jot]
    \> \> \> \( {} \lor (\varphi_{\negSR}(c,n, b,n) \land \varphi_{\negSR}(c',n', b,n))\)\\[\jot]
    \> \> \> \( {} \lor (\varphi_{\onSR}(c,n, b,n) \land \varphi_{\onSR}(c',n', b,n)))) \)\\[\jot]
    \( \eqF(c,n,c',n') := \) \\[\jot]
    \> \( n=n' \)\\[\jot]
    \> \( {} \land ((\land c=c') \)\\[\jot]
    \> \> \( {} \lor (\varphi_{\breakplaneR}(c,n) \eqHpF(c,n,c',n')) \)\\[\jot]
    \> \> \( {} \lor (\varphi_{\ptopeR}(c,n) \land \eqCellF(c,n,c',n'))) \)
  \end{tabbing}
\end{proof}

\section{Lemma~\ref{lemsel} for general $m$}
\label{seclemsel}

Lemma~\ref{lemsel} as stated in Section~\ref{subproofmainthm} uses the notion of the hyperplane arrangement $\Arr_f$ corresponding to a PWL function $f:\R^m \to \R$. However, we only defined this for $m=1$ in Section~\ref{subproofmainthm}. Here we define $\Arr_f$ in general and also show that $\Arr_f$ can be uniformly constructed from $f$ by an $\wal$ translation.

\subsection{Definition and creation of hyperplane arrangement of PWL functions}
\label{secpwlarr}
\newcommand{\Upsconstarr}[2]{\Upsconst_{\mathrm{arr},{#1},{#2}}}

Consider a PWL function $f:\R^m \to R$, given by a set $B$ of breakplanes (hyperplanes in $\R^m$) so that $f$ is an affine function on every cell of the partition of $\R^m$ induced by $B$. Recall that we refer to these cells as polytopes, to avoid confusion with the cells in a cylindrical cell decomposition. Let $d > m$. We define the hyperplane arrangement corresponding to $f$ in $d$ dimensions, to consist of the following hyperplanes:
\begin{itemize}
  \item For each hyperplane $a_0 + a_1 x_1 + \cdots + a_m x_m = 0$ in $B$, and for each combination of indices $0 < g_1 < \cdots < g_m < d$, we include the hyperplane $a_0 + a_1 x_{g_1} + \cdots + a_m x_{g_m} = 0$.
  \item For each polytope $p$, let $f$ restricted to $p$ be the affine function $y = a_0 + a_1 x_1 + \cdots + a_m x_m$. Then for each combination $0 < g_1 < \cdots < g_{m+1} \leq d$ we include the hyperplane $x_{g_{m+1}} = a_0 + a_1 x_{g_1} + \cdots + a_m x_{g_m}$.
\end{itemize}
We denote this hyperplane arrangement by $\Arr_f$ when $d$ is understood.

In the next lemma we assume a $d$-dimensional hyperplane arrangement $\Arr_H$. We make use of the vocabulary $\Upsconstarr{m}{H}$ containing the following auxiliary constant symbols:

\begin{itemize}
  \item $\dimC{j}$, for $j\in\{1,\dots,m+1\}$;
  \item $\noneDimC$;
  \item $\extraHpC{h}$, for each $h \in \Arr_H$
\end{itemize}

\begin{lemma}
  \label{lemfuncHarr}
  Let $m$ and $d$ be natural numbers with $d>m$, and let $\Arr_H$ be a $d$-dimensional hyperplane arrangement.
  There exists an $\wal$ translation $\varphi$ from $\PWLV{m}$ from $\PWLV{m} \cup \Upsconstarr{m}{H}$ to $\hpArrV{d}$ with the following behavior.
  For any $\PWLV{m}$-structure $\F$ representing a function $f \in \PL(m)$, let $\F'$ be $\F$ expanded with fresh constants for $\Upsconstarr{m}{H}$. Then $\varphi$ maps $\F'$ to $\Arr_f \cup \Arr_H$.
\end{lemma}
\begin{proof}
  We refer to the definition of $\Arr_f$ given above. We encode the hyperplanes in $\Arr_f \cup \Arr_{H}$ by tuples $(q_1, \dots, q_d, p)$.
  \begin{itemize}
    \item Each $h \in Arr_H$ is encoded by $q_1 = \cdots = q_d = p = \extraHpC{h}$.
    \item Each hyperplane $a_0 + a_1 x_{g_1} + \cdots + a_m x_{g_m} = 0$ coming from a breakplane $p \in B$ is encoded by the tuple $(q_1, \dots, q_d, p)$ where $q_{g_j} = \dimC{j}$ for $j \in \{1,\dots,m\}$ and $q_i = \noneDimC$ for $i \in \{1,\dots,d\} \setminus \{g_1, \dots, g_m\}$.
    \item Each hyperplane $x_{g_{m+1}} = a_0 + a_1 x_{g_1} + \cdots + a_m x_{g_m}$ coming from the definition of $f$ on a polytope $p$ is encoded by the tuple $(q_1, \dots, q_d, p)$ where $q_{g_j} = \dimC{j}$ for $j \in \{1,\dots,m+1\}$ and $q_i = \noneDimC$ for $i \in \{1,\dots,d\} \setminus \{g_1, \dots, g_{m+1}\}$.
  \end{itemize}
  \begin{tabbing}
    \quad\=\,\,\,\=\,\,\,\=\,\,\,\=\,\,\,\=\,\,\,\=\,\,\,\=\,\,\,\=\kill
    \(\domF(q_1, \dots, q_d, p) :=\) \\[\jot]
    \> \( \displaystyle(q_1 = \dots = q_d = p \land (\bigvee_{i \in \{1,\dots,d\}} p = \extraHpC{i})) \) \\[\jot]
    \> \( {} \lor ((\ptopeR(p)\) \\[\jot]
    \> \> \> \( \displaystyle {} \land \biggl(\bigvee_{g_1,\dots,g_{m+1} \in \{1,\dots,d\} \land g_1 < \cdots < g_{m+1}}\,\, (\bigwedge_{j\in \{1,\dots,{m+1}\}}q_{g_j} = \dimC{j}) \)\\[\jot]
    \> \> \> \> \( \displaystyle {} \land (\bigwedge_{k \in \{1, \dots, d\}\setminus \{g_1,\dots,g_{m+1}\}} q_k = \noneDimC)\biggr))\)\\[\jot]
    \> \> \( {} \lor (\breakplaneR(p)\) \\[\jot]
    \> \> \> \( \displaystyle {} \land \biggl(\bigvee_{g_1,\dots,g_{m} \in \{1,\dots,d-1\} \land g_1 < \cdots < g_{m}}\,\, (\bigwedge_{j\in \{1,\dots,{m}\}}q_{g_j} = \dimC{j}) \)\\[\jot]
    \> \> \> \> \( \displaystyle {} \land (\bigwedge_{k \in \{1, \dots, d\}\setminus \{g_1,\dots,g_{m}\}} q_k = \noneDimC)\biggr)))\)
  \end{tabbing}

  The coefficients are now set accordingly using the following formulas:
  \begin{tabbing}
    \quad\=\,\,\,\=\,\,\,\=\,\,\,\=\,\,\,\=\,\,\,\=\,\,\,\=\,\,\,\=\kill
    \( \varphi_{\coefArrW{d}{i}}(q_1, \dots, q_d, p) := \) \\[\jot]
    \> \( \ifText q_i = \dimC{(m+1)} \thenTextNS -1 \) \\[\jot]
    \> \( \elseTextNS \ifText q_i = \dimC{j} \thenText \coefPWLW{j}(p) \) \\
    \> \vdots \quad \raisebox{0.2em}{(for each $j \in \{1,\dots,m\}$)} \\
    \> \( \elseTextNS \ifText q_i = \noneDimC \thenText 0 \) \\[\jot]
    \> \( \elseTextNS \ifText p = \extraHpC{h} \thenText c_{h,i} \) \\
    \> \vdots \quad \raisebox{0.2em}{(for each $h \in \Arr_H$)} \\
    \> \( \elseTextNS \bot \) \\[\jot]
    \> (for \( i \in \{1, \dots, d\} \)) \\[\jot]
    
    \( \varphi_{\coefArrW{d}{0}}(q_1, \dots, q_d, p) := \) \\[\jot]
    \> \( \ifText \ptopeR(p) \lor \breakplaneR(p) \thenText \coefPWLW{0}(p) \) \\[\jot]
    \> \( \elseTextNS \ifText p = \extraHpC{h} \thenText c_{h,0} \) \\
    \> \vdots \quad \raisebox{0.2em}{(for each $h \in \Arr_H$)} \\
    \> \( \elseTextNS \bot \)
  \end{tabbing}

  Finally, the equality formula simply checks that the coefficients are the same.
  \begin{tabbing}
    \quad\=\,\,\,\=\,\,\,\=\,\,\,\=\,\,\,\=\,\,\,\=\,\,\,\=\,\,\,\=\kill
    \( \eqF(q_1, \dots, q_d,q_1', \dots, q_d') := \)\\[\jot]
    \> \( \displaystyle \bigwedge_{i \in \{0,\dots,d\}} \varphi_{\coefArrW{d}{i}}(q_1, \dots, q_d) = \varphi_{\coefArrW{d}{i}}(q_1', \dots, q_d')\)
  \end{tabbing}
\end{proof}

Recall the hyperplane arrangement $\Arr_\psi$ introduced for any ordered $\bblin$ formula $\psi$ in Section~\ref{subproofmainthm}.
The following corollary follows directly from Lemma~\ref{lemfuncHarr} by substituting $\Arr_\psi$ for $\Arr_H$.
\begin{corollary}
  \label{corformulaarr}
  Let $m$ and $d$ be natural numbers with $d>m$, and let $\psi$ be an ordered $\bblin$ formula on $d$ variables (free or bound).
  There exists an $\wal$ translation $\varphi$ from $\PWLV{m}$ from $\PWLV{m} \cup \Upsconstarr{m}{\psi}$ to $\hpArrV{d}$ with the following behavior.
  For any $\PWLV{m}$-structure $\F$ representing a function $f \in \PL(m)$, let $\F'$ be $\F$ expanded with fresh constants for $\Upsconstarr{m}{\psi}$. Then $\varphi$ maps $\F'$ to $\Arr_f \cup \Arr_\psi$.
\end{corollary}

\subsection{Cell selection for general $m$}
\label{subcellselec}
\newcommand{\NNCR}[1]{\textrm{F-cell}_{{#1}}}
\newcommand{\baseChainF}[1]{\textit{base-chain}_{#1}}
\newcommand{\beforeF}[1]{\textit{lower-in-}{#1}\textit{-stack}}
\newcommand{\afterF}[1]{\textit{higher-in-}{#1}\textit{-stack}}
\newcommand{\onF}[1]{\textit{on-hp-in-}{#1}}
\newcommand{\firstNonVerticalF}[1]{\textit{largest-non-vertical-}{#1}}
\newcommand{\projectVertF}[1]{\textit{projected-vertical-hp-}{#1}}
\newcommand{\insertedBpF}[1]{\textit{breakplane-in-context}_{#1}}
\newcommand{\insertedCellF}[1]{\textit{function-in-context}_{#1}}
\newcommand{\hpCR}[1]{\textrm{constraint-cell}_{#1}}

In the proof sketch of Lemma~\ref{lemsel} for $m=1$, we described how to select the cells that lie on the formula $F(x_i) = x_j$ as well as the cells that satisfy a linear constraint of the formula $\psi$. Here we will define the formulas $\NNCR{g_1,\dots,g_{m + 1}}(c)$, that select the cells that satisfy the formula $F(x_{g_1}, \dots x_{g_m}) = x_{g_{m+1}}$, as well as the formulas $\hpCR{h}(c)$ for a hyperplane $h \in \Arr_\psi$ of a linear constraint of the formula $\psi$, which select the cells that satisfy that constraint.

In the proof sketch we explained that the first step to selecting those cells representing a certain piece of the continuous PWL function is to select those cells that contain $x_i$ values between the breakpoints of the piece. 
To do this in a cell decomposition we must be able to select cells based on their positioning relative to the breakplanes of the function. Note that these breakplanes are present as hyperplanes in $\Arr_f$ by definition, so they are present in $\D$ since $\D$ is compatible with $\Arr_f$.

More generally, we will show how to select all cells based on their positioning relative to a hyperplane and this in any dimension $i$.
\begin{lemma}
  \label{lemselrelhp}
  There exist an $\wal$-formula $\hpPosCCHpR(c,h)$ over $\cellV{m}$ with the following properties.
  Let $m$ be a natural number, let $\Arr$ be an $m$-dimensional hyperplane arrangement, and let $\D=\D_1,\dots,\D_m$ be a CD compatible with $\Arr$. Then on $\D$, the formula $\hpPosCCHpR$ defines all pairs $(c,h)$ such that $h$ is a hyperplane in $\D$ from $\Arr$ and $c$ is an $i$-cell from $\D$ that is on the positive side of $h$. Similarly, $\hpNegCCHpR$ and $\hpOnCCHpR$ define all triples $(c,h)$ such that $h$ is a hyperplane in $\D$ from $\Arr$ and $c$ is an $i$-cell from $\D$ that is on the negative side of $h$ or contained by $h$ respectively.
\end{lemma}
\begin{proof}
  First, we will define a formula $\inCF{i}(c)$ that selects cells that are in $\D_i$. We can determine this by for each cell, counting how long the chain of base cells is until we reach the origin cell. Next we will define the formula $\beforeF{i}(c,h)$ that selects all $i$-cells $c$ that are below the hyperplane $h$ in any stack in $\D_i$. (Here $h$ lives in $\R^i$.) To do this we take the starting hyperplane of $c$ and check if it is strictly below $h$ in the stack of the base cell of $c$. If it is not, then either it is a section on $h$ or the entire cell is above $h$ in the stack of the base cell of $c$. Similarly, the formula $\afterF{i}(c,h)$ selects all cells $c$ in $\D_i$ that are above the hyperplane $h$. For this we take the ending hyperplane of $c$ and check that is strictly above the $h$ in the stack above the base cell of $c$. Similar to before, if this is not the case, $c$ is either a section on $h$ or is entirely below $h$. Finally, we also add the formula $\onF{i}(c,h)$ which selects those cells $c$ that are sections on the hyperplane $h$, which is done by checking that the $h$ is both a starting and an ending hyperplane of $c$. All of these formulas can be written in $\wal$ as follows:
  \begin{tabbing}
    \quad\=\,\,\,\=\,\,\,\=\,\,\,\=\,\,\,\=\,\,\,\=\,\,\,\=\,\,\,\=\kill
    \( \displaystyle \baseChainF{i}(b_0,b_i) := \exists b_1, \dots, b_{i-1} (\bigwedge_{k \in \{0,\dots,i-1\}}\baseR(b_{k+1}, b_{k})) \) \\[\jot]
    \> (for \( i \in \{1,\dots,d\} \))\\[\jot]
    \( \inCF{i}(c) := \baseChainF{i}(\originCellC,c) \) \\[\jot]
    \> (for \( i \in \{1,\dots,d\} \))\\[\jot]
  
    \( \beforeF{i}(c,h) := \inCF{i}(c) \land \exists b,h_c (\baseR(c,b) \land \startR(c,h_c) \land \stackLtiR{i}(h_c,h,b)) \) \\[\jot]
    \> (for \(i \in \{1,\dots,d\}\)) \\[\jot]
    \( \afterF{i}(c,h) := \inCF{i}(c) \land \exists b,h_l (\baseR(c,b) \land  \EndR(c,h_l) \land \stackLtiR{i}(h,h_l,b)) \) \\[\jot]
    \> (for \(i \in \{1,\dots,d\}\)) \\[\jot]
    \( \onF{i}(c,h) := \inCF{i}(c) \land \EndR(c,h) \land \startR(c,h) \)\\[\jot]
    \> (for \(i \in \{1,\dots,d\}\))
  \end{tabbing}
  
  Using the above, we can expreess that a $d$-cell $c$ is on the positiveside of a hyperplane $h$ from $\D$ that lives in $\R^d$. Indeed, we can use the predicates $\beforeF{i}$ or $\afterF{i}$ where we project down to $i$ dimensions, where $i$ is the highest index where $h$ has a nonzero coefficient.
  \begin{tabbing}
    \quad\=\,\,\,\=\,\,\,\=\,\,\,\=\,\,\,\=\,\,\,\=\,\,\,\=\,\,\,\=\kill
    \( \displaystyle \firstNonVerticalF{i}(h) := \hpR(h) \land \coefCellW{i}(h) \neq 0 \land (\bigwedge_{k \in {i+1,\dots,d}} \coefCellW{k}(h) = 0)  \)\\[\jot]
    \> (for $i \in \{1,\dots,d\}$)
  \end{tabbing}
  
  Now positive side is expressed by $\afterF{i}$, or $\beforeF{i}$, depends on whether the $i$-th coefficient is positive or negative. Similarly, we can express that a cell is on the negative side of a hyperplane. We can also express a cell $c$ is entirely contained by a hyperplane $h$ with the formula $\hpOnCCHpR(c,h)$ using $\onF{i}(c,h)$. The auxiliary formula $\projectVertF{i}(h,h_p)$, retrieves the projection $h_p$ of $h$ that lives in $\R^i$. By Remark~\ref{remverthpproj}, this projection exists in $\D$.
  \begin{tabbing}
    \quad\=\,\,\,\=\,\,\,\=\,\,\,\=\,\,\,\=\,\,\,\=\,\,\,\=\,\,\,\=\kill
    \( \projectVertF{i}(h,h_p) := \) \\[\jot]
    \> \( \displaystyle \hpR(h) \land \hpR(h_p) \land (\bigwedge_{j \in {0,\dots,i}} \coefCellW{j}(h_p) = \coefCellW{j}(h)) \land (\bigwedge_{k \in {i+1,\dots,d}} \coefCellW{k}(h_p) = \bot)  \)\\[\jot]
    \> (for $i \in \{1,\dots,d\}$) \\[\jot]
  
    \( \hpPosCCHpR(c,h) :=\) \\[\jot]
    \> \( \displaystyle \hpR(h) \land (\bigwedge_{i \in \{1,\dots,d\}} \firstNonVerticalF{i}(h)\)\\[\jot]
    \> \> \> \( \implies  \exists c', h'(\inCF{i}(c') \land \baseChainF{d-i}(c',c) \land \projectVertF{i}(h,h') \)\\[\jot]
    \> \> \> \> \( {} \land ((\coefCellW{i}(h') > 0 \land \afterF{i}(c',h'))\) \\[\jot]
    \> \> \> \> \> \( {} \lor (\coefCellW{i}(h') < 0 \land \beforeF{i}(c',h')))))\) \\[\jot]
    \( \hpNegCCHpR(c,h) :=\) \\[\jot]
    \> \( \displaystyle \hpR(h) \land (\bigwedge_{i \in \{1,\dots,d\}} \firstNonVerticalF{i}(h)\)\\[\jot]
    \> \> \> \( \implies \exists c', h'(\inCF{i}(c') \land \baseChainF{d-i}(c',c) \land \projectVertF{i}(h,h')\)\\[\jot]
    \> \> \> \> \( {} \land ((\coefCellW{i}(h) > 0 \land \beforeF{i}(c',h')) \lor (\coefCellW{i}(h) < 0 \land \afterF{i}(c',h')))))\) \\[\jot]
    \( \hpOnCCHpR(c,h) :=\) \\[\jot]
    \> \( \displaystyle \hpR(h) \land (\bigwedge_{i \in \{1,\dots,d\}} \firstNonVerticalF{i}(h)\)\\[\jot]
    \> \> \> \( \implies \exists c',h'(\inCF{i}(c') \land \baseChainF{d-i}(c',c) \land \projectVertF{i}(h,h')\) \\[\jot]
    \> \> \> \> \( {} \land \onF{i}(c',h')))\)
  \end{tabbing}
\end{proof}

We are now ready to describe how to handle a subformula of the form $F(x_{g_1},\dots, x_{g_m}) = x_{g_m+1}$ with $0 < g_1 < \cdots < g_{m+1} \leq d$. We refer to $(g_1, \dots ,g_{m})$ as a \emph{polytope context} and to $(g_1, \dots ,g_{m+1})$ as a \emph{function context}. Recall that $\Arr_f$ (and hence $\D$) contains a hyperplane in context $(g_1, \dots ,g_{m})$ for every breakplane $b$ of $f$, and a hyperplane in context $(g_1, \dots ,g_{m+1})$ for every piece of $f$ on every polytope $p$. The following formulas define these relations

\begin{tabbing}
  \quad\=\,\,\,\=\,\,\,\=\,\,\,\=\,\,\,\=\,\,\,\=\,\,\,\=\,\,\,\=\kill
  \( \insertedBpF{g_1,\dots,g_m}(b,h) := \)\\[\jot]
  \> \( \domR{\PWLV{m}}(b) \land \breakplaneR(b) \land \domR{\cellV{d}}(h) \land \hpR(h) \)\\[\jot]
  \> \( \displaystyle {} \land (\bigwedge_{i \in 1,\dots,m} \coefPWLW{i}(b) = \coefCellW{g_i}(h)) \land (\bigwedge_{i \in \{1,\dots,d\}\setminus \{g_1,\dots,g_m\}} \coefCellW{i}(h) = 0) \)\\[\jot]
  \> (for $1 \leq g_1 < \cdots < g_m < d$) \\[\jot]

  \( \insertedCellF{g_1,\dots,g_{m+1}}(p,h) := \)\\[\jot]
  \> \( \domR{\PWLV{m}}(p) \land \ptopeR(p) \land \domR{\cellV{d}}(h) \land \hpR(h) \)\\[\jot]
  \> \( \displaystyle {} \land (\bigwedge_{i \in 1,\dots,m+1} \coefPWLW{i}(p) = \coefCellW{g_i}(h)) \land (\bigwedge_{i \in \{1,\dots,d\}\setminus \{g_1,\dots,g_{m+1}\}} \coefCellW{i}(h) = 0) \)\\[\jot]
  \> (for $1 \leq g_1 < \cdots < g_{m+1} \leq d$)
\end{tabbing}

Our final formula now defines the $d$-cells satisfying $F(x_{g_1},\dots, x_{g_m}) = x_{g_m+1}$ by checking whether it lies on the hyperplane of a piece of $F$ in context $(g_1, \dots ,g_{m+1})$ and checking whether it has the same relative position to all hyperplanes of breakplanes in context $(g_1, \dots ,g_{m})$ as its polytope using the formulas of Lemma~\ref{lemselrelhp}:
\begin{tabbing}
  \quad\=\,\,\,\=\,\,\,\=\,\,\,\=\,\,\,\=\,\,\,\=\,\,\,\=\,\,\,\=\kill
  \( \NNCR{g_1,\dots,g_{m+1}}(c) := \) \\[\jot]
  \> \( \cellR(c) \) \\[\jot]
  \> \( {} \land \exists p(\domR{\PWLV{m}}(p) \land \ptopeR(p)  \)\\[\jot]
  \> \> \( {} \land \exists h_p (\insertedCellF{g_1,\dots,g_{m+1}}(p,h_p) \land \hpOnCCHpR(c,h_p)) \) \\[\jot]
  \> \> \( {} \land \forall b ((\domR{\PWLV{m}}(b) \land \breakplaneR(b)) \)\\[\jot]
  \> \> \> \( {} \implies  \exists h_b (\domR{\cellV{d}}(h) \land \hpR(h_b) \land \insertedBpF{g_1,\dots,g_m}(b,h_b) \)\\[\jot]
  \> \> \> \> \> \( {} \land ((\posSR(p,b) \land \hpPosCCHpR(c,h_b))\)\\[\jot]
  \> \> \> \> \> \> \( {} \lor (\negSR(p,b) \land \hpNegCCHpR(c,h_b))\)\\[\jot]
  \> \> \> \> \> \> \( {}  \lor (\onSR(p,b) \land \hpOnCCHpR(c,h_b)))))) \) \\[\jot]
  \> (for $1 \leq g_1 < \cdots < g_{m+1} \leq d$)
\end{tabbing}

It remains to deal with atomic subformulas that are linear constraints of the form $c_0 + c_1 x_1 + \cdots + c_d x_d > 0$. These are present as hyperplanes in $\Arr_\psi$. The cells satisfying this constraint are defined as follows:
\begin{tabbing}
  \quad\=\,\,\,\=\,\,\,\=\,\,\,\=\,\,\,\=\,\,\,\=\,\,\,\=\,\,\,\=\kill
  \( \hpCR{h}(c) := \) \\[\jot]
  \> \( \exists h'((\bigwedge_{i \in \{0,\dots,d\}} \coefCellW{i}(h') = c_i) \land \domR{\cellV{d}}(h') \land \hpR(h') \land \hpPosCCHpR(c,h'))  \) \\[\jot]
  \> (for each $h \in \Arr_\psi$, defined by \(c_0 + c_1 x_1 + \dots + c_d x_d > 0\))
\end{tabbing}

\section{Proof of Lemma~\ref{lemintm} for any fixed $m$ and $\ell$}\label{secproofintmd}

\newcommand{\inIntervalF}{\textit{cell-in-interval}}
\newcommand{\posVolCellF}{\textit{pos-side-vol-cell}}
\newcommand{\negVolCellF}{\textit{neg-side-vol-cell}}
\newcommand{\cellInPtopeF}{\textit{cell-in-ptope}}
\newcommand{\startHpF}[1]{\textit{interval-start-hp}_{#1}}
\newcommand{\EndHpF}[1]{\textit{interval-end-hp}_{#1}}
\newcommand{\allSectorsF}{\textit{all-sectors-cell}}
\newcommand{\realFaceICellF}[2]{\textit{face-of-}\mathit{cell}_{#1}^{#2}}
\newcommand{\validFacesSimplexF}[1]{\textit{valid-faces-}\mathit{simplex}_{#1}}
\newcommand{\noCornerParentF}[1]{\textit{not-contain-corner}_{#1}}
\newcommand{\realSimplexCellF}[1]{\textit{simplex-of-cell}_{#1}}
\newcommand{\gensNoSimplicesF}[1]{\textit{generates-no-simplices}_{#1}}
\newcommand{\detT}[1]{\det_{#1}}
\newcommand{\simplexVolT}[1]{\textit{simplex-vol}_{#1}}
\newcommand{\cellVolT}{\textit{cell-vol}}

\begin{lemma}[Repetition of Lemma~\ref{lemintm}]
  For any $m$, there exists an $\wal$ term $t$ over $\Upsnet(m,1)$ with 
  $m$ additional pairs of weight constant symbols $\mathit{min}_i$ and
  $\mathit{max}_i$ for $i \in \{ 1, \dots, m\}$, such that for any fixed $\ell$,
  any network $\Net$ in $\K(m,\ell)$, and $m$ pairs of values $a_i$ and $b_i$ for
  $\mathit{min}_i$ and $\mathit{max}_i$, we have \[
    t^{\Net,a_1,b_1,\dots,a_m,b_m} = \int_{a_1}^{b_1}\cdots \int_{a_m}^{b_m} \Func\Net\,dx_1\dots dx_m .
  \]
\end{lemma}

The proof proceeds in four steps. The first step is to translate our network $\Net$ from an $\Upsnet$-structure into an $\PWLV{m}$ structure. Lemma~\ref{lempwl} gives this translation for any network in $\K(m', \ell')$ for any arbitrary $m'$ and $\ell'$.

The second step is to create a hyperplane arrangement in $\R^{m+1}$ that contains $\Arr_{\Func{\Net}}$ together with $x_{m+1} = 0$, and all hyperplanes that either are at the start or end of the intervals over which we are integrating for each dimension in $x_1,\dots,x_m$. Concretely, $\Arr_H = S \cup E \cup \{ x_{m+1} = 0\}$ where $S = \{ x_i - a_i = 0 \mid i \in \{1,\dots,m\}\}$ and $E = \{ x_i - b_i= 0 \mid i \in \{1,\dots,m\}\}$. Using Lemma~\ref{lemfuncHarr} we can then create a translation that maps $\Func{\Net}$ to $\Arr_{\Func{\Net}} \cup \Arr_H$.

The third step is to create a cylindrical decomposition $\D$ compatible with $\Arr_{\Func{\Net}} \cup \Arr_H$ by invoking Lemma~\ref{lemcd}.

The final step is to define the term $\integrateT()$ over the vocabulary $\cellV{m+1} \uplus \PWLV{m}$ that calculates the desired integral on a structure representing the disjoint union of $\D$ and $\Func{\Net}$. It does this by selecting all cells that are within the given intervals and that are either below the function and above the $x_1,\dots,x_m$-hyperplane, or are above the function and below the $x_1,\dots,x_m$-hyperplane. We then calculate their volumes and either add them to or subtract them from the total volume, depending on whether they are above or below the $x_1,\dots,x_m$-hyperplane. 

We shall start by defining the formula $\inIntervalF(c)$ that selects all $(m+1)$-cells that are within the given interval. For this we will use the formulas $\hpPosCCHpR$, $\hpNegCCHpR$ from Lemma~\ref{lemselrelhp}. Additionally, we define the formulas $\startHpF{i}(h)$ and $\EndHpF{i}(h)$ that for each $i \in \{1,\dots,m\}$ select the hyperplane at the start and end of the interval for the dimension $x_i$.

\begin{tabbing}
  \quad\=\,\,\,\=\,\,\,\=\,\,\,\=\,\,\,\=\,\,\,\=\,\,\,\=\,\,\,\=\kill
  \( \displaystyle \startHpF{i}(h) := \domR{\cellV{m}}(h) \land \hpR(h) \land \coefCellW{i} = a_i \land (\bigwedge_{j \in \{1,\dots,m+1\}\setminus\{i\}} \coefCellW{j} = 0)\) \\[\jot]
  \> (for each $i \in \{1,\dots,m\}$)\\[\jot]
  \( \displaystyle \EndHpF{i}(h) := \domR{\cellV{m}}(h) \land \hpR(h) \land \coefCellW{i} = b_i \land (\bigwedge_{j \in \{1,\dots,m+1\}\setminus\{i\}} \coefCellW{j} = 0)\) \\[\jot]
  \> (for each $i \in \{1,\dots,m\}$)\\[\jot]
  
  \( \inIntervalF(c) := \) \\[\jot]
  \> \( \displaystyle \forall h (((\bigvee_{i \in \{1,\dots,m\}}\startHpF{i}(h)) \implies \hpPosCCHpR(c,h)) \)\\[\jot]
  \> \> \( \displaystyle {} \land ((\bigvee_{i \in \{1,\dots,m\}}\EndHpF{i}(h)) \implies \hpNegCCHpR(c,h))) \)
\end{tabbing}

We will now define formulas that select the $(m+1)$-cells $c$ within the interval and that are between the function and $x_{m+1} = 0$; formula $\posVolCellF(c)$ select those on the positive side of $x_{m+1} = 0$, $\negVolCellF(c)$ those on the negative side. To be able to define these formulas we will first define the formula $\cellInPtopeF(c,p)$ which is true if all values the $(m+1)$-cell $c$ have their first $m$-components restricted to the values contained in the polytope $p$.
\begin{tabbing}
  \quad\=\,\,\,\=\,\,\,\=\,\,\,\=\,\,\,\=\,\,\,\=\,\,\,\=\,\,\,\=\kill
  \( \cellInPtopeF(c,p) :=  \) \\[\jot]
  \> \( \domR{\PWLV{m}}(p) \land \ptopeR(p)  \)\\[\jot]
  \> \> \( {} \land \forall b ((\domR{\PWLV{m}}(b) \land \breakplaneR(b)) \)\\[\jot]
  \> \> \> \( \displaystyle {} \implies  \exists h_b (\domR{\cellV{d}}(h) \land \hpR(h_b) \land \coefCellW{m+1}=0 \land (\bigwedge_{i\in\{1,\dots,m\}} \coefPWLW{i}(b) = \coefCellW{i}(h_b)) \)\\[\jot]
  \> \> \> \> \> \( {} \land ((\posSR(p,b) \land \hpPosCCHpR(c,h_b))\)\\[\jot]
  \> \> \> \> \> \> \( {} \lor (\negSR(p,b) \land \hpNegCCHpR(c,h_b))\)\\[\jot]
  \> \> \> \> \> \> \( {}  \lor (\onSR(p,b) \land \hpOnCCHpR(c,h_b))))) \) \\[\jot]
  
  \( \posVolCellF(c):= \) \\[\jot]
  \> \( \displaystyle \exists p,h_p,h_m ( \hpR(h_m) \land \coefCellW{m+1}(h_m) = 1 \land (\bigwedge_{i\in\{1,\dots,m\}} \coefCellW{i}(h_m) = 0) \) \\[\jot]
  \> \> \( \displaystyle {} \land \cellInPtopeF(c,p) \land \hpR(h_p) \land \coefCellW{m+1}(h_p)=-1 \land (\bigwedge_{i\in\{1,\dots,m\}} \coefPWLW{i}(p) = \coefCellW{i}(h_p)) \) \\[\jot]
  \> \> \( {} \land \inIntervalF(c) \land \hpPosCCHpR(c,h_p) \land \hpPosCCHpR(c,h_m) ) \) \\[\jot]
  \( \negVolCellF(c):= \) \\[\jot]
  \> \( \displaystyle \exists p,h_p,h_m ( \hpR(h_m) \land \coefCellW{m+1}(h_m) = 1 \land (\bigwedge_{i\in\{1,\dots,m\}} \coefCellW{i}(h_m) = 0) \) \\[\jot]
  \> \> \( \displaystyle {} \land \cellInPtopeF(c,p) \land \hpR(h_p) \land \coefCellW{m+1}(h_p)=-1 \land (\bigwedge_{i\in\{1,\dots,m\}} \coefPWLW{i}(p) = \coefCellW{i}(h_p)) \) \\[\jot]
  \> \> \( {} \land \inIntervalF(c) \land \hpNegCCHpR(c,h_p) \land \hpNegCCHpR(c,h_m) ) \)
\end{tabbing}

Since we will be calculating a volume in $m+1$ dimensions, we only want to select cells that are actually $(m+1)$-dimensional $(m+1)$-cells. This is done by the formula $\allSectorsF(c)$, which simply checks that all cells in the chain of base cells of the $(m+1)$-cell $c$ are sectors.

\begin{tabbing}
  \quad\=\,\,\,\=\,\,\,\=\,\,\,\=\,\,\,\=\,\,\,\=\,\,\,\=\,\,\,\=\kill
  \( \allSectorsF(c) :=  \) \\[\jot]
  \> \( \exists c_0, h_{s,1}, h_{e,1}, c_{1}, \dots, h_{s,m+1}, h_{e,m+1}, c_{m+1} (c_{m+1} = c \land c_0 = \originCellC \)\\[\jot]
  \> \> \( \displaystyle {} \land (\bigwedge_{i\in\{1,\dots,m+1\}} \startR(c_i,h_{s,i}) \land \EndR(c_i,h_{e,i}) \land \neg \stackEqR(c_i,h_{s,i}, h_{e,i}) \land \baseR(c_{i},c_{i-1})) )) \)
\end{tabbing}

Now all that remains is the calculate the volumes of the relevant $(m+1)$-dimensional $(m+1)$-cells and add their contributions. 
To do this we will split up these cells into $(m+1)$-simplices, calculate their volumes and add these up.
Observe that these cells are convex polytopes. In general for $i \in \{1,\dots,m+1\}$, let us represent a convex $i$-polytope by the set of its corner points. We use a straightforward recursive method for triangulating a convex $i$-polytope~\cite{clarkson88} given as a finite set $F$ of corner points.
Specifically, consider the following algorithm $\textrm{Triangulate}(F,i)$ returning a set of $i$-simplices; each $i$-simplex is given as a set of $i+1$ points from $F$.
\begin{enumerate}
  \item If $i=1$, return ${F}$.
  \item Otherwise, choose $r \in F$ arbitrarily.
  \item Note that each facet $F' \subset F$ is an $(i-1)$-polytope. Return \[\{ s \cup \{r\} \mid \exists F': F' \text{ is a facet of } F \text{ not containing } r \text{ and } s \in \textrm{Triangulate}(F',i-1)\}.\]
\end{enumerate}

To implement this algorithm in $\wal$ we will first need to be able to represent the corner points of a cell. Since an $i$-cell can at most have $2^i$ corner points, we identify them with strings of length $i$ over the alphabet $\{\StartC,\EndC\}$ as done in the proof of Lemma~\ref{lemcdstackord}. Recall that if a corner $r$ of an $i$-cell $c$ has $r_i = \StartC$, then $r$ is the corner $r_1,\dots,r_{i-1}$ of the base cell of $c$ projected onto the hyperplane representing the lower delineating section mapping of $c$. Likewise, if $r_i = \EndC$, then $r$ is $r_1,\dots,r_{i-1}$ of the base cell of $c$ projected onto the hyperplane representing the upper delineating section mapping of $c$.

From the proof of Lemma~\ref{lemcdstackord}, we have the terms $\coordT_{j,r}(c)$ for each $1 \leq j \leq i$ and each $r \in \{\StartC, \EndC\}^i$, which return the $j$-th coordinate of the corner $r$ of the $i$-cell $c$.

Next we will need to represent a facet of an $(m+1)$-dimensional $(m+1)$-cell. We discuss this for any $i$-dimensional $i$-cell $c$, with $i\in \{1,\dots,m+1\}$, since we will use some notation later.
For any facet of $c$ we claim that all its corner points are identified by $R_{(a, j)}^i = \{ r \in \{\StartC,\EndC\}^i \mid r_j = a \}$ for some $a \in \{\StartC, \EndC\}$ and $j \in \{1,\dots,i\}$. For a $1$-dimensional $1$-cell this is trivial, since each facet only has a single corner point. Assume this holds for an $(i-1)$-dimensional $(i-1)$-cell $c_b$ that is the base cell of $c$. Let $h_1$ and $h_2$ be the hyperplanes that represent the lower an upper deliniating section mappings of $c$. Then a facet of $c$ is either the set of corner points of a facet of $c_b$ projected onto $h_1$ and $h_2$, which we will call a type A facet, or the set of all corner points of $c_b$ projected on either $h_1$ or $h_2$, which we will call a type B facet. In Figure~\ref{fig:facetIdEx} both the facets $(\EndC,1)$ and $(\EndC,2)$ of the 3-dimensional cell are type A facets constructed from the facets $(\EndC,1)$ and $(\EndC,2)$ of the base cell respectively. The facet $(\EndC,3)$ of the 3-dimensional cell is of type B, constructed by projecting all corners of the base cell onto the upper deliniating hyperplane.
If a facet of $c$ is type A, then by induction the corner points of the facet of $c_b$ are identified by $R^{i-1}_{(a,j)}$ for some $a \in \{\StartC, \EndC\}$ and $j \in {1,\dots,i-1}$ and thus the corner points of the facet of $c$ are identified by $R^{i}_{a,j}$.
If the facet of $c$ is type B, then clearly they are identified by $R_{\StartC, i}^i$ or $R_{\EndC, i}^i$ respectively.
Thus we will represent for an $i$-dimensional $i$-cell a facet with corner points $R_{a, j}^i$ by the pair $(a, j)$ if $i$ is clear from context, as done in Figure~\ref{fig:facetIdEx}.

\begin{figure}
  \centering
  \includegraphics{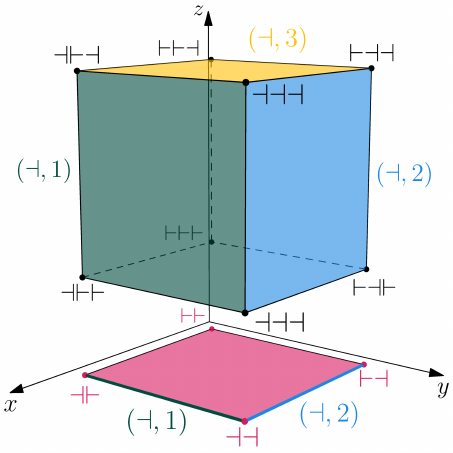}
  \caption{A $3$-dimensional $3$-cell and its base cell bellow it with some facets identified and highlighted in color, and all corners identified}
  \label{fig:facetIdEx}
\end{figure}

Note that for an $i$-dimensional $i$-cell multiple identifiers
may identify the same corner point. For example, in
Figure~\ref{figrealfacet}, in $c_{\triangle}$ the identifiers
$\StartC\StartC$ and $\StartC\EndC$ both identify the upper right
corner of $c_{\triangle}$.  As a consequence,
not every set of the form $R^i_{(a,j)}$ necessarily identifies a
facet.  For example, in
Figure~\ref{figrealfacet}, we have that $(\StartC,1)$ is a facet
in $c_{\square}$ but not of $c_{\triangle}$ even though both are
$2$-dimensional $2$-cells.

So far we have shown how to identify a facet of an
$i$-dimensional $i$-cell $c$. However, our algorithm recursively
takes facets of these facets, which we will show how to identify
next. If a facet has been recursively taken $k$ times of $c$,
this is an $(i-k)$-face of $c$, which is always an
$(i-k)$-polytope.  Similarly, such a face is identified by
$R^{i}_{(a_1, j_1), (a_2,
j_2), \dots, (a_k, j_k)} = \{ r \in \{\StartC, \EndC\}^i \mid
\bigwedge_{\ell \in \{1, \dots, k\}} r_\ell = a_\ell \}$ for some
$a_1, \dots, a_k \in \{\StartC, \EndC\}$ and $j_1, \dots, j_k \in
\{ 1, \dots, i \}$ with all $j$'s distinct.

\begin{figure}
  \centering
  \begin{tikzpicture}
      \begin{axis}[
        xlabel = \(x\),
        xmin=-0.2, xmax=4.5,
        ylabel = \(y\),
        ymin=-0.2, ymax=3.8,
        ticks = none,
        axis lines = center,
        width = 0.5\linewidth,
        axis equal image
      ]
      \addplot [
        color = black
      ] coordinates {
          (-0.2,3) (5.5,3)
        };
      \addplot [
        color = black
      ] coordinates {
          (-0.2,3.2) (3.5,-0.5)
        };
      \addplot [
        color = black
      ] coordinates {
          (-0.2,5.2) (5.5,-0.5)
        };
      \addplot [
        color = black
      ] coordinates {
          (2,-1) (5.5,2.5)
        };
      \addplot [
        color = black,
        dashed
      ] coordinates {
          (2,4) (2,3)
        };
      \addplot [
        color = black,
        dashed
      ] coordinates {
          (2,1) (2,-1)
        };
      \addplot [
        color = black,
        dashed
      ] coordinates {
          (3,4) (3, 3)
        };
      \addplot [
        color = black,
        dashed
      ] coordinates {
          (3, 2) (3,-1)
        };
      \addplot [
        color = black,
        dashed
      ] coordinates {
          (4,4) (4,3)
        };
      \addplot [
        color = black,
        dashed
      ] coordinates {
          (4, 1) (4,-1)
        };
      \addplot [
        color = red,
        only marks
      ] coordinates {
          (0,3) (2, 1) (2,3)
        };
      \addplot [
        color = blue,
        only marks
      ] coordinates {
          (3, 3) (3, 2) (4, 1) (4, 3)
        };
      \addplot[
        color = red,
        dashed,
        line width=0.125em
      ] coordinates {
        (0,3) (2, 1) (2,3)
      } -- cycle;
      \addplot[
        color = blue,
        dashed,
        line width=0.125em
      ] coordinates {
        (3, 3) (3, 2) (4, 1) (4, 3) 
        } -- cycle;
      \node[anchor=south west, color=red] at (0, 3) {$\StartC\StartC, \StartC\EndC$};
      \node[anchor=west, color=red] at (2, 1) {$\EndC\StartC$};
      \node[anchor=south west, color=red] at (2,3) {$\EndC\EndC$};
      \node[anchor=north east, color=blue] at (3, 2) {$\StartC\StartC$};
      \node[anchor=south east, color=blue] at (3, 3) {$\StartC\EndC$};
      \node[anchor=east, color=blue] at (4, 1) {$\EndC\StartC$};
      \node[anchor=south west, color=blue] at (4, 3)  {$\EndC\EndC$};
      \node[] at (3.5, 2.25)  {$c_{\square}$};
      \node[] at (1.5, 2.25)  {$c_{\triangle}$};
    \end{axis}
    \end{tikzpicture}
    \caption{Cylindrical decomposition of a $2$-D hyperplane arrangement with two 2-dimensional 2-cells $c_{\triangle}$ (in red) and $c_{\square}$ (in blue) highlighted with the identifiers of their corner points given.}
    \label{figrealfacet}
\end{figure}
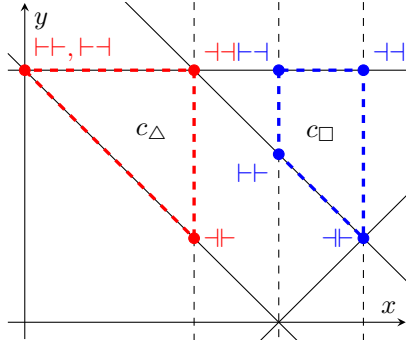

To make our algorithm
deterministic, we choose $r$ predictably, by
always choosing $r$ to be the
string
in $R^i_{(a_1, j_1),\dots,(a_k, j_k)}$
with the most $\StartC$ symbols.
As a consequence of this fixed choice
of $r$, we can simplify the identifiers for the faces we will
consider in our algorithm. Assume we have an $i$-dimensional
$i$-cell $c$ and for some $0 \leq k \leq i-2$ a $(i-k)$-face
$(a_1, j_1),\dots,(a_k, j_k)$ of which we want to identify the
facets. The deterministic algorithm chooses
$r \in \{\StartC,\EndC\}^i$ with
$r_{j_\ell} = a_\ell$ for $\ell \in \{1, \dots, k\}$ and
$\StartC$ elsewhere. Then any facet $(a_1, j_1),\dots,(a_k, j_k),
(\StartC,j_{k+1})$ for any $j_{k+1} \in \{1, \dots, i\} \setminus
\{j_1, \dots, j_k\}$ would contain $r$. Thus $a_k$ can only be
$\EndC$. Since this holds for any $k$ we can simplify the
identifiers of the faces our algorithm considers by identifying
$(a_1, j_1),\dots,(a_k, j_k)$ by $j_1,\dots, j_k$ since we know
that $a_1 = \cdots = a_k = \EndC$. Similarly, we will denote the
set of corners of $j_1,\dots, j_k$ by $R^i_{j_1,\dots, j_k}$ rather
than $R^i_{(\EndC, j_1), \dots, (\EndC, j_k)}$.

Recall that our ambient dimension is $m+1$.
Simplices returned by the algorithm can be represented by the
sequence of chosen facets.  Since every chosen facet has the
previous chosen facet as its prefix, we thus represent each
simplex by the last chosen facet, i.e., by a sequence $j_1,\dots,
j_m$ as above.  Due to the problem that different strings may
identify the same points, mentioned earlier, not every sequence
$j_1,\dots,j_m$
represents a sequence of actual facet choices.
Since at each step the dimension decreases by at least one, we
can check that $j_1,\dots,j_m$ 
represents a sequence of actual facet choices by checking that
the two strings in $R_{j_1,\dots,j_m}$ do not identify the same
point:

\begin{tabbing}
  \quad\=\,\,\,\=\,\,\,\=\,\,\,\=\,\,\,\=\,\,\,\=\,\,\,\=\,\,\,\=\kill
  \( \displaystyle \validFacesSimplexF{j_1,\dots,j_m}(c) :=
\neg \bigwedge_{\ell \in \{1,\dots, m+1\}}\coordT_{\ell,r_1}(c) =
  \coordT_{\ell,r_2}(c) \) \\[\jot]
  \> (with $r_1,r_2$ the two distinct elements of
  $R^{m+1}_{j_1,\dots,j_m}$)
\end{tabbing}

The algorithm only chooses facets that do not contain the corner
$r$. This is expressed as follows:

\begin{tabbing}
  \quad\=\,\,\,\=\,\,\,\=\,\,\,\=\,\,\,\=\,\,\,\=\,\,\,\=\,\,\,\=\kill
  \( \displaystyle \noCornerParentF{j_1,\dots,j_{k}}(c) :=
  \bigwedge_{r' \in R_{j_1,\dots,j_k}}
  \neg \bigwedge_{\ell \in \{1,\dots, m+1\}}
  \coordT_{\ell,r}(c) = \coordT_{\ell,r'}(c) \)
  \\[\jot]
  \> (with $r$ the string where $r_\ell={\dashv}$ for
  $\ell\in\{j_1,\dots,j_{k-1}\}$ and $r_\ell = {\vdash}$ elsewhere)
\end{tabbing}

Hence, the following formula checks that $j_1,\dots,j_m$
represents a simplex returned by the algorithm:

\begin{tabbing}
  \quad\=\,\,\,\=\,\,\,\=\,\,\,\=\,\,\,\=\,\,\,\=\,\,\,\=\,\,\,\=\kill
  \( \displaystyle \realSimplexCellF{j_1,\dots,j_m}(c) := \)\\[\jot]
  \> \(\validFacesSimplexF{j_1,\dots,j_m}(c) \land (\bigwedge_{k \in \{1,\dots, m\}} \noCornerParentF{j_1,\dots,j_{k}}(c)) \)
\end{tabbing}

Note that the corner points of a returned simplex $j_1,\dots,j_m$
are identified by the set of $m+2$ strings $S_{j_1,\dots,j_m}$ defined
below.  The idea is that these are the two points of face
$j_1,\dots,j_m$, plus the deterministically chosen corners of all
higher chosen facets.

\begin{tabbing}
  \quad\=\,\,\,\=\,\,\,\=\,\,\,\=\,\,\,\=\,\,\,\=\,\,\,\=\,\,\,\=\kill
  \( \displaystyle S_{j_1,\dots,j_{m}}^1=\{ r \in \{\StartC,\EndC\}^{m+1} \mid \bigwedge_{j \in \{j_1,\dots,j_{m}\}} r_j = \EndC\}\) \\[\jot]
  \( S_{j_1,\dots,j_{m}}^i = S_{j_1,\dots,j_m}^{i-1} \cup {} \) \\[\jot]
  \> \( \displaystyle \{ r \in \{\StartC,\EndC\}^{m+1} \mid
  \bigwedge_{j \in \{j_1,\dots,j_{m+1-i}\}} r_j = \EndC \land
  \bigwedge_{j \in \{1,\dots, m+1\} \setminus \{j_1,\dots,j_{m+1-i}\}}
  r_j = \StartC\}\)\\[\jot] 
  \> (for $1 < i \leq m$ ) \\[\jot]
  \( \displaystyle S_{j_1,\dots,j_{m}}= S_{j_1,\dots,j_{m}}^{m}
  \cup \{\StartC^{m+1}\} \)
\end{tabbing}

It remains to calculate the volume of every returned simplex.
In general, recall that the volume of an $n$-simplex given by
$n+1$ corners $\vv_1,\dots,\vv_{n+1}$, where each $\vv_i \in
\R^{n}$ is a column vector of coordinates, is given by $\det
(\vv_1-\vv_{n+1},\dots,\vv_n-\vv_{n+1})$.  
In our case, $n=m+1$, and $\vv_1,\dots,\vv_{m+2}$ are given by
the strings $r_1,\dots,r_{m+2}$ in $S_{j_1,\dots,j_m}$.
The coordinates of $r_i$ are defined by the $\wal$ terms
$\coordT_{j,r_i}(c)$ for $j=1,\dots,n$. Using these terms in the
well-known expression for determinant (the Laplace expansion), we
obtain an $\wal$ term $\simplexVolT{j_1,\dots,j_m}(c)$.

Consequently, the volume of cell $c$ is given by the following
term:

\begin{tabbing}
  \quad\=\,\,\,\=\,\,\,\=\,\,\,\=\,\,\,\=\,\,\,\=\,\,\,\=\,\,\,\=\kill
  \( \displaystyle \cellVolT(c) := \sum_{s \in S}
  \ifText \realSimplexCellF{s}(c) \thenText \simplexVolT{s}(c) \elseText 0 \)\\[\jot]
  \> (where $S$ is the set of all $m$-length sequences
  $s=j_1,\dots, j_m$ over $\{1,\dots,m+1\}$ \\
  \> \> with all $j$'s distinct)
\end{tabbing}

Finally, the term $\integrateT()$ can be written as the sum of all volumes of all cells with a positive contribution to the integral and subtract from that the sum of all volumes of all cells with a negative contribution to the integral.

\begin{tabbing}
  \quad\=\,\,\,\=\,\,\,\=\,\,\,\=\,\,\,\=\,\,\,\=\,\,\,\=\,\,\,\=\kill
  \( \integrateT() := \)\\[\jot] 
  \> \( \displaystyle (\sum_{c:\posVolCellF(c) \land \allSectorsF(c)} \cellVolT(c)) \)\\[\jot] 
  \> \( \displaystyle {} - (\sum_{c:\negVolCellF(c) \land \allSectorsF(c)} \cellVolT(c)) \)
\end{tabbing}

\end{document}